\documentclass{llncs}

\usepackage{amssymb}
\usepackage{amsmath}       
\usepackage{graphics,color} 
\usepackage{rotating}       
\usepackage{bbm}
\usepackage{tikz}
\usetikzlibrary{shapes,decorations,decorations.pathmorphing}
\usetikzlibrary{decorations.markings,calc,patterns}
\usetikzlibrary{positioning,arrows,fit,shapes}
\usetikzlibrary{decorations.pathmorphing}
\usetikzlibrary{calc}
\usetikzlibrary{patterns}
\usetikzlibrary{fit}
\usetikzlibrary{decorations.pathreplacing,backgrounds}

\tikzset{>=latex,
	point/.style = {circle,draw=black,thick,minimum size=1.1mm,inner sep=0pt},
	hm/.style = {dotted,semithick},
	role/.style = {semithick,-latex},
	tree/.style = {rounded corners=10pt, dashed, fill opacity=0.5, fill=nullscolour},
	constant/.style = {fill},
}

\usepackage{stmaryrd} 
\usepackage{xspace}   

\usepackage{multirow}
\usepackage{booktabs}
\usepackage{makeidx}
\makeindex

\newcommand{\mn}[1]{\ensuremath{\mathsf{#1}}}
\newcommand{\sig}{\mn{sig}}   
\newcommand{\vp}{\varphi}     

\newcommand{\Amc}{\ensuremath{\mathcal{A}}\xspace}  
\newcommand{\Cmc}{\ensuremath{\mathcal{C}}\xspace}  
\newcommand{\Hmc}{\ensuremath{\mathcal{H}}\xspace}  
\newcommand{\Gmc}{\ensuremath{\mathcal{G}}\xspace}  
\newcommand{\Imc}{\ensuremath{\mathcal{I}}\xspace}  
\newcommand{\Jmc}{\ensuremath{\mathcal{J}}\xspace}  
\newcommand{\Lmc}{\ensuremath{\mathcal{L}}\xspace}  
\newcommand{\Kmc}{\ensuremath{\mathcal{K}}\xspace}  
\newcommand{\Mmc}{\ensuremath{\mathcal{M}}\xspace}  
\newcommand{\Omc}{\ensuremath{\mathcal{O}}\xspace}  
\newcommand{\Qmc}{\ensuremath{\mathcal{Q}}\xspace}  
\newcommand{\Tmc}{\ensuremath{\mathcal{T}}\xspace}  

\newcommand{\NC}{\ensuremath{{\sf N_C}}\xspace}     
\newcommand{\NI}{\ensuremath{{\sf N_I}}\xspace}     
\newcommand{\NR}{\ensuremath{{\sf N_R}}\xspace}     
\newcommand{\NV}{\ensuremath{{\sf N_V}}\xspace}     

\newcommand{\EL}{\ensuremath{{\cal EL}}\xspace}     
\newcommand{\ALC}{\ensuremath{{\cal ALC}}\xspace}
\newcommand{\ALCI}{\ensuremath{{\cal ALCI}}\xspace}

\newcommand{\ALCHI}{\mathcal{ALCHI}\xspace}

\newcommand{\ALCQ}{\ensuremath{{\cal ALCQ}}\xspace}

\newcommand{\ALCQI}{\ensuremath{{\cal ALCQI}}\xspace}
\newcommand{\ALCQIO}{\ensuremath{{\cal ALCQIO}}\xspace}

\newcommand{\DLite}{\textsl{DL-Lite}\xspace}
\newcommand{\DLc}{{\ensuremath\textsl{DL-Lite}_\textit{core}}\xspace}
\newcommand{\DLcH}{{\ensuremath\textsl{DL-Lite}_\textit{core}^{\smash{\mathcal{H}}}}\xspace}
\newcommand{\DLh}{{\ensuremath{\textsl{DL-Lite}_\textit{horn}}}\xspace}
\newcommand{\DLhH}{{\ensuremath{\textsl{DL-Lite}_\textit{horn}^{\smash{\mathcal{H}}}}}\xspace}

\newcommand{\hALCHI}{{\ensuremath\textsl{Horn-}{\cal ALCHI}}\xspace}
\newcommand{\hALCI}{{\ensuremath\textsl{Horn-}\mathcal{ALCI}}\xspace}
\newcommand{\hALC}{{\ensuremath\textsl{Horn-}\mathcal{ALC}}\xspace}

\newcommand{\ELI}{\mathcal{ELI}\xspace}

\newcommand{\qnrleq}[3]{\ensuremath{(\leqslant #1 \; #2 \; #3)}}  
\newcommand{\qnrgeq}[3]{\ensuremath{(\geqslant #1 \; #2 \; #3)}}  

\newcommand{\NP}{\textnormal{\sc NP}\xspace}             

\newcommand{\ExpTime}{\textnormal{\sc ExpTime}\xspace}

\newcommand{\PTime}{\textnormal{\sc PTime}\xspace}

\newcommand{\NExp}{\textnormal{\sc NExp}\xspace}
\newcommand{\coNExp}{\mbox{\sc coNExp}\xspace}

\newcommand{\Mod}{\boldsymbol{M}} 
\newcommand{\avec}[1]{\boldsymbol{#1}} 

\newcommand{\RCQ}{\text{rCQ}} 
\newcommand{\RUCQ}{\text{rUCQ}} 

\newcommand{\vertexfont}{\small}
\newcommand{\edgefont}{\scriptsize}

\tikzset{ %
  point/.style={thick,circle,draw=black,minimum size=1.3mm,inner sep=0pt},%
  constant/.style={fill=black},%
  rowind/.style={fill=gray},%
  cwitness/.style={circle,draw=black,minimum size=3.5mm,inner sep=0pt,
    label=center:{$\land$}},%
  dwitness/.style={minimum width=1.1cm, minimum height=0.3cm},%
  homomorphism/.style={line width=0.1cm,-latex},%
  role/.style={-latex, semithick},%
  wiggly/.style={role, decorate, decoration={snake, amplitude=0.3mm,segment
      length=2mm, post length=1mm}},%
  subtree/.style={isosceles triangle, draw, very thick, subtreecolor,
    anchor=north, outer sep=0.1, isosceles triangle stretches, minimum
    width=1cm, minimum height=0.6cm, shape border rotate=90}, %
  backward/.style={rectangle, draw=black, fill=gray!50, minimum size=1mm, inner
    ysep=4pt, inner xsep=7pt, outer sep=0.05cm, rounded corners=2mm},%
  sbound/.style={rectangle, draw=black, fill=gray!10, minimum size=1mm, inner
    ysep=4pt, inner xsep=7pt, outer sep=0.05cm},%
  infinite/.style={rectangle, draw=black, fill=gray!25, minimum size=1mm, inner
    ysep=4pt, inner xsep=7pt, outer sep=0.05cm, minimum width=1.5cm},%
  trans/.style={-stealth', semithick},%
  edge/.style={semithick},%
  strategy/.style={dotted},%
}

\newcommand{\homofont}{\scriptsize}

\colorlet{subtreecolor}{black!70}

\def \tikzdots[#1]{
  \begin{scope}[shift={#1}]
    \node at (0,0.15) {$\cdot$}; \node {$\cdot$}; \node at (0,-0.15) {$\cdot$};
  \end{scope}
}

\newcommand{\q}{\boldsymbol{q}}
\newcommand{\I}{\Imc}
\newcommand{\K}{\Kmc}
\newcommand{\A}{\Amc}
\newcommand{\T}{\Tmc}
\newcommand{\C}{\mathcal{C}}

\newcommand{\ind}{\mathsf{ind}}

\begin{document}
\mainmatter              

\title{Inseparability and Conservative Extensions of Description Logic Ontologies: A Survey\thanks{This is a post-peer-review, pre-copyedit version of an article published in RW2016. The final authenticated version is available online at: http://dx.doi.org/10.1007/978-3-319-49493-7\_2}}

\titlerunning{Inseparability Relations between Description Logic Ontologies: A Survey}  

\author{Elena Botoeva\inst{1} \and Boris Konev\inst{2} \and Carsten Lutz\inst{3} \and Vladislav Ryzhikov\inst{1} \and Frank Wolter\inst{2} \and Michael Zakharyaschev\inst{4}}
\authorrunning{Konev \textit{et~al.}}   


\institute{Free University of Bozen-Bolzano, Italy\\
\email{\{botoeva,ryzhikov\}@inf.unibz.it}
\and
University of Liverpool, UK\\
\email{\{konev,wolter\}@liverpool.ac.uk}
\and
University of Bremen, Germany\\
\email{clu@informatik.uni-bremen.de}
\and
Birkbeck, University of London, UK\\
\email{michael@dcs.bbk.ac.uk}
}

\maketitle   

\newcommand{\equivalently}{safely\xspace}

\begin{abstract}
  The question whether an ontology can \equivalently be replaced by
  another, possibly simpler, one is fundamental for many ontology
  engineering and maintenance tasks. It underpins, for example,
  ontology versioning, 
  ontology
  modularization, 
  forgetting, 
  and knowledge
  exchange. 
  What `safe replacement' means depends on the intended application
  of the ontology. If, for example, it is used to query data, then the
  answers to any relevant ontology-mediated query should be the same over any 
  relevant data set; if, in contrast, the ontology is used for
  conceptual reasoning, then the entailed subsumptions between concept
  expressions should coincide. This gives rise to different notions of
  ontology \emph{inseparability} such as query inseparability and
  concept inseparability, which generalize corresponding notions of
  conservative extensions. We survey results on various notions of
  inseparability in the context of description logic ontologies,
  discussing their applications, useful model-theoretic
  characterizations, algorithms for
  determining whether two ontologies are inseparable (and, sometimes,
  for computing the difference between them if they are not), and 
  the computational complexity of this problem.
\end{abstract}



\section{Introduction}

Description logic (DL) ontologies provide a common vocabulary for a domain of interest together with a formal modeling of the semantics of the vocabulary items (concept names and role names). In modern information systems, they are employed to capture domain knowledge and to promote interoperability.
Ontologies can become large and complex as witnessed, for example, by the widely used healthcare ontology {\sc Snomed CT}, which contains more than 300,000 concept names, and the National Cancer Institute ({\sc NCI}) Thesaurus ontology, which contains more than 100,000 concept names. Engineering, maintaining and deploying such ontologies is challenging and labour intensive; it crucially relies on extensive tool support for tasks such as ontology versioning, ontology modularization, ontology summarization, and forgetting parts of an ontology. At the core of many of these tasks lie notions of  \emph{inseparability} of two ontologies, indicating that inseparable  ontologies can safely be replaced by each other for the task at hand. The aim of this article is to survey the current research on  inseparability of DL ontologies. We present and discuss different types of inseparability, their applications and interrelation, model-theoretic characterizations, as well as results on the decidability and computational complexity of inseparability.

The exact formalization of when an ontology `can safely be replaced by another one' (that is, of inseparability) depends on the task for which the ontology is to be used. As we are generally going to abstract away from the syntactic presentation of an ontology, a  natural first candidate for the notion of inseparability between two ontologies is their logical equivalence.
%
However, this can be an unnecessarily strong requirement for most applications since also ontologies that are not logically equivalent can be replaced by each other without adverse effects. This is due to two main reasons.  First, applications of ontologies often make use of only a fraction of the vocabulary items. As an example, consider {\sc Snomed CT}, which contains a vocabulary for a multitude of domains related to health case, including clinical findings, symptoms, diagnoses, procedures, body structures, organisms, pharmaceuticals, and devices.  In a concrete application such as storing electronic patient records, only a small part of this vocabulary is going to be used.  Thus, two ontologies should be separable only if they differ with respect to the \emph{relevant} vocabulary items. Consequently, all our inseparability notions will be parameterized by a signature (set of concept and role names) $\Sigma$; when we want to emphasize $\Sigma$, we speak of $\Sigma$-inseparability. Second, even for the relevant vocabulary items, many applications do not rely on all details of the semantics provided by the ontology.  For example, if an ontology is employed for conjunctive query answering over data sets that use vocabulary items from the ontology, then only the existential positive aspects of the semantics are relevant since the queries
are positive existential, too.

A fundamental decision to be taken when defining an inseparability notion is whether the definition should be model-theoretic or in terms of logical consequences. Under the first approach, two ontologies are inseparable when the reducts of their models to the signature $\Sigma$ coincide. We call the resulting inseparability notion \emph{model inseparability}.  Under the second approach, two ontologies are inseparable when they have the same logical consequences in the  signature~$\Sigma$. This actually gives rise to potentially many notions of inseparability since we can vary the logical language in which the logical consequences are formulated. Choosing the same language as the one used for formulating the ontologies results in what we call \emph{concept inseparability}, which is appropriate when the ontologies are used for conceptual reasoning. Choosing a logical language that is based on database-style queries results in notions of \emph{query inseparability}, which are appropriate for querying applications.  Model inseparability implies all the resulting consequence-based notions of inseparability, but the converse does not hold for all standard DLs. The notion of query inseparability suggests some additional aspects. In particular, this type of inseparability is important both for ontologies that contain data as an integral part (\emph{knowledge bases or KBs}, in DL parlance) and for those that do not (\emph{TBoxes}, in DL parlance) and are maintained independently of the data.  In the latter case, two TBoxes should be regarded as inseparable if they give the same answers to any relevant query \emph{for any possible data}. One might then even want to work with two signatures: one for the data and one for the query. It turns out that, for both KBs and TBoxes, one obtains notions of inseparability that behave very differently from concept inseparability.

Inseparability generalizes conservative extensions, as known from classical logic.  In fact, conservative extensions can also be defined in a model-theoretic and in a consequence-based way, and they correspond to the special case of inseparability where one ontology is syntactically contained in the other and the signature is the set of vocabulary items in the smaller ontology. Note that none of these two additional assumptions is appropriate for many applications of inseparability, such as ontology versioning.  Instead of directly working with inseparability, we will often consider corresponding notions of entailment which, intuitively, is inseparability `in one direction'; for example, two ontologies are concept $\Sigma$-inseparable if and only if they concept $\Sigma$-entail each other. Thus, one could say that an ontology concept (or model) entails another ontology if it is \emph{sound} to replace the former by the latter in applications for which concept inseparability is the `right' inseparability notion. Algorithms and complexity upper bounds are most general when established for entailment instead of inseparability, as they carry over to inseparability and conservative extensions. Similarly, lower bounds for conservative extensions imply lower bounds for inseparability and for entailment.

In this survey, we provide an in-depth discussion of four inseparability relations, as indicated above. For TBoxes, we look at $\Sigma$-concept inseparability (do two TBoxes entail the same concept inclusions over $\Sigma$?), $\Sigma$-model inseparability (do the $\Sigma$-reducts of the models of two TBoxes coincide?), and $(\Sigma_{1},\Sigma_{2})$-$\mathcal{Q}$-inseparability (do the answers given by two TBoxes coincide for all $\Sigma_{1}$-ABoxes and all $\Sigma_{2}$-queries from the class $\mathcal{Q}$ of queries?). Here, we usually take $\mathcal{Q}$ to be the class of conjunctive queries (CQs) and unions thereof (UCQs), but some smaller classes of queries are considered as well.  For KBs, we consider $\Sigma$-$\mathcal{Q}$-inseparability (do the answers to $\Sigma$-queries in $\mathcal{Q}$ given by two KBs coincide?). When discussing proof techniques in detail, we focus on the standard expressive DL $\mathcal{ALC}$ and tractable DLs from the $\mathcal{EL}$ and {\sl DL-Lite} families.  We shall, however, also mention results for extensions of $\mathcal{ALC}$ and other
DLs such as Horn-$\mathcal{ALC}$.

The structure of this survey is as follows. In Section~\ref{sec:introdesc}, we introduce description logics.  In Section~\ref{sec:inseparability}, we introduce an abstract notion of inseparability and discuss applications of inseparability in ontology versioning, refinement, re-use, modularity, the design of ontology mappings, knowledge base exchange, and forgetting. In Section~\ref{sect:separability}, we discuss concept inseparability. We focus on the description logics $\mathcal{EL}$ and $\mathcal{ALC}$ and give model-theoretic characterizations of $\Sigma$-concept inseparability which are then used to devise automata-based approaches to deciding concept inseparability. We also present polynomial time algorithms for acyclic $\mathcal{EL}$ TBoxes.  In Section~\ref{sect:model-inseparability}, we discuss model inseparability. We show that
it is undecidable even in simple cases, but that by restricting the signature $\Sigma$ to concept names, it often becomes decidable.  We also consider model inseparability from the empty TBox, which is important for modularization and locality-based approximations of model inseparability. In Section~\ref{sect:query-separability}, we discuss query inseparability between KBs.  We develop model-theoretic criteria for query inseparability and use them to obtain algorithms for deciding query inseparability between KBs and their complexity.  We consider description logics from the $\mathcal{EL}$ and {\sl DL-Lite} families, as well as $\mathcal{ALC}$ and its Horn fragment.  In Section~\ref{sec:query-insep-tboxes}, we consider query inseparability between TBoxes and analyse in how far the techniques developed for KBs can be generalized to TBoxes. We again consider a wide range of DLs. Finally, in Section~\ref{sect:final} we discuss further inseparability relations, approximation algorithms and the computation of representatives of classes of inseparable TBoxes.
  


\section{Description Logic}
\label{sec:introdesc}
In description logic, knowledge is represented using concepts and
roles that are inductively defined starting from a set \NC of
\emph{concept names} and a set \NR of \emph{role names}, and using a
set of concept and role constructors~\cite{BCMNP03}.  Different sets
of concept and role constructors give rise to different DLs.

We start by introducing the description logic \ALC. 
The concept constructors available in \ALC
are shown in Table~\ref{tab:syntax-semantics}, where $r$ is a
role name and $C$ and $D$ are concepts.
A concept built from these
constructors is called an \emph{\ALC-concept}. \ALC does not have any role constructors.
\begin{table}[tb]
  \begin{center}
    \small
    \leavevmode
    \begin{tabular}{|l|c|c|}
      \hline Name &Syntax&Semantics\\
      \hline\hline   & &\\[-1em]
      top concept    & $\top$ & $\Delta^\Imc$\\
      \hline         & &\\[-1em]
      bottom concept & $\bot$ & $\emptyset$ \\
      \hline         & &\\[-1em]
      negation       & $\neg C$ & $\Delta^\Imc \setminus C^\Imc$\\
      \hline         & &\\[-1em]
      conjunction    & $C\sqcap D$ & $C^\Imc\cap D^\Imc$\\
      \hline         & &\\[-1em]
      disjunction    & $C\sqcup D$ & $C^\Imc\cup D^\Imc$\\
      \hline         & &\\[-1em]
      existential restriction & $\exists r . C$ & $\{ d \in \Delta^\Imc \mid \exists e \in C^\Imc \, (d,e) \in r^\Imc \}$ \\
      \hline         & &\\[-1em]
      universal restriction & $\forall r . C$ & $\{ d \in \Delta^\Imc \mid \forall e \in \Delta^\Imc \, ((d,e) \in r^\Imc \to e \in C^\Imc) \}$ 
      \\ \hline
    \end{tabular} 
    \caption{Syntax and semantics of \ALC.}
    \label{tab:syntax-semantics}
  \end{center}
\end{table}
An \emph{\ALC TBox} is a finite set of \emph{\ALC concept inclusions} (CIs) of the form $C \sqsubseteq D$ and \emph{\ALC concept equivalences} (CEs) of the form $C \equiv D$. (A CE $C \equiv D$ can be regarded as an abbreviation for the two CIs $C \sqsubseteq D$ and $D \sqsubseteq C$.)   
 The \emph{size} $|\Tmc|$ of a TBox $\Tmc$ is the number of occurrences of symbols in $\Tmc$.

DL semantics is given by \emph{interpretations}
$\Imc=(\Delta^\Imc,\cdot^\Imc)$ in which the \emph{domain} $\Delta^\Imc$ is a
non-empty set and the \emph{interpretation function} $\cdot^\Imc$ maps
each concept name $A\in\NC$ to a subset $A^\Imc$ of $\Delta^\Imc$,
and each role name $r\in\NR$ to a binary relation $r^\Imc$ on
$\Delta^\Imc$. The extension of $\cdot^\Imc$ to arbitrary
concepts is defined inductively as shown in the third column of
Table~\ref{tab:syntax-semantics}.
We say that an interpretation \Imc \emph{satisfies} a CI $C
\sqsubseteq D$ if $C^\Imc \subseteq D^\Imc$, and that \Imc is a
\emph{model} of a TBox $\Tmc$ if it satisfies all the CIs in $\Tmc$.
A TBox is \emph{consistent} (or \emph{satisfiable}) if it has a
model. A concept $C$ is \emph{satisfiable w.r.t.~\Tmc} if there exists
a model $\Imc$ of $\Tmc$ such that $C^{\Imc}\neq\emptyset$.  A concept
$C$ is \emph{subsumed by a concept $D$ w.r.t.~\Tmc} ($\Tmc \models C
\sqsubseteq D$, in symbols) if every model \Imc of $\Tmc$ satisfies
the CI $C \sqsubseteq D$. For TBoxes $\Tmc_1,\Tmc_2$, we write
$\Tmc_1\models \Tmc_2$ and say that \emph{$\Tmc_1$ entails $\Tmc_2$}
if $\Tmc_1\models \alpha$ for all $\alpha\in \Tmc_2$.  TBoxes
$\Tmc_{1}$ and $\Tmc_{2}$ are \emph{logically equivalent} if they have
the same models.  This is the case if and only if $\Tmc_{1}$ entails
$\Tmc_{2}$ and vice versa.

A \emph{signature} $\Sigma$ is a finite set of concept and role names.
The \emph{signature} $\sig(C)$ of a concept $C$ is the set of concept
and role names that occur in $C$, and likewise for TBoxes $\T$, CIs $C
\sqsubseteq D$, assertions $r(a,b)$ and $A(a)$, ABoxes $\Amc$, KBs
$\K$, UCQs~$\q$. Note that the universal role is not regarded as a
role name, and so does not belong in any signature.  Similarly,
individual names are not in any signature and, in particular, not in
the signature of an assertion, ABox, or KB.  We are often interested
in concepts, TBoxes, KBs, and ABoxes that are formulated using a
specific signature. Therefore, we talk of a $\Sigma$-TBox \Tmc if
$\sig(\Tmc) \subseteq \Sigma$, and likewise for $\Sigma$-concepts,
etc.

There are several extensions of \ALC relevant for this paper, which
fall into three categories: extensions with (\emph{i})~additional
concept constructors, (\emph{ii})~additional role constructors, and
(\emph{iii})~additional types of statements in TBoxes. These
extensions are detailed in Table~\ref{tab:language-extensions}, where
$\# \mathcal{X}$ denotes the size of a set $\mathcal{X}$ and double
horizontal lines delineate different types of extensions. The last
column gives an identifier for each extension, which is simply
appended to the name \ALC for constructing extensions of \ALC. For
example, \ALC extended with number restrictions, inverse roles, and
the universal role is denoted by $\ALCQI^{u}$.

\begin{table}[tb]
  \begin{center}
    \small
    \leavevmode
    \begin{tabular}{|l|c|c|c|}
      \hline
      Name                       & Syntax           & Semantics & Identifier\\
      \hline\hline & &\\[-1em]
      number restrictions & $\qnrleq n r C$  & $\{d \mid \#\{e\mid  (d,e)\in r^{\Imc} \wedge e\in C^{\Imc} \}\leq n\}$ & \Qmc
       \\
       & $\qnrgeq n r C$  & $\{d \mid \#\{e\mid  (d,e)\in r^{\Imc} \wedge e\in C^{\Imc} \}\geq n\}$ & \\
      \hline\hline & &\\[-1em]
      inverse role               & \ $r^-$          & $\{(d,e) \mid (e,d) \in r^\Imc\}$ & \Imc\\
			universal role             & \ $u$            & $\Delta^{\Imc} \times \Delta^{\Imc}$ & $\cdot^{u}$\\
      \hline\hline & &\\[-1em]
      role inclusions (RIs) & $r \sqsubseteq s$ & $r^\Imc \subseteq s^\Imc$ & \Hmc \\
	\hline
    \end{tabular}\\ 
    \caption{Additional constructors: syntax and semantics.}
    \label{tab:language-extensions}
  \end{center}
\end{table}

We next define a number of syntactic fragments of \ALC and its extensions,
which often have dramatically lower computational complexity. The
fragment of \ALC obtained by disallowing the constructors $\bot$,
$\neg$, $\sqcup$ and $\forall$ is known as \EL. Thus, $\mathcal{EL}$
concepts are constructed using $\top$, $\sqcap$ and $\exists$ only
\cite{BaBL05}. We also consider extensions of $\mathcal{EL}$ with the
constructors in Table~\ref{tab:language-extensions}. For example,
$\mathcal{ELI}^{u}$ denotes the extension of $\mathcal{EL}$ with
inverse roles and the universal role. The fragments of $\ALCI$ and
$\ALCHI$, in which CIs are of the form
\begin{equation*}
B_1 \sqsubseteq B_2\qquad\text{ and }\qquad B_1 \sqcap B_2 \sqsubseteq \bot,
\end{equation*}
and the $B_i$ are concept names, $\top$, $\bot$ or $\exists r.\top$, are denoted by $\DLc$ and
$\DLcH$ (or $\DLite_{\mathcal{R}}$), respectively \cite{CDLLR07,ACKZ09}. 
\begin{example}\label{alc-fragments} The CI $\forall {\sf
    childOf}^-.{\sf Tall} \sqsubseteq {\sf Tall}$ (saying that
  everyone with only tall parents is also tall) is in $\ALCI$ but not in  $\ALC$, $\mathcal{EL}$ or $\DLcH$. The RI ${\sf childOf}^- \sqsubseteq {\sf parentOf}$ is in both $\ALCHI$ and $\DLcH$.
\end{example}

$\mathcal{EL}$ and the {\sl DL-Lite} logics introduced above are examples of Horn DLs, that is, fragments of DLs in the $\mathcal{ALC}$
family that restrict the syntax in such a way that conjunctive query
answering (see below) becomes tractable in data
complexity. A few
additional Horn DLs have become important in recent years. Following
\cite{HuMS07,Kaza09}, we say that a concept $C$ occurs positively in
$C$ itself and, if $C$ occurs positively (negatively) in $C'$, then
\begin{itemize}
\item[--] $C$ occurs positively (respectively, negatively) in $C' \sqcup D$,
$C' \sqcap D$, $\exists r.C'$, $\forall r.C'$, $D \sqsubseteq C'$, and
\item[--] $C$ occurs
negatively (respectively, positively) in $\neg C'$ and $C' \sqsubseteq D$.
\end{itemize}
Now, we call a TBox $\T$ \emph{Horn} if no concept of the form $C
\sqcup D$ occurs positively in $\T$, and no concept of the form $\neg
C$ or $\forall R.C$ occurs negatively in $\T$. For any DL $\Lmc$ from
the $\mathcal{ALC}$ family introduced above (e.g., $\ALCHI$), the DL
\textsl{Horn-}$\Lmc$ only allows for Horn TBoxes in $\Lmc$.  Note that
$\forall \mn{childOf}^-.{\sf Tall}$ occurs negatively in the CI
$\alpha$ from Example~\ref{alc-fragments}, and so the TBox $\T = \{
\alpha \}$ is not Horn.

TBoxes $\Tmc$ used in practice often turn out to be \emph{acyclic} in the following sense:
\begin{itemize}

\item all CEs in $\Tmc$ are of the form $A \equiv C$ (\emph{concept
    definitions}) and all CIs in~$\Tmc$ are of the form $A \sqsubseteq
  C$ (\emph{primitive concept inclusions}), where $A$ is a concept name;

\item no concept name occurs more than once on the left-hand side of
  a statement in $\Tmc$;

\item $\Tmc$ contains no cyclic definitions, as detailed below.
\end{itemize}
Let \Tmc be a TBox that contains only concept definitions and
primitive concept inclusions. The relation ${\prec_{\Tmc}} \subseteq
\NC \times \mn{sig}(\Tmc)$ is defined by setting $A\prec_{\Tmc}X$ if
there exists a TBox statement $A \bowtie C$ such that $X$ occurs in
$C$, where $\bowtie$ ranges over $\{ {\sqsubseteq},{\equiv} \}$.  A
concept name $A$ \emph{depends} on a symbol $X \in \NC \cup \NR$ if $A
\prec_{\Tmc}^{+} X$, where $\cdot^+$ denotes transitive closure.  We
use $\mn{depend}_\Tmc(A)$ to denote the set of all symbols $X$ such
that $A$ depends on $X$. We can now make precise what it means for
\Tmc to \emph{contain no cyclic definitions}: $A \not\in
{\sf depend}_{\Tmc}(A)$, for all $A \in \NC$. Note that the TBox $\T = \{ \alpha \}$ with $\alpha$ from Example~\ref{alc-fragments} is cyclic.


In DL, data is represented in the form of ABoxes. To introduce ABoxes,
we fix a set \NI of \emph{individual names}, which correspond to
constants in first-order logic. An \emph{assertion} is an
expression of the form $A(a)$ or $r(a,b)$, where $A$ is a concept
name, $r$ a role name, and $a,b$ individual names.  An
\emph{ABox} \Amc is just a finite set of assertions. We call the pair
$\Kmc=(\Tmc,\Amc)$ of a TBox $\Tmc$ in a DL $\Lmc$ and an ABox $\Amc$
an \emph{$\Lmc$ knowledge base} (KB, for short).  By ${\sf ind}(\A)$
and ${\sf ind}(\K)$, we denote the set of individual names in $\A$ and
$\K$, respectively.

To interpret ABoxes $\Amc$, we consider interpretations $\Imc$ that map all
individual names $a\in {\sf ind}(\A)$ to elements $a^{\Imc}\in \Delta^{\Imc}$ in
such a way that $a^{\Imc}\not=b^{\Imc}$ if $a\not=b$ (thus, we adopt the
\emph{unique name assumption}).  We say that $\Imc$ \emph{satisfies} an
assertion $A(a)$ if $a^\Imc \in C^\Imc$, and $r(a,b)$ if
$(a^\Imc,b^\Imc) \in r^\Imc$.  It is a \emph{model} of an ABox \Amc if it
satisfies all assertions in \Amc, and of a KB $\Kmc=(\Tmc,\Amc)$ if it is a model
of both \Tmc and \Amc. We say that \Kmc is \emph{consistent} (or
\emph{satisfiable}) if it has a model.  We use the terminology introduced for
TBoxes for KBs as well. For example, KBs $\K_{1}$ and $\K_{2}$ are
\emph{logically equivalent} if they have the same models (or, equivalently,
entail each other).

We next introduce query answering for KBs, beginning with conjunctive queries~\cite{GliHoLuSa-JAIR08,CalvaneseGLLR06,CalvaneseE007}.
An \emph{atom} is of the form $A(x)$ or $r(x,y)$, where $x,y$ are
from a set of \emph{variables} \NV, $A$ is a concept name, and $r$  a role name.
A \emph{conjunctive query} (or CQ) is an expression of the form $\q(\vec{x}) = \exists
\vec{y} \, \vp(\vec{x},\vec{y})$, where $\vec{x}$ and $\vec{y}$ are
disjoint sequences of variables and $\vp$ is a conjunction of atoms
that only contain variables from $\vec{x} \cup \vec{y}$---we (ab)use
set-theoretic notation for sequences where
convenient. We often write $A(x) \in \q$ and $r(x,y) \in \q$ to
indicate that $A(x)$ and $r(x,y)$ are conjuncts of $\varphi$. We call
a CQ $\q$ \emph{rooted} (\RCQ) if every $y \in \avec{y}$ is connected
to some $x \in \avec{x}$ by a path in the undirected graph whose nodes are the variables in $\q$ and edges are the pairs $\{u,v\}$ with $r(u,v)\in\q$, for some $r$.
A \emph{union of CQs} (UCQ) is a disjunction $\q(\avec{x}) = \bigvee_i \q_i(\avec{x})$ of CQs $\q_i(\avec{x})$ with the same \emph{answer variables} $\avec{x}$; it is \emph{rooted} (\RUCQ) if all the $\q_i$ are rooted. If the sequence $\avec{x}$ is empty, $\q(\avec{x})$ is called a \emph{Boolean} CQ or UCQ.

Given a UCQ $\q(\vec{x}) = \bigvee_i \q_i(\avec{x})$ and a KB \Kmc,
a sequence $\vec{a}$ of individual names from $\Kmc$ of the same  length as $\vec{x}$ is called a
\emph{certain answer to $\q(\vec{x})$ over \Kmc} if, for every model \Imc of \Kmc, there exist a CQ $\q_i$ in $\q$ and a map (homomorphism) $h$ of its variables to $\Delta^\Imc$ such that
\begin{itemize}
\item if $x$ is the $j$-th element of $\vec{x}$ and $a$ the $j$-th element of $\vec{a}$, then $h(x)=a^\Imc$;

\item $A(z) \in \q$ implies $h(z) \in A^\Imc$, and $r(z,z') \in \q$ implies $(h(z),h(z')) \in r^\Imc$.
\end{itemize}
If this is the case, we write $\K \models\q(\vec{a})$. For a Boolean
UCQ $\q$, we also say that the certain answer over \Kmc is `yes' if
$\K \models\q$ and `no' otherwise.  \emph{CQ} or \emph{UCQ answering}
means to decide, given a CQ or UCQ $\q(\vec{x})$, a KB $\Kmc$ and a
tuple $\vec{a}$ from $\ind(\K)$, whether $\K \models\q(\vec{a})$.



\section{Inseparability}\label{sec:inseparability}

Since there is no single inseparability relation between ontologies
that is appropriate for all applications, we start by identifying
basic properties that any semantic notion of inseparability between
TBoxes or KBs should satisfy. 
We also introduce notation that
will be used throughout the survey and discuss a few
applications of inseparability.  

%
%

For uniformity, we assume that the term `ontology' refers to both
TBoxes and KBs.

\begin{definition}[inseparability]
\label{def:inseprel}
\em
Let $\mathcal{S}$ be the set of ontologies \textup{(}either TBoxes or
KBs\textup{)} formulated in a description logic $\mathcal{L}$. A
map that assigns to each  signature $\Sigma$ 
an equivalence relation $\equiv_{\Sigma}$ on $\mathcal{S}$ is an
\emph{inseparability relation on} $\mathcal{S}$ if the following conditions hold:
\begin{description}
\item[\rm (\emph{i})]
if $\mathcal{O}_{1}$ and $\mathcal{O}_{2}$
are logically equivalent, then $\mathcal{O}_{1} \equiv_{\Sigma}
\mathcal{O}_{2}$, for all signatures $\Sigma$ and
$\Omc_1,\Omc_2 \in \mathcal{S}$;

\item[\rm (\emph{ii})]$\Sigma_{1}\subseteq \Sigma_{2}$ implies $\equiv_{\Sigma_{1}}\;\supseteq\; \equiv_{\Sigma_{2}}$, for all
finite signatures $\Sigma_{1}$ and $\Sigma_{2}$.
\end{description}
\end{definition}

By condition (\emph{i}), an inseparability relations does not depend
on the syntactic presentation of an ontology, but only on its
semantics. Condition (\emph{ii}) formalizes the requirement that if
the set of relevant symbols increases
$(\Sigma_{2} \supseteq \Sigma_{1}$), then more ontologies become
separable. Depending on the intended application, additional
properties may also be required. For example, we refer the reader to
\cite{DBLP:series/lncs/KonevLWW09} for a detailed discussion of
robustness properties that are relevant for applications to modularity.
We illustrate inseparability relations by three very basic examples.

\begin{example}
\label{ex:simple}
  (1) Let $\mathcal{S}$ be the set of ontologies formulated in a
  description logic $\mathcal{L}$, and let
  $\mathcal{O}_{1} \equiv_{\sf equiv} \mathcal{O}_{2}$ if and only if
  $\mathcal{O}_{1}$ and $\mathcal{O}_{2}$ are logically equivalent,
  for any $\mathcal{O}_{1},\mathcal{O}_{2}\in \mathcal{S}$. Then
  $\equiv_{\sf equiv}$ is an inseparability relation that does not
  depend on the concrete signature. It is the finest inseparability
  relation possible. The inseparability relations considered in this
  survey are more coarse.

(2) Let $\mathcal{S}$ be the set of KBs in a description logic $\mathcal{L}$, and
let $\mathcal{K}_{1} \equiv_{\sf sat} \mathcal{K}_{2}$
if and only if $\mathcal{K}_{1}$ and $\mathcal{K}_{2}$ are  equisatisfiable, for any $\mathcal{K}_{1},\mathcal{K}_{2}\in \mathcal{S}$. Then $\equiv_{\sf sat}$ is
another inseparability relation that does not depend on the concrete
signature. It has
two equivalence classes---the satisfiable KBs and the unsatisfiable KBs---and is not sufficiently
fine-grained for most applications.

(3) Let $\mathcal{S}$ be the set of TBoxes in a description logic $\mathcal{L}$, and let $\mathcal{T}_{1} \equiv^{\sf hierarchy}_{\Sigma} \mathcal{T}_{2}$ if and only if
$$
\mathcal{T}_{1} \models A \sqsubseteq B \quad \Longleftrightarrow \quad \mathcal{T}_{2}\models A\sqsubseteq B, \quad \text{for all concept names $A,B\in \Sigma$}.
$$
Then each relation $\equiv^{\sf hierarchy}_{\Sigma}$ is an
inseparability relation. It distinguishes between two TBoxes if and
only if they do not entail the same subsumption hierarchy over the
concept names in $\Sigma$, and it is appropriate for applications that
are only concerned with subsumption hierarchies such as producing a
systematic catalog of vocabulary items, which is in fact the prime use
of {\sc Snomed CT}\footnote{\url{http://www.ihtsdo.org/snomed-ct}}\!.
\end{example}

As discussed in the introduction, the inseparability relations
considered in this paper are more sophisticated than those in
Example~\ref{ex:simple}. Details are given in the subsequent
sections. We remark that some versions of query inseparability
that we are going to consider are, strictly speaking, not covered
by Definition~\ref{def:inseprel} since two signatures are involved
(one for the query and one for the data). However, it is easy to
extend Definition~\ref{def:inseprel} accordingly.

We now present some important applications
of inseparability.

\smallskip
\noindent
\textit{\textbf{Versioning.}} Maintaining and updating ontologies is
very difficult without tools that support versioning. One can
distinguish \emph{three} approaches to
versioning~\cite{DBLP:journals/jair/KonevL0W12}: versioning based on
syntactic difference (syntactic diff), versioning based on structural
difference (structural diff), and versioning based on logical
difference (logical diff).  The \emph{syntactic diff} underlies most
existing version control systems used in software
development~\cite{ConradiWestfechtel98} such as RCS, CVS, SCCS. It
works with text files and represents the difference between versions
as blocks of text present in one version but not in the other.
As observed in~\cite{promptdiff}, ontology versioning cannot rely on a
purely syntactic diff operation since many syntactic differences
(e.g., the order of ontology axioms) do not affect the semantics.  The
\emph{structural diff} extends the syntactic diff by taking into
account information about the structure of
ontologies. 
Its main characteristic is that it regards ontologies as structured
objects, such as an \emph{is-a} taxonomy~\cite{promptdiff}, a set of
RDF triples~\cite{OntoView} or a set of class defining
axioms~\cite{RSDT08,HORROCKS}. Though helpful, the structural diff
still has no unambiguous semantic foundation and is syntax
dependent. Moreover, it is tailored towards applications of ontologies
that are based on the induced concept hierarchy (or some mild
extension thereof), but does not capture other applications such as
querying data under ontologies. In contrast, the \emph{logical diff}
\cite{KWW08,WoZakr08} completely abstracts away from the presentation
of the ontology and regards two versions of an ontology as identical
if they are inseparable with respect to an appropriate inseparability
relation such as concept inseparability or query inseparability.  The
result of the logical diff is then presented in terms of witnesses for
separability. 

\smallskip
\noindent
\textit{\textbf{Ontology refinement.}} When extending an ontology with
new concept inclusions or other statements, one usually wants to
preserve the semantics of a large part $\Sigma$ of its vocabulary.
For example, when extending {\sc Snomed CT} with 50 additional concept
names on top of the more than 300K existing ones, one wants to ensure
that the meaning of unrelated parts of the vocabulary does not
change. This preservation problem is formalized 
by demanding that the original ontology $\mathcal{O}_{\sf original}$
and the extended ontology
$\mathcal{O}_{\sf original}\cup \mathcal{O}_{\sf add}$ are
$\Sigma$-inseparable (for an appropriate notion of inseparability)
\cite{GhilardiLutzWolter-KR06}. It should be noted that ontology
refinement can be regarded as a versioning problem as discussed above,
where $\mathcal{O}_{\sf original}$ and
$\mathcal{O}_{\sf original}\cup \mathcal{O}_{\sf add}$ are versions of
an ontology that have to be compared.

\smallskip
\noindent
\textit{\textbf{Ontology reuse.}} A frequent operation in ontology
engineering is to import an existing ontology $\mathcal{O}_{\sf im}$
into an ontology $\mathcal{O}_{\sf host}$ that is currently being developed, with the
aim of reusing the vocabulary of $\mathcal{O}_{\sf im}$. Consider, for
example, a host ontology $\mathcal{O}_{\sf host}$ describing  research
projects that imports an ontology $\mathcal{O}_{\sf im}$, which defines medical terms
$\Sigma$ to be used in the definition of research projects
in $\mathcal{O}_{\sf host}$.  Then one typically wants to use the
medical terms $\Sigma$ exactly as defined in $\mathcal{O}_{\sf im}$.
However, using those terms to define concepts in $\mathcal{O}_{\sf
  host}$ might have unexpected consequences also for the terms
in $\mathcal{O}_{\sf im}$, that is, it might `damage' the modeling
of those terms. To avoid this, one wants to ensure that 
$\mathcal{O}_{\sf host}\cup \mathcal{O}_{\sf im}$ and
$\mathcal{O}_{\sf im}$ are $\Sigma$-inseparable
\cite{JairGrau}. Again, this can be regarded as a versioning problem
for the ontology $\mathcal{O}_{\sf im}$.

\smallskip
\noindent
\textit{\textbf{Modularity.}} Modular ontologies and the extraction of
modules are an important ontology engineering challenge
\cite{DBLP:series/lncs/5445,DBLP:journals/lu/KutzML10}.  Understanding
$\Sigma$-inseparability of ontologies is crucial for most approaches
to this problem. For example, a very natural and popular definition of
a module $\mathcal{M}$ of an ontology $\mathcal{O}$ demands that
$\mathcal{M}\subseteq \mathcal{O}$ and that $\mathcal{M}$ is
$\Sigma$-inseparable from $\mathcal{O}$ for the signature $\Sigma$ of
$\mathcal{M}$ (called self-contained module).  Under this definition,
the ontology $\mathcal{O}$ can be safely replaced by the module
$\mathcal{M}$ in the sense specified by the inseparability relation
and as far as the signature $\Sigma$ of $\mathcal{M}$ is concerned.  A
stronger notion of module (called depleting module \cite{KWZ10})
demands that $\mathcal{M} \subseteq \mathcal{O}$ and that
$\mathcal{O}\setminus\mathcal{M}$ is $\Sigma$-inseparable from the
empty ontology for the signature $\Sigma$ of $\mathcal{M}$.  The
intuition is that the ontology statements outside of $\mathcal{M}$
should not say anything non-tautological about signature items in the
module $\mathcal{M}$.

\smallskip
\noindent
\textit{\textbf{Ontology mappings.}} The construction of mappings (or
alignments) between ontologies is an important challenge in ontology
engineering and integration
\cite{DBLP:journals/tkde/ShvaikoE13}. Given two ontologies
$\mathcal{O}_{1}$ and $\mathcal{O}_{2}$ in different signatures
$\Sigma_{1}$ and $\Sigma_{2}$, the problem is to align the vocabulary
items in $\Sigma_{1}$ with those in $\Sigma_{2}$ using a
TBox $\Tmc_{12}$ that states logical relationships between
$\Sigma_{1}$ and $\Sigma_{2}$. For example, $\Tmc_{12}$ could consist
of statements of the form $A_{1} \equiv A_{2}$ or
$A_{1} \sqsubseteq A_{2}$, where $A_{1}$ is a concept name in
$\Sigma_{1}$ and $A_{2}$ is a concept name in $\Sigma_{2}$. When
constructing such mappings, we typically do not want one ontology to
interfere with the semantics of the other ontology via the mapping
\cite{DBLP:conf/semweb/SolimandoJG14,DBLP:conf/semweb/KharlamovHJLLPR15,DBLP:conf/kr/Jimenez-RuizPST16}.
This condition can (and has been) formalized using inseparability. In
fact, the non-interference requirement can be given by the condition
that $\Omc_{i}$ and $\Omc_{1} \cup \Omc_{2}\cup \Tmc_{12}$ are
$\Sigma_{i}$ inseparable, for $i=1,2$.

\smallskip
\noindent
\textit{\textbf{Knowledge base exchange.}} This application is a natural extension  of  \emph{data exchange}~\cite{DBLP:journals/tcs/FaginKMP05}, where the task is to transform a data instance $\mathcal{D}_1$ structured under a source schema $\Sigma_1$ into a data instance $\mathcal{D}_2$ structured under a target schema $\Sigma_2$ given a mapping $\mathcal{M}_{12}$ relating $\Sigma_1$ and $\Sigma_2$. 
In knowledge base exchange~\cite{DBLP:aij-kb-exchange}, we are interested in  translating a KB $\K_1$ in a source signature $\Sigma_1$ to a KB $\K_2$ in a target signature $\Sigma_2$ according to a mapping given by a TBox $\T_{12}$ that 
consists of CIs and RIs in $\Sigma_{1}\cup \Sigma_{2}$ defining concept and role names in $\Sigma_{2}$ in terms of concepts and roles in $\Sigma_{1}$. 
A good solution to this problem can be viewed as a KB $\Kmc_{2}$ that it is inseparable from $\Kmc_{1} \cup \Tmc_{12}$ with respect to a suitable $\Sigma_2$-inseparability relation.

%

\smallskip
\noindent
\textit{\textbf{Forgetting and uniform interpolation.}} When adapting
an ontology to a new application, it is often useful to eliminate
those symbols in its signature that are not relevant for the new
application while retaining the semantics of the remaining ones.
Another reason for eliminating symbols is predicate hiding, i.e., an
ontology is to be published, but some part of it should be concealed
from the public because it is confidential
\cite{DBLP:journals/jair/GrauM12}. Moreover, one can view the
elimination of symbols as an approach to ontology summary: the smaller
and more focussed ontology summarizes what the original ontology says
about the remaining signature items. The idea of eliminating symbols
from a theory has been studied in AI under the name of forgetting a
signature $\Sigma$ 
\cite{Lin}.
In mathematical logic and modal logic, forgetting has been
investigated under the dual notion of uniform interpolation
\cite{Pitts,DagostinoHollenberg2,Visser,GZ,French,DBLP:journals/jair/SuSLZ09}. Under both names, the
problem has been studied extensively in DL research \cite{DBLP:conf/ijcai/KonevWW09,DBLP:journals/amai/WangWTP10,DBLP:conf/ijcai/LutzW11,DBLP:journals/ci/WangWTPA14,DBLP:conf/cade/KoopmannS14,DBLP:journals/ai/NikitinaR14,DBLP:conf/aaai/KoopmannS15}. Using
inseparability, we can formulate the condition that the result
$\mathcal{O}_{\sf forget}$ of eliminating $\Sigma$ from $\mathcal{O}$ should not
change the semantics of the remaining symbols by demanding that $\mathcal{O}$
and $\mathcal{O}_{\sf forget}$ are ${\sf sig}(\mathcal{O})\setminus\Sigma$-inseparable for the signature
${\sf sig}(\mathcal{O})$ of $\mathcal{O}$.



\newcommand{\Amf}{\ensuremath{\mathfrak{A}}\xspace}

\section{Concept Inseparability}
\label{sect:separability}

We consider inseparability relations that distinguish TBoxes if and
only if they do not entail the same concept inclusions in a selected
signature.\footnote{For DLs that admit role inclusions, one
  additionally considers entailment of these.} The resulting
\emph{concept inseparability} relations are appropriate for
applications that focus on TBox reasoning.
%
%
We start by defining concept inseparability and the related notions of 
concept entailment and concept conservative extensions. We give
illustrating examples and discuss the relationship between the three
notions and their connection to logical equivalence. We then take a
detailed look at concept inseparability in $\mathcal{ALC}$ and in
\EL. In both cases, we first establish a model-theoretic
characterization and then show how this characterization can be used 
to decide concept entailment with the help of automata-theoretic techniques.
We also briefly discuss extensions of \ALC and the special case of
\EL with acyclic TBoxes.



\begin{definition}[concept inseparability, entailment and conservative extension]\label{def1}\em
Let $\Tmc_{1}$ and $\Tmc_{2}$ be TBoxes formulated in some DL $\Lmc$,
and let $\Sigma$ be a signature. Then
\begin{itemize}
\item the \emph{${\Sigma}$-concept difference between $\Tmc_{1}$ and
    $\Tmc_{2}$} is the set ${\sf cDiff}_{\Sigma}(\Tmc_{1},\Tmc_{2})$
  of all ${\Sigma}$-concept inclusions (and role inclusions, if
  admitted by \Lmc) $\alpha$ that are formulated in \Lmc and satisfy
  $\Tmc_{2}\models \alpha$ and $\Tmc_{1}\not\models\alpha$;

\item $\Tmc_1$ \emph{${\Sigma}$-concept entails} $\Tmc_2$ if
      ${\sf cDiff}_{\Sigma}(\Tmc_{1},\Tmc_{2})= \emptyset$;

\item $\Tmc_1$ and $\Tmc_2$ are \emph{${\Sigma}$-concept inseparable}
      if $\Tmc_{1}$ ${\Sigma}$-concept entails $\Tmc_{2}$ and vice
      versa;

\item $\Tmc_{2}$ is a \emph{concept conservative extension of} $\Tmc_{1}$
      if $\Tmc_{2} \supseteq \Tmc_{1}$ and $\Tmc_{1}$ and $\Tmc_{2}$ are ${\sig(\Tmc_{1})}$-inseparable.
\end{itemize}
\end{definition}
We illustrate this definition by a number of examples.

\begin{example}[concept entailment vs.\ logical entailment]\label{ex:ex1}
  If $\Sigma \supseteq \mn{sig}(\Tmc_1 \cup \Tmc_2)$, then
  ${\Sigma}$-concept entailment is equivalent to logical entailment,
  that is, $\Tmc_{1}$ ${\Sigma}$-concept entails $\Tmc_{2}$ iff
  $\Tmc_{1} \models \Tmc_{2}$. We recommend the reader to verify
  that this is a straightforward consequence of the definitions (it is
  crucial to observe that, because of our assumption on $\Sigma$, the
  concept inclusions in $\Tmc_2$ qualify as potential members of ${\sf
    cDiff}_{\Sigma}(\Tmc_{1},\Tmc_{2})$).
%
\end{example}

\begin{example}[definitorial extension]\label{ex:ex2}
  An important way to extend an ontology is to introduce definitions
  of new concept names. Let $\Tmc_{1}$ be a TBox, say formulated in
  \ALC, and let $\Tmc_2=\{ A\equiv C\} \cup \Tmc_{1}$, where $A$ is a
  fresh concept name.
  Then $\Tmc_{2}$ is called a \emph{definitorial extension} of
  $\Tmc_{1}$.  Clearly, unless $\Tmc_{1}$ is inconsistent, we have
  $\Tmc_{1}\not\models \Tmc_{2}$.  However, $\Tmc_{2}$ is a
  concept-conservative extension of $\Tmc_{1}$.  For the proof, assume
  that $\Tmc_{1}\not\models \alpha$ and ${\sf sig}(\alpha) \subseteq
  {\sf sig}(\Tmc_{1})$. We show that
  $\Tmc_{2}\not\models\alpha$. There is a model $\Imc_{1}$ of
  $\Tmc_{1}$ such that $\Imc\not\models \alpha$. Modify \Imc by
  setting $A^{\Imc}=C^{\Imc}$.  Then, since $A\not\in{\sf
    sig}(\Tmc_{1})$, the new $\Imc$ is still a model of $\Tmc_{1}$ and
  we still have $\Imc\not\models\alpha$. Moreover, $\Imc$ satisfies
  $A\equiv C$, and thus is a model of $\Tmc_{2}$. Consequently,
  $\Tmc_{2}\not\models \alpha$.
\end{example}

The notion of concept inseparability depends on the DL in which the
separating concept inclusions can be formulated. Note that, in
Definition~\ref{def1}, we assume that this DL is the one in which the
original TBoxes are formulated. Throughout this paper, we will thus
make sure that the DL we work with is always clear from the
context. We illustrate the difference that the choice of the `separating DL'
can make by two examples.
\begin{example}\label{ALCQ-ALC}
Consider the $\mathcal{ALC}$ TBoxes
$$
\Tmc_{1}=\{A \sqsubseteq \exists r.\top\} \quad \text{and} \quad 
\Tmc_{2}=\{A \sqsubseteq \exists r.B\sqcap \exists r.\neg B\}
$$
and the signature $\Sigma = \{A,r\}$. If we view $\Tmc_1$ and $\Tmc_2$
as \ALCQ TBoxes and consequently allow concept inclusions formulated
in \ALCQ to separate them, then $A \sqsubseteq (\geq 2 r.\top)\in {\sf
  cDiff}_{\Sigma}(\Tmc_{1},\Tmc_{2})$, and so $\Tmc_{1}$ and $\Tmc_{2}$
are ${\Sigma}$-concept separable.  However, $\Tmc_{1}$ and $\Tmc_{2}$
are ${\Sigma}$-concept inseparable when we only allow separation in
terms of \ALC-concept inclusions.  Intuitively, this is the case
because, in $\mathcal{ALC}$, one cannot count the number of
$r$-successors of an individual. We will later introduce the
model-theoretic machinery required to prove such statements in a
formal way.
\end{example}
\begin{example}\label{ex:e7}
Consider the $\mathcal{EL}$ TBoxes
\begin{align*}
\Tmc_{1} &= \{{\sf Human} \sqsubseteq \exists {\sf eats}.\top, \ 
{\sf Plant} \sqsubseteq \exists {\sf grows\_in}.{\sf Area}, \ 
{\sf Vegetarian} \sqsubseteq {\sf Healthy}\},\\
\Tmc_{2} &= \Tmc_{1} \cup
\{{\sf Human} \sqsubseteq \exists {\sf eats}.{\sf Food}, \ 
{\sf Food} \sqcap {\sf Plant} \sqsubseteq {\sf Vegetarian}\}.
\end{align*}
It can be verified that
$$
{\sf Human} \sqcap \forall {\sf eats}.{\sf Plant} \sqsubseteq \exists {\sf eats}.{\sf Vegetarian}
$$
is entailed by $\Tmc_{2}$ but not by $\Tmc_{1}$. If we view $\Tmc_1$
and $\Tmc_2$ as \ALC TBoxes, then $\Tmc_{2}$ is thus not a
concept conservative extension of $\Tmc_1$.  However, we will show later that
if we view $\Tmc_{1}$ and $\Tmc_{2}$ as $\mathcal{EL}$ TBoxes, then $\Tmc_{2}$
is a concept conservative extension of $\Tmc_{1}$ (i.e., $\Tmc_{1}$ and $\Tmc_{2}$
are $\mn{sig}(\Tmc_1)$-inseparable in terms of $\mathcal{EL}$-concept inclusions).
\end{example}

As remarked in the introduction, conservative extensions are a special case
of both inseparability and entailment. The former is by definition and
the latter since $\Tmc_2$ is a concept conservative extension of
$\Tmc_1 \subseteq \Tmc_2$ iff $\Tmc_1$ $\mn{sig}(\Tmc_1)$-entails
$\Tmc_2$.  Before turning our attention to specific DLs, we discuss a
bit more the relationship between entailment and inseparability. On
the one hand, inseparability is defined in terms of entailment and
thus inseparability can be decided by two entailment checks. One
might wonder about the converse direction, i.e., whether entailment
can be reduced in some natural way to inseparability. This question
is related to the following robustness condition.
\begin{definition}[robustness under joins]\label{robustjoin}\em 
  A DL $\Lmc$ is \emph{robust under joins} for concept inseparability
  if, for all $\mathcal{L}$ TBoxes $\Tmc_{1}$ and $\Tmc_{2}$ and
  signatures $\Sigma$ with ${\sf sig}(\Tmc_{1}) \cap {\sf
    sig}(\Tmc_{2}) \subseteq \Sigma$, the following are equivalent:
\begin{description}
\item[\rm (\emph{i})] $\Tmc_{1}$ $\Sigma$-concept entails $\Tmc_{2}$ in $\Lmc$;
\item[\rm (\emph{ii})] $\Tmc_{1}$ and $\Tmc_{1}\cup \Tmc_{2}$ are $\Sigma$-concept inseparable in $\Lmc$.
\end{description}
\end{definition}

Observe that the implication (\emph{ii})~$\Rightarrow$~(\emph{i}) is
trivial. The converse holds for many DLs such as $\mathcal{ALC}$,
$\mathcal{ALCI}$ and $\mathcal{EL}$;
see~\cite{DBLP:series/lncs/KonevLWW09} for details. However, there are
also standard DLs such as $\mathcal{ALCH}$ for which robustness under
joins fails.
\begin{theorem}
\label{thm:pipapo}
If a DL $\Lmc$ is robust under joins for concept inseparability, then concept entailment in $\Lmc$
can be polynomially reduced to concept inseparability in~$\Lmc$.
\end{theorem}
\begin{proof}
  Assume that we want to decide whether $\Tmc_{1}$ ${\Sigma}$-concept
  entails $\Tmc_{2}$. By replacing every non-$\Sigma$-symbol $X$
  shared by $\Tmc_{1}$ and $\Tmc_{2}$ with a fresh symbol $X_1$ in
  $\Tmc_1$ and a distinct fresh symbol $X_2$ in $\Tmc_2$, we can
  achieve that $\Sigma \supseteq \sig(\Tmc_{1}) \cap \sig(\Tmc_{2})$
  without changing (non-)${\Sigma}$-concept entailment of $\Tmc_2$ by
  $\Tmc_1$. We then have, by robustness under joins, that $\Tmc_{1}$
  $\Sigma$-concept entails $\Tmc_{2}$ iff $\Tmc_{1}$ and
  $\Tmc_{1}\cup \Tmc_{2}$ are $\Sigma$-concept inseparable.
\qed
\end{proof}

For DLs $\Lmc$ that are not robust under joins for concept inseparability (such as $\mathcal{ALCH}$) 
it has not yet been investigated whether there exist natural polynomial reductions of concept entailment
to concept inseparability. 



\subsection{Concept inseparability for $\mathcal{ALC}$}
\label{sect:concinsepalc}

We first give a model-theoretic characterization of concept entailment
in \ALC in terms of bisimulations and then show how this characterization
can be used to obtain an algorithm for deciding concept entailment based
on automata-theoretic techniques. We also discuss the complexity,
which is 2\ExpTime-complete, and the size of minimal counterexamples
that witness inseparability. 

Bisimulations are a central tool for studying the expressive power of
\ALC and of modal logics; see for example
\cite{GorankoOtto,DBLP:conf/ijcai/LutzPW11}.  By a \emph{pointed
  interpretation} we mean a pair $(\Imc,d)$, where $\Imc$ is an
interpretation and $d\in \Delta^{\Imc}$.
\begin{definition}[$\Sigma$-bisimulation]\em
  Let $\Sigma$ be a finite signature and $(\Imc_{1},d_{1})$ and 
  $(\Imc_{2},d_{2})$ pointed interpretations. A relation $S \subseteq
  \Delta^{\Imc_{1}}\times \Delta^{\Imc_{2}}$ is a
  \emph{$\Sigma$-bisimulation} between $(\Imc_{1},d_{1})$ and
  $(\Imc_{2},d_{2})$ if $(d_{1},d_{2})\in S$ and, for all $(d,d') \in
  S$, the following conditions are satisfied:
\begin{description}
\item[\rm (base)] $d \in A^{\Imc_{1}}$ iff $d' \in A^{\Imc_{2}}$, for all $A \in \Sigma\cap \NC$;

  \item[\rm (zig)] if $(d,e) \in r^{\Imc_{1}}$, then there exists $e'\in \Delta^{\Imc_{2}}$ such
that $(d',e') \in r^{\Imc_{2}}$ and $(e,e')\in S$, for all $r \in \Sigma\cap \NR$;

  \item[\rm (zag)] if $(d',e') \in r^{\Imc_{2}}$, then there exists $e\in \Delta^{\Imc_{1}}$ such
that $(d,e) \in r^{\Imc_{1}}$ and $(e,e')\in S$, for all $r \in \Sigma\cap \NR$.
\end{description}
We say that $(\Imc_{1},d_{1})$ and $(\Imc_{2},d_{2})$ are
\emph{$\Sigma$-bisimilar} and write $(\Imc_{1},d_{1}) \sim_\Sigma^{{\sf bisim}}
(\Imc_{2},d_{2})$ if there exists a $\Sigma$-bisimulation between them.
\end{definition}

We now recall the main connection between bisimulations and \ALC.  Say
that $(\Imc_{1},d_{1})$ and $(\Imc_{2},d_{2})$ are
\emph{$\ALC_{\Sigma}$-equivalent}, in symbols
$(\Imc_{1},d_{1})\equiv_{\Sigma}^{\mathcal{ALC}}(\Imc_{2},d_{2})$, in
case $d_{1}\in C^{\Imc_{1}}$ iff $d_{2}\in C^{\Imc_{2}}$ for all
$\Sigma$-concepts $C$ in $\mathcal{ALC}$. An interpretation $\mathcal{I}$ is of
\emph{finite outdegree} if the set $\{ d' \mid (d,d') \in \bigcup_{r
  \in \NR} r^\Imc \}$ is finite, for any \mbox{$d\in
  \Delta^{\mathcal{I}}$}.
\begin{theorem}\label{theorem:bisimchar}
Let $(\Imc_{1},d_{1})$ and $(\Imc_{2},d_{2})$ be pointed
interpretations
and $\Sigma$ a signature. Then 
  $(\Imc_{1},d_{1}) \sim_\Sigma^{{\sf bisim}} (\Imc_{2},d_{2})$ implies
  $(\Imc_{1},d_{1})\equiv_{\Sigma}^{\mathcal{ALC}}(\Imc_{2},d_{2})$.  The converse
  holds if $\Imc_1$ and $\Imc_2$ are of finite outdegree.
\end{theorem}

\begin{example}\label{boundedoutdegree}
  The following classical example shows that without the condition of
  finite outdegree, the converse direction does not hold.

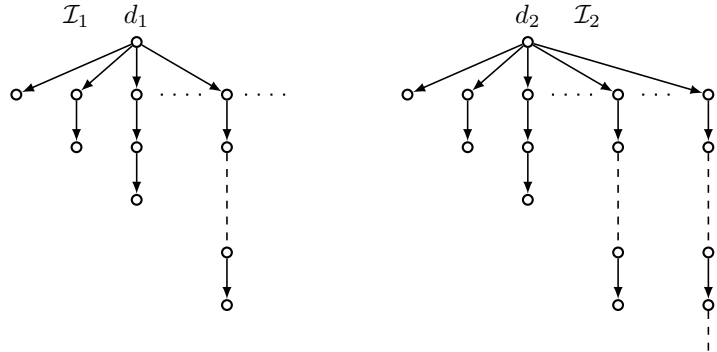
\begin{center}
  \begin{tikzpicture}[yscale=0.7, xscale=0.8]
    \begin{scope}
      \node at (-1,0.5) {$\Imc_1$};

      \foreach \al/\x/\y in {%
        d1/0/0,%
        x1/-2/-1,%
        x21/-1/-1,%
        x22/-1/-2,%
        x31/0/-1,%
        x32/0/-2,%
        x33/0/-3,%
        xn1/1.5/-1,%
        xn2/1.5/-2,%
        xnn1/1.5/-4,%
        xnn/1.5/-5%
      }{\node[point] (\al) at (\x,\y) {};}

      \node[anchor=south] at (d1.north) {$d_1$};

      \foreach \from/\to in { d1/x1,%
        d1/x21, x21/x22, %
        d1/x31, x31/x32, x32/x33,%
        d1/xn1, xn1/xn2, xnn1/xnn%
      }{ \draw[role] (\from) -- (\to); }

      \draw[dashed,semithick] (xn2) -- (xnn1);

      \draw[loosely dotted,thick] ($(x31)+(0.4,0)$) --
      ($(xn1)-(0.4,0)$) ($(xn1)+(0.3,0)$) -- +(0.8,0);
    \end{scope}

    \begin{scope}[xshift=6.5cm]
      \node at (1,0.5) {$\Imc_2$};

      \foreach \al/\x/\y in {%
        d2/0/0,%
        x1/-2/-1,%
        x21/-1/-1,%
        x22/-1/-2,%
        x31/0/-1,%
        x32/0/-2,%
        x33/0/-3,%
        xn1/1.5/-1,%
        xn2/1.5/-2,%
        xnn1/1.5/-4,%
        xnn/1.5/-5,%
        y1/3/-1,%
        y2/3/-2,%
        yn1/3/-4,%
        yn/3/-5%
      }{\node[point] (\al) at (\x,\y) {};}

      \node[anchor=south] at (d2.north) {$d_2$};

      \foreach \from/\to in { d2/x1,%
        d2/x21, x21/x22, %
        d2/x31, x31/x32, x32/x33,%
        d2/xn1, xn1/xn2, xnn1/xnn,%
        d2/y1, y1/y2, yn1/yn%
      }{ \draw[role] (\from) -- (\to); }

      \draw[dashed,semithick] (xn2) -- (xnn1)%
      (y2) -- (yn1) (yn) -- +(0,-1);

      \draw[loosely dotted,thick] ($(x31)+(0.4,0)$) --
      ($(xn1)-(0.4,0)$) ($(xn1)+(0.4,0)$) -- ($(y1)-(0.5,0)$);
    \end{scope}
  \end{tikzpicture}
\end{center}
Here, $(\Imc_1,d_1)$ is a pointed interpretation with an $r$-chain of
length $n$ starting from $d_1$, for each $n \geq 1$. $(\Imc_2,d_2)$
coincides with $(\Imc_1,d_1)$ except that it also contains an infinite $r$-chain
starting from $d_2$. Let $\Sigma=\{r\}$. It can be proved that
$(\Imc_{1},d_{1})\equiv_{\Sigma}^{\mathcal{ALC}}(\Imc_{2},d_{2})$. However,
$(\Imc_{1},d_{1}) \not\sim_\Sigma^{{\sf bisim}} (\Imc_{2},d_{2})$ due
to the infinite chain in $(\Imc_2,d_2)$.
\end{example}

As a first application of Theorem~\ref{theorem:bisimchar}, we note
that $\mathcal{ALC}$ cannot distinguish between an interpretation and
its \emph{unraveling} into a tree. An interpretation $\Imc$ is called
a \emph{tree interpretation} if $r^{\Imc} \cap s^{\Imc}=\emptyset$ for
any $r\not=s$ and the directed graph $(\Delta^{\Imc},\bigcup_{r\in
  \NR}r^{\Imc})$ is a (possibly infinite) tree. The root of $\Imc$ is
denoted by $\rho^{\Imc}$. By the unraveling technique
\cite{GorankoOtto}, one can show that every pointed interpretation
$(\Imc,d)$ is $\Sigma$-bisimilar to a pointed tree interpretation
$(\Imc^{\ast},\rho^{\Imc^{\ast}})$, for any signature~$\Sigma$. Indeed, suppose $(\Imc,d)$ is given. The domain
$\Delta^{\Imc^{\ast}}$ of $\Imc^{\ast}$ is the set of words
$w=d_{0}r_{0}d_{1}\cdots r_{n}d_{n}$ such that $d_{0}=d$ and
$(d_{i},d_{i+1})\in r_{i}^{\Imc}$ for all $i<n$ and roles names~$r_{i}$.  We set ${\sf tail}(d_{0}r_{0}d_{1}\cdots r_{n}d_{n})=d_{n}$
and define the interpretation $A^{\Imc^{\ast}}$ and $r^{\Imc^{\ast}}$
of concept names $A$ and role names $r$ by setting:
\begin{itemize}
\item $w\in A^{\Imc^{\ast}}$ if ${\sf tail}(w)\in A^{{\Imc}^{\ast}}$;
\item $(w,w')\in r^{\Imc^{\ast}}$ if $w'=wrd'$.
\end{itemize}
The following lemma can be proved by a straightforward induction.
\begin{lemma}\label{lem:bis}
The relation
$
S= \{ (w,{\sf tail}(w)) \mid w\in \Delta^{{\Imc}^{\ast}}\}
$
is a $\Sigma$-bisimulation between $(\Imc^{\ast},\rho^{\Imc^{\ast}})$ and $(\Imc,d)$, for any signature $\Sigma$.
\end{lemma}

We now characterize concept entailment (and thus also concept
inseparability and concept conservative extensions) in \ALC using
bisimulations, following~\cite{DBLP:conf/ijcai/LutzW11}.


\begin{theorem}\label{bisimuniform}
  Let $\Tmc_{1}$ and $\Tmc_{2}$ be $\mathcal{ALC}$ TBoxes and $\Sigma$
  a signature.  Then $\Tmc_{1}$ ${\Sigma}$-concept
  entails $\Tmc_{2}$ iff, for any model $\Imc_{1}$ of $\Tmc_{1}$ and
  any $d_{1}\in \Delta^{\Imc_{1}}$, there exist a model $\Imc_{2}$ of
  $\Tmc_{2}$ and $d_{2}\in\Delta^{\Imc_{2}}$ such that
  $(\Imc_{1},d_{1}) \sim_\Sigma^{{\sf bisim}} (\Imc_{2},d_{2})$.
\end{theorem}

For $\Imc_{1}$ of finite outdegree, one can prove this result directly by employing
compactness arguments and Theorem~\ref{theorem:bisimchar}. For the
general
case, we refer to~\cite{DBLP:conf/ijcai/LutzW11}.
%
We illustrate Theorem~\ref{bisimuniform} by sketching a proof of the statement from Example~\ref{ALCQ-ALC} (in a slightly more general form).

\begin{example}
\label{ex:charact} 
Consider the $\mathcal{ALC}$ TBoxes
$$
\Tmc_{1}= \{A\sqsubseteq \exists r.\top\}\cup \Tmc \quad \text{and} \quad \Tmc_{2} = \{A
\sqsubseteq \exists r.B\sqcap \exists r.\neg B\}\cup \Tmc,
$$
where $\Tmc$ is an $\mathcal{ALC}$ TBox and $B\not\in \Sigma=
\{A,r\}\cup {\sf sig}(\Tmc)$.  We use Theorem~\ref{bisimuniform} and
Lemma~\ref{lem:bis} to show that $\Tmc_{1}$ ${\Sigma}$-concept entails
$\Tmc_{2}$.  Suppose $\Imc$ is a model of $\Tmc_{1}$ and $d\in
\Delta^{\Imc}$.  Using tree unraveling, we construct a tree model
$\Imc^{\ast}$ of $\Tmc_{1}$ with $(\Imc,d) \sim_{\Sigma}^{{\sf bisim}}
(\Imc^{\ast},\rho^{\Imc^{\ast}})$.  As bisimilations and \ALC{} TBoxes
are oblivious to duplication of successors, we find a tree model
$\Jmc$ of $\Tmc_1$ such that $e \in A^\Jmc$ implies $\#\{d \mid (e,d)
\in r^\Jmc \} \geq 2$ for all $e\in \Delta^{\Jmc}$ and
$(\Imc^{\ast},\rho^{\Imc^{\ast}}) \sim_{\Sigma}^{{\sf bisim}}
(\Jmc,\rho^{\Jmc})$.  By reinterpreting $B\not\in \Sigma$, we can find
$\Jmc'$ that coincides with $\Jmc$ except that now we ensure that $e
\in A^\Jmc$ implies $e \in (\exists r . B \sqcap \exists r . \neg
B)^{\Jmc'}$ for all $e\in \Delta^{\Jmc}$.  But then $\Jmc'$ is a model
of $\Tmc_{2}$ and $(\Imc,d) \sim_{\Sigma}^{\sf bisim}
(\Jmc',\rho^{\Jmc})$, as required.

Below, we illustrate possible interpretations $\Imc^{\ast}$, $\Jmc$ and
$\Jmc'$ satisfying the above conditions, for a given interpretation $\Imc$.

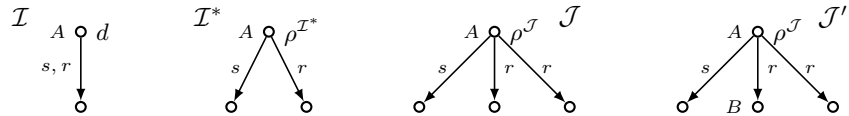
\begin{center}
  \begin{tikzpicture}
    \begin{scope}
      \node at (-0.8,0.2) {$\Imc$};

      \foreach \al/\x/\y/\lab in {%
        d/0/0/d,%
        a/0/-1/%
      }{\node[point, label=right:{\vertexfont $\lab$}] (\al) at (\x,\y) {};}

      \node[anchor=east] at (d.west) {\edgefont $A$};

      \foreach \from/\to/\lab/\wh in {%
        d/a/{s,r}/left%
      }{ \draw[role] (\from) -- node[\wh] {\edgefont $\lab$} (\to); }
    \end{scope}

    \begin{scope}[xshift=2.5cm]
      \node at (-0.8,0.2) {$\Imc^\ast$};

      \foreach \al/\x/\y/\lab in {%
        d/0/0/{\rho^{\Imc^\ast}},%
        x1/-0.5/-1/,%
        x2/0.5/-1/%
      }{\node[point, label=right:{\vertexfont$\lab$}] (\al) at (\x,\y) {};}

      \node[anchor=east] at (d.west) {\edgefont $A$};

      \foreach \from/\to/\lab/\wh in {%
        d/x1/s/left, d/x2/r/right%
      }{ \draw[role] (\from) -- node[\wh] {\edgefont $\lab$} (\to); }
    \end{scope}

    \begin{scope}[xshift=5.5cm]
      \node at (1,0.2) {$\Jmc$};

      \foreach \al/\x/\y/\lab in {%
        d/0/0/{\rho^{\Jmc}},%
        x1/-1/-1/,%
        x2/0/-1/,%
        x3/1/-1/%
      }{\node[point, label=right:{\vertexfont$\lab$}] (\al) at (\x,\y) {};}

      \node[anchor=east] at (d.west) {\edgefont $A$};

      \foreach \from/\to/\lab/\wh in {%
        d/x1/s/left, d/x2/r/right, d/x3/r/right%
      }{ \draw[role] (\from) -- node[\wh] {\edgefont $\lab$} (\to); }
    \end{scope}

    \begin{scope}[xshift=9cm]
      \node at (1,0.2) {$\Jmc'$};

      \foreach \al/\x/\y/\lab in {%
        d/0/0/{\rho^{\Jmc}},%
        x1/-1/-1/,%
        x2/0/-1/,%
        x3/1/-1/%
      }{\node[point, label=right:{\vertexfont$\lab$}] (\al) at (\x,\y) {};}

      \node[anchor=east] at (d.west) {\edgefont $A$}; \node[anchor=east] at
      (x2.west) {\edgefont $B$};

      \foreach \from/\to/\lab/\wh in {%
        d/x1/s/left, d/x2/r/right, d/x3/r/right%
      }{ \draw[role] (\from) -- node[\wh] {\edgefont $\lab$} (\to); }
    \end{scope}

  \end{tikzpicture}
\end{center}
\end{example}

Theorem~\ref{bisimuniform} is a useful starting point for constructing
decision procedures for concept entailment in \ALC and related
problems. This can be done from first principles as in
\cite{GhilardiLutzWolter-KR06,LutzWW07}. Here we present an approach that uses tree automata.
We use amorphous alternating parity tree automata
\cite{Wilke-Automata}, which actually
run on unrestricted interpretations rather than on trees. They still
belong to the family of \emph{tree} automata as they are in the
tradition of more classical forms of such automata and cannot
distinguish between an interpretation and its unraveling into a tree
(which indicates a connection to bisimulations).
\begin{definition}[APTA]\em 
\label{def:apta}
  An (\emph{amorphous}) \emph{alternating parity tree automaton} (or APTA for short) is a tuple $\Amc
  = (Q,\Sigma_N,\Sigma_E,q_0, \delta, \Omega)$, where $Q$ is a finite
  set of \emph{states}, $\Sigma_N \subseteq \NC$ is the finite
  \emph{node alphabet}, $\Sigma_E \subseteq \NR$ is the finite
  \emph{edge alphabet}, $q_0 \in Q$ is the \emph{initial state},
  $
    \delta: Q  \rightarrow \mn{mov}(\Amc)
  $
  is the transition function with
  $
    \mn{mov}(\Amc)=\{ \mn{true}, \mn{false}, A, \neg A, q, q \wedge q',
    q \vee q', \langle r \rangle q, [r] q \mid 
     A \in \Sigma_N, q,q' \in Q, r \in \Sigma_E \}
  $
  the set of \emph{moves} of the automaton, and $\Omega:Q
  \rightarrow \mathbbm{N}$ is the \emph{priority function}.
\end{definition}

Intuitively, the move $q$ means that the automaton sends a copy of
itself in state $q$ to the element of the interpretation that it is
currently processing, $\langle r \rangle q$ means that a copy in state $q$ is
sent to an $r$-successor of the current element, and $[r] q$ means
that a copy in state $q$ is sent to every $r$-successor.

It will be convenient to use unrestricted modal logic formulas in
negation normal form when specifying the transition function of
APTAs. The more restricted form required by Definition~\ref{def:apta}
can then be attained by introducing intermediate
states. We next introduce the semantics of APTAs.

In what follows, a \emph{$\Sigma$-labelled tree} is a pair $(T,\ell)$
with $T$ a tree and $\ell:T \rightarrow \Sigma$ a node
labelling function. A \emph{path} $\pi$ in a tree $T$ is a subset of
$T$ such that $\varepsilon \in \pi$ and for each $x \in \pi$ that is
not a leaf in $T$, $\pi$ contains one child of $x$.
\begin{definition}[run]\em 
\label{def:altrun}
Let $(\Imc,d_0)$ be a pointed $\Sigma_N \cup \Sigma_E$-interpretation
and let $\Amf= (Q,\Sigma_N,\Sigma_E,q_0, \delta, \Omega)$ be an APTA. A
\emph{run} of \Amf on $(\Imc ,d_0)$ is a $Q \times \Delta^\Imc$-labelled tree
$(T,\ell)$ such that $\ell(\varepsilon)=(q_0,d_0)$ and for every
$x \in T$ with $\ell(x)=(q,d)$:
  \begin{itemize}

  \item $\delta(q) \neq \mn{false}$;
    
  \item if $\delta(q) = A$ ($\delta(q) = \neg A$), then $d \in A^\Imc$ ($d \notin A^\Imc$);

  \item if $\delta(q) = q' \wedge q''$, then there are
    children $y,y'$ of $x$ with $\ell(y)=(q',d)$ and $\ell(y')=(q'',d)$;

  \item if $\delta(q) = q' \vee q''$, then there is
    a child $y$ of $x$ such that $\ell(y)=(q',d)$ or $\ell(y')=(q'',d)$;

  \item if $\delta(q) =\langle r \rangle q'$, then
    there is a $(d,d') \in r^\Imc$ and a child $y$ of $x$ such that
    $\ell(y)=(q',d')$;

  \item if $\delta(q) =[ r ] q'$ and $(d,d') \in
    r^\Imc$, then there is a child $y$ of $x$ with $\ell(y)=(q',d')$.

  \end{itemize}
  A run $(T,\ell)$ is \emph{accepting} if, for every path $\pi$ of $T$,
  the maximal $i \in \mathbbm{N}$ with $\{ x \in \pi \mid \ell(x)=(q,d)
  \text{ with } \Omega(q)=i \}$ infinite is even.  We use $L(\Amf)$ to
  denote the language accepted by \Amf, i.e., the set of pointed $\Sigma_N
  \cup \Sigma_E$-interpretations $(\Imc,d)$ such that there is an
  accepting run of \Amf on $(\Imc,d)$.
\end{definition}
%

APTAs can easily be complemented in polynomial time in the same way as
other alternating tree automata, and for all APTAs $\Amf_1$ and
$\Amf_2$, one can construct in polynomial time an APTA that accepts
$L(\Amf_1) \cap L(\Amf_2)$. The emptiness problem for APTAs is
\ExpTime-complete \cite{Wilke-Automata}.

We now describe how APTAs can be used to decide concept entailment in
\ALC. Let $\Tmc_1$ and $\Tmc_2$ be \ALC TBoxes and $\Sigma$ a
signature. By Theorem~\ref{bisimuniform}, $\Tmc_1$ does not
$\Sigma$-concept entail $\Tmc_2$ iff there is a model $\Imc_1$ of
$\Tmc_1$ and a $d_1 \in \Delta^{\Imc_1}$ such that $(\Imc_1,d_1)
\not\sim^{\mn{bisim}}_\Sigma (\Imc_2,d_2)$ for all models $\Imc_2$ of
$\Tmc_2$ and $d_2 \in \Delta^{\Imc_2}$. We first observe that this
still holds when we restrict ourselves to \emph{rooted
  interpretations}, that is, to pointed interpretations $(\Imc_i,d_i)$
such that every $e \in \Delta^{\Imc_i}$ is reachable from $d_i$ by
some sequence of role names. In fact, whenever $(\Imc_1,d_1)
\not\sim^{\mn{bisim}}_\Sigma (\Imc_2,d_2)$, then also $(\Imc^r_1,d_1)
\not\sim^{\mn{bisim}}_\Sigma (\Imc^r_2,d_2)$ where $\Imc^r_i$ is the
restriction of $\Imc_i$ to the elements reachable from $d_i$. Moreover, if 
$\Imc_1$ and $\Imc_2$ are models of $\Tmc_1$ and $\Tmc_2$,
respectively, then the same is true for $\Imc^r_1$ and $\Imc^r_2$.
Rootedness is important because APTAs can obviously not speak about
unreachable parts of a pointed interpretation.  We now construct two
APTAs $\Amf_1$ and $\Amf_2$ such that for all rooted
$\mn{sig}(\Tmc_1)$-interpretations $(\Imc,d)$,
\begin{enumerate}

\item $(\Imc,d) \in L(\Amf_1)$ iff $\Imc \models \Tmc_1$;

\item $(\Imc,d) \in L(\Amf_2)$ iff there exist a model \Jmc of $\Tmc_2$
  and an $e \in \Delta^\Jmc$ such that $(\Imc,d) \sim^{\mn{bisim}}_\Sigma
  (\Jmc,e)$.

\end{enumerate}
Defining $\Amf$ as $\Amf_1 \cap \overline{\Amf_2}$, we then have
$L(\Amf)=\emptyset$ iff $\Tmc_1$ $\Sigma$-concept entails $\Tmc_2$. It
is easy to construct the automaton $\Amf_1$. We only illustrate the idea by an
example. Assume that $\Tmc_1 = \{A \sqsubseteq \neg \forall r
. B\}$. We first rewrite $\Tmc_1$ into the equivalent TBox $\{ \top \sqsubseteq \neg
A \sqcup \exists r . \neg B \}$ and then use an APTA with only state
$q_0$ and
$$
\delta(q_0) = \bigwedge_{s \in \NR} [s] q_0 \wedge (\neg A \vee \langle r \rangle
\neg B).
$$
The acceptance condition is trivial, that is, $\Omega(q_0)=0$. 
The construction of $\Amf_2$ is more interesting. 
We require the notion of a type, which occurs in many constructions
for \ALC. Let $\mn{cl}(\Tmc_2)$ denote the set of concepts used in
$\Tmc_2$, closed under subconcepts and single negation. A type
$t$ is a set $t \subseteq \mn{cl}(\Tmc_2)$ such that, for some model
\Imc of $\Tmc_2$ and some $d \in \Delta^\Imc$, we have $t=\{ C\in {\mn{cl}(\Tmc_2)} \mid d \in C^\Imc\}$.
Let $\mn{TP}(\Tmc_2)$ denote the set of all types for $\Tmc_2$.  For
$t,t' \in \mn{TP}(\Tmc_2)$ and a role name $r$, we write $t
\leadsto_{r} t'$ if (i)~$\forall r . C \in t$ implies $C \in t'$ and (ii)~$C\in
t'$ implies $\exists r . C\in t$ whenever $\exists r.C\in
{\mn{cl}(\Tmc)}$.  Now we define $\Amf_2$ to have state
set $Q = \mn{TP}(\Tmc_2) \uplus \{ q_0 \}$ and the following transitions:
$$
\begin{array}{r@{\;}c@{\;}l}
\delta(q_0)&~=~&\displaystyle \bigvee \mn{TP}(\Tmc_2), 
\\[1mm]
\delta(t)&~=~&\displaystyle \bigwedge_{A \in t \cap \NC \cap \Sigma} A
\wedge \bigwedge_{A \in (\NC \cap \Sigma)\setminus t} \neg A \\[5mm]
&&\wedge \; \displaystyle  \bigwedge_{r \in \Sigma \cap \NR} \!
  [r] \bigvee \{ t' \in \mn{TP}(\Tmc_2) \mid t \leadsto_{r} t'\} 
\\[5mm]
&& \wedge \; \displaystyle
\bigwedge_{\exists r . C \in t, r \in \Sigma} \!\!\!\!\!\! \langle r \rangle
   \bigvee \{ t' \in \mn{TP}(\Tmc_2) \mid t \leadsto_{r} t',~ C \in t' \}. 
\end{array}
$$
Here, the empty conjunction represents \mn{true} and the empty
disjunction represents \mn{false}.  The acceptance condition is again
trivial, but note that this might change with complementation.  The
idea is that $\Amf_2$ (partially) guesses a model $\Jmc$ that is
$\Sigma$-bisimilar to the input interpretation \Imc, represented as
types. Note that $\Amf_2$ verifies only the $\Sigma$-part of \Jmc on
\Imc, and that it might label the same element with different types
(which can then only differ in their non-$\Sigma$-parts). A detailed
proof that the above automaton works as expected is provided in
\cite{DBLP:conf/ijcai/LutzW11}. In summary, we obtain the upper bound in the following
theorem.
\begin{theorem}
  \label{thm:ceupper}
  In \ALC, concept entailment, concept inseparability, and concept
  conservative extensions are 2\ExpTime-complete.
\end{theorem}

The sketched APTA-based decision procedure actually yields an upper
bound that is slightly stronger than what is stated in
Theorem~\ref{thm:ceupper}: the algorithm for concept entailment (and
concept conservative extensions) actually runs in time $2^{p(|\Tmc_1|
  \cdot 2^{|\Tmc_2|})}$ for some polynomial $p()$
and is thus only single exponential in $|\Tmc_1|$. For simplicity, in
the remainder of the paper we will typically not explicitly report on
such fine-grained upper bounds that distinguish between different
inputs.

The lower bound stated in Theorem~\ref{thm:ceupper} is proved (for
concept conservative extensions) in \cite{GhilardiLutzWolter-KR06} using a rather
intricate reduction of the word problem of exponentially space bounded
alternating Turing machines (ATMs).  An interesting issue that is
closely related to computational hardness is to analyze the size of
the smallest concept inclusions that witness non-$\Sigma$-concept
entailment of a TBox $\Tmc_2$ by a TBox $\Tmc_1$, that is, of the
members of $\mn{cDiff}_\Sigma(\Tmc_1,\Tmc_2)$. It is shown in
\cite{GhilardiLutzWolter-KR06} for the case of concept conservative extensions in \ALC
(and thus also for concept entailment) that smallest witness
inclusions can be triple exponential in size, but not larger. An
example that shows why witness inclusions can get large is given in
Section~\ref{sect:el}.

\subsection{Concept inseparability for extensions of $\mathcal{ALC}$}

We briefly discuss results on concept inseparability for extensions of
\ALC and give pointers to the literature.

In principle, the machinery and results that we have presented for
$\mathcal{ALC}$ can be adapted to many extensions of $\mathcal{ALC}$,
for example, with number restrictions, inverse roles, and role
inclusions.  To achieve this, the notion of bisimulation has to be
adapted to match the expressive power of the considered DL and the
automata construction has to be modified. In particular, amorphous
automata as used above are tightly linked to the expressive power of
\ALC and have to be replaced by traditional alternating tree automata
(running on trees with fixed outdegree) which requires a slightly more
technical automaton construction.

As an illustration, we only give some brief examples. To obtain an
analogue of Theorem~\ref{theorem:bisimchar} for $\mathcal{ALCI}$, one
needs to extend bisimulations that additionally respect successors
reachable by an inverse role; to obtain such a result for
$\mathcal{ALCQ}$, we need bisimilations that respect the number of
successors~\cite{GorankoOtto,DBLP:conf/ijcai/LutzPW11,rewritecons16}.
Corresponding versions of Theorem~\ref{bisimuniform} can then be proved using techniques
from \cite{DBLP:conf/ijcai/LutzPW11,rewritecons16}.
\begin{example}\label{ALCQ}
Consider the $\mathcal{ALCQ}$ TBoxes
$$
\Tmc_{1}=\{A \sqsubseteq \mathop{\ge 2} r.\top\} \cup \Tmc \quad \text{and} \quad 
\Tmc_{2}=\{A \sqsubseteq \exists r.B\sqcap \exists r.\neg B\} \cup \Tmc,
$$
where $\Tmc$ is an $\mathcal{ALCQ}$ TBox that does not use the concept
name $B$. Suppose $\Sigma=\{A,r\}\cup {\sf sig}(\Tmc)$.  Then $\Tmc_{1}$
and $\Tmc_{2}$ are ${\Sigma}$-concept inseparable in $\mathcal{ALCQ}$. Formally, this can be shown
using the characterizations from \cite{rewritecons16}.
\end{example}
The above approach has not been fully developed in the
literature. However, using more elementary methods, the following
complexity result has been established in \cite{LutzWW07}.
\begin{theorem}
  \label{thm:ceupper1}
  In \ALCQI, concept entailment, concept inseparability, and concept
  conservative extensions are 2\ExpTime-complete.
\end{theorem}

It is also shown in \cite{LutzWW07} that, in \ALCQI, smallest
counterexamples are still triple exponential, and that further
adding nominals to \ALCQI results in undecidability.
\begin{theorem}
  \label{thm:ceupper2}
  In \ALCQIO, concept entailment, concept inseparability, and concept
  conservative extensions are undecidable.
\end{theorem}

For a number of prominent extensions of \ALC, concept inseparability
has not yet been investigated in much detail. This particularly
concerns extensions with transitive roles~\cite{DBLP:journals/logcom/HorrocksS99}.
We note that it is not
straightforward to lift the above techniques to DLs with transitive
roles; see \cite{GLWZ} where conservative extensions in modal logics
with transitive frames are studied and \cite{French} in which modal
logics with bisimulation quantifiers (which are implicit in
Theorem~\ref{bisimuniform}) are studied, including cases with
transitive frame classes. As illustrated in Section~\ref{sect:final}, extensions of \ALC with the universal
role are also an interesting subject to study.


\subsection{Concept inseparability for $\mathcal{EL}$}
\label{sect:el}

We again start with model-theoretic characterizations and then proceed
to decision procedures, complexity, and the length of
counterexamples. In contrast to \ALC, we use simulations, which
intuitively are `half a bisimulation', much like \EL is `half of
\ALC'. The precise definition is as follows.
\begin{definition}[$\Sigma$-simulation]\em 
  Let $\Sigma$ be a finite signature and $(\Imc_{1},d_{1})$,
  $(\Imc_{2},d_{2})$ pointed interpretations. A relation $S \subseteq
  \Delta^{\Imc_{1}}\times \Delta^{\Imc_{2}}$ is a
  \emph{$\Sigma$-simulation} from $(\Imc_{1},d_{1})$ to
  $(\Imc_{2},d_{2})$ if $(d_{1},d_{2})\in S$ and, for all $(d,d') \in
  S$, the following conditions are satisfied:
\begin{description}

  \item[\rm (base$^\ell$)] if $d \in A^{\Imc_{1}}$, then $d' \in A^{\Imc_{2}}$,
     for all $A \in \Sigma\cap \NC$;

  \item[\rm (zig)] if $(d,e) \in r^{\Imc_{1}}$, then there exists $e'\in \Delta^{\Imc_{2}}$ such
that $(d',e') \in r^{\Imc_{2}}$ and $(e,e')\in S$, for all $r \in \Sigma\cap \NR$.

\end{description}
We say that
$(\Imc_{2},d_{2})$ \emph{$\Sigma$-simulates} 
$(\Imc_{1},d_{1})$ and write $(\Imc_{1},d_{1}) \leq_\Sigma^{{\sf sim}} (\Imc_{2},d_{2})$
if there exist a $\Sigma$-simulation from
$(\Imc_{1},d_{1})$ to $(\Imc_{2},d_{2})$. 
We say that
$(\Imc_{1},d_{1})$  and 
$(\Imc_{2},d_{2})$  are 
\emph{$\Sigma$-equisimilar}, in symbols $(\Imc_{1},d_{1}) \sim_\Sigma^{{\sf esim}}
(\Imc_{2},d_{2})$, if 
both 
 $(\Imc_{1},d_{1}) \leq_\Sigma^{{\sf sim}} (\Imc_{2},d_{2})$ and
 $(\Imc_{2},d_{2}) \leq_\Sigma^{{\sf sim}} (\Imc_{1},d_{1})$.
\end{definition}

A pointed interpretation $(\Imc_{1},d_{1})$ is \emph{$\mathcal{EL}_{\Sigma}$-contained in} $(\Imc_{2},d_{2})$,
in symbols $(\Imc_{1},d_{1})\leq_{\Sigma}^{\mathcal{EL}} (\Imc_{2},d_{2})$, if 
$d_{1}\in C^{\Imc_{1}}$ implies $d_{2}\in C^{\Imc_{2}}$, for all $\mathcal{EL}_{\Sigma}$-concepts $C$.
We call pointed interpretations $(\Imc_{1},d_{1})$ and $(\Imc_{2},d_{2})$ 
\emph{$\mathcal{EL}_{\Sigma}$-equivalent}, in symbols
$(\Imc_{1},d_{1})\equiv_{\Sigma}^{\mathcal{EL}}(\Imc_{2},d_{2})$, in case 
 $(\Imc_{1},d_{1})\leq_{\Sigma}^{\EL}(\Imc_{2},d_{2})$  and 
 $(\Imc_{2},d_{2})\leq_{\Sigma}^{\EL}(\Imc_{1},d_{1})$.
The following was shown in \cite{DBLP:journals/jsc/LutzW10,DBLP:conf/kr/LutzSW12}.
\begin{theorem}\label{theorem:equisimchar}
  Let $(\Imc_1,d_1)$ and $(\Imc_2,d_2)$ be pointed interpretations and
  $\Sigma$ a signature. Then $(\Imc_{1},d_{1}) \leq_\Sigma^{{\sf sim}}
  (\Imc_{2},d_{2})$ implies $(\Imc_{1},d_{1})
  \leq_{\Sigma}^{\mathcal{EL}} (\Imc_{2},d_{2})$.  The converse holds
  if $\Imc_1$ and $\Imc_2$ are of finite outdegree.
\end{theorem}

The interpretations given in Example~\ref{boundedoutdegree} can be
used to show that the converse direction in
Theorem~\ref{theorem:equisimchar} does not hold in general (since
$(\Imc_{2},d_{2})\not\leq_{\Sigma}^{\sf sim} (\Imc_{1},d_{1})$).  It
is instructive to see pointed interpretations that are equisimilar but
not bisimilar.
\begin{example}
  Consider the interpretations $\Imc_{1}=(\{d_{1},e_{1}\}$,
  $A^{\Imc_{1}}=\{e_{1}\}$, $r^{{\Imc}_{1}}=\{(d_{1},e_{1})\})$ and
  $\Imc_{2}=(\{d_{2},e_{2},e_{3}\}$, 
  $A^{\Imc_{2}}=\{e_{2}\}$, $r^{{\Imc}_{2}}=\{(d_{2},e_{2}),(d_{2},e_{3})\})$ and
  let $\Sigma= \{r,A\}$. Then $(\Imc_{1},d_{1})$ and $(\Imc_{2},d_{2})$ are
  $\Sigma$-equisimilar but not $\Sigma$-bisimilar.

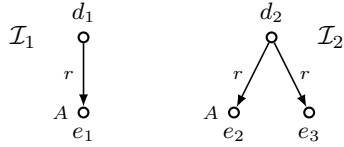
\begin{center}
  \begin{tikzpicture}
    \begin{scope}
      \node at (-0.8,0) {$\Imc_1$};

      \foreach \al/\x/\y/\lab\wh in {%
        d/0/0/d_1/above,%
        e/0/-1/e_1/below%
      }{\node[point, label=\wh:{\vertexfont $\lab$}] (\al) at (\x,\y) {};}

      \node[anchor=east] at (e.west) {\edgefont $A$};

      \foreach \from/\to/\lab/\wh in {%
        d/e/{r}/left%
      }{ \draw[role] (\from) -- node[\wh] {\edgefont $\lab$} (\to); }
    \end{scope}

    \begin{scope}[xshift=2.5cm]
      \node at (0.8,0) {$\Imc_2$};

      \foreach \al/\x/\y/\lab/\wh in {%
        d/0/0/{d_2}/above,%
        x1/-0.5/-1/e_2/below,%
        x2/0.5/-1/e_3/below%
      }{\node[point, label=\wh:{\vertexfont$\lab$}] (\al) at (\x,\y) {};}

      \node[anchor=east] at (x1.west) {\edgefont $A$};

      \foreach \from/\to/\lab/\wh in {%
        d/x1/r/left, d/x2/r/right%
      }{ \draw[role] (\from) -- node[\wh] {\edgefont $\lab$} (\to); }
    \end{scope}

  \end{tikzpicture}
\end{center}
\end{example}

Similar to Theorem~\ref{bisimuniform}, $\Sigma$-equisimilarity can be
used to give a model-theoretic characterization of concept entailment
(and thus also concept inseparability and concept conservative
extensions) in \EL \cite{DBLP:conf/kr/LutzSW12}.


\begin{theorem}\label{equisimuniform}
  Let $\Tmc_{1}$ and $\Tmc_{2}$ be $\mathcal{EL}$ TBoxes and $\Sigma$
  a signature.  Then $\Tmc_{1}$ $\Sigma$-concept entails $\Tmc_{2}$
  iff, for any model $\Imc_{1}$ of $\Tmc_{1}$ and any $d_{1}\in
  \Delta^{\Imc_{1}}$, there exist a model $\Imc_{2}$ of $\Tmc_{2}$ and
  $d_{2}\in\Delta^{\Imc_{2}}$ such that $(\Imc_{1},d_{1})
  \sim_\Sigma^{{\sf esim}} (\Imc_{2},d_{2})$.
\end{theorem}

We illustrate Theorem~\ref{equisimuniform} by proving that the TBoxes $\Tmc_{1}$ and $\Tmc_{2}$ from
Example~\ref{ex:e7} are ${\Sigma}$-concept inseparable in $\mathcal{EL}$.

\begin{example}
Recall that $\Sigma={\sf sig}(\Tmc_{1})$ and 
\begin{align*}
\Tmc_{1} &= \{{\sf Human} \sqsubseteq \exists {\sf eats}.\top, \ 
{\sf Plant} \sqsubseteq \exists {\sf grows\_in}.{\sf Area}, \ 
{\sf Vegetarian} \sqsubseteq {\sf Healthy}\},\\
\Tmc_{2} &= \Tmc_{1} \cup
\{{\sf Human} \sqsubseteq \exists {\sf eats}.{\sf Food}, \ 
{\sf Food} \sqcap {\sf Plant} \sqsubseteq {\sf Vegetarian}\}.
\end{align*}
Let $\Imc$ be a model of $\Tmc_{1}$ and $d\in \Delta^{\Imc}$.  We may
assume that ${\sf Food}^{\Imc}=\emptyset$.  Define $\Imc'$ by adding,
for every $e\in {\sf Human}^{\Imc}$, a fresh individual ${\sf new}(e)$
to $\Delta^{\Imc}$ with $(e,{\sf new}(e))\in {\sf eats}^{\Imc'}$ and
${\sf new}(e)\in {\sf Food}^{\Imc'}$.  Clearly, $\Imc'$ is a model of
$\Tmc_{2}$.  We show that $(\Imc,d)$ and $(\Imc',d)$ are
$\Sigma$-equisimilar.  The identity $\{(e,e) \mid e\in
\Delta^{\Imc}\}$ is obviously a $\Sigma$-simulation from $(\Imc,d)$ to
$(\Imc',d)$.  Conversely, pick for each $e\in {\sf Human}^{\Imc'}$ an
${\sf old}(e)\in \Delta^{\Imc}$ with $(e,{\sf old}(e))\in {\sf
  eats}^{\Imc}$, which must exist by the first CI of $\Tmc_{1}$. It
can be verified that
$$
S= \{ (e,e)\mid e\in \Delta^{\Imc}\} \cup \{({\sf new}(e),{\sf old}(e))\mid
e\in \Delta^{\Imc}\}
$$
is a $\Sigma$-simulation from $(\Imc',d)$ to $(\Imc,d)$.  Note that
$(\Imc,d)$ and $(\Imc',d)$ are not guaranteed to be $\Sigma$-bisimilar.
\end{example}

As in the \ALC case, Theorem~\ref{equisimuniform} gives rise to a
decision procedure for concept entailment based on tree
automata. However, we can now get the complexity down to \ExpTime.\footnote{An alternative elementary proof is given in \cite{DBLP:journals/jsc/LutzW10}.} To
achieve this, we define the automaton $\Amf_2$ in a more careful way
than for \ALC, while we do not touch the construction of $\Amf_1$. Let
$\mn{sub}(\Tmc_2)$ denote the set of concepts that occur in $\Tmc_2$,
closed under subconcepts. For any $C \in \mn{sub}(\Tmc_2)$, we use
$\mn{con}_\Tmc(C)$ to denote the set of concepts $D \in
\mn{sub}(\Tmc_2)$ such that $\Tmc \models C \sqsubseteq D$.  We define
the APTA based on the set of states
$$
  Q=\{q_0\} \uplus \{q_C, \overline{q}_C \mid C \in 
  \mn{sub}(\Tmc_2) \}, 
$$
where $q_0$ is the starting state. The transitions are as follows:
$$
\begin{array}{r@{\;}c@{\;}l}
  \delta(q_0)&=&\displaystyle \bigwedge_{C \in \mn{sub}(\Tmc_2)} ( q_C \vee \overline{q}_C ) \wedge \bigwedge_{r \in \Sigma} [r] q_0, \\[5mm]
  \delta(q_A)&=& \displaystyle A \wedge \bigwedge_{C \in \mn{con}_\Tmc(A)} q_C \text{ and } \delta(\overline{q}_A)= \neg A \text{ for all } A \in \mn{sub}(\Tmc_2) 
  \cap \NC \cap \Sigma, \\[1mm]
  \delta(q_{C \sqcap D})&=& \displaystyle q_C \wedge q_D  \wedge
  \bigwedge_{E \in \mn{con}_\Tmc(C \sqcap D)} q_E
  \text{ and }  \\[1mm]
\delta(\overline{q}_{C \sqcap D})&=& \overline{q}_C \vee \overline{q}_D 
\text{ for all } C \sqcap D \in \mn{sub}(\Tmc_2), 
  \\[1mm]
  \delta(q_{\exists r . C})&=& \displaystyle \langle r \rangle q_C
  \wedge \bigwedge_{D \in \mn{con}_\Tmc(\exists r . C)} q_D
  \text{ and }  \\[1mm]
\delta(\overline{q}_{\exists r . C}) &=& [r] \overline{q}_C
\text{ for all } \exists r . C \in \mn{sub}(\Tmc_2) \text{ with } r\in \Sigma,
  \\[1mm]
  \delta(q_\top)&=& \displaystyle \bigwedge_{C \in \mn{con}_\Tmc(\top)} q_C \text{
    and } \delta(\overline{q}_\top)= \mn{false} .
\end{array}
$$
Observe that, in each case, the transition for $\overline{q}_C$ is the
dual of the transition for $q_C$, except that the latter has an
additional conjunction pertaining to
$\mn{con}_\Tmc$. 
As before, we set $\Omega(q)=0$ for all $q \in Q$. An essential
difference
between the above APTA $\Amf_2$ and the one that we had 
constructed for \ALC is that the latter had to look at sets of
subconcepts (in the form of a type) while the automaton above
always considers only a single subconcept at the time.
A proof that the above automaton works as expected can be extracted
from \cite{DBLP:conf/kr/LutzSW12}. 
\begin{theorem}
  \label{thm:ceupperel}
  In \EL, concept entailment, concept inseparability, and concept
  conservative extensions are \ExpTime-complete.
\end{theorem}

The lower bound in Theorem~\ref{thm:ceupperel} is proved (for concept
conservative extension) in \cite{DBLP:journals/jsc/LutzW10} using a reduction of the word
problem of polynomially space bounded ATMs.  It can be extracted from
the proofs in \cite{DBLP:journals/jsc/LutzW10} that smallest concept inclusions that
witness failure of concept entailment (or concept conservative
extensions) are at most double exponentially large, measured in the
size of the input TBoxes.\footnote{This should not be confused with
  the size of uniform interpolants, which can even be triple
  exponential in \EL \cite{expexpexplosion}.} The
following example shows a case where they are also at least double
exponentially large.

\begin{example}
  \newcommand{\ol}{\overline} For each $n \geq 1$, we give TBoxes
$\Tmc_1$ and $\Tmc_2$ whose size is polynomial in $n$ and such that
$\Tmc_2$ is not a concept conservative extension of $\Tmc_1$, but the
elements of $\mn{cDiff}_{\Sigma}(\Tmc_1,\Tmc_2)$ are of size at
least~$2^{2^n}$ for $\Sigma={\sf sig}(\Tmc_{1})$. It is instructive to start with the definition of
$\Tmc_2$, which is as follows:
$$
\begin{array}{rcll}
  A & \sqsubseteq & \ol{X}_0 \sqcap \cdots \sqcap \ol{X}_{n-1}, \\
  {\bigsqcap}_{\sigma \in \{r,s\}} \exists \sigma . ( \ol{X}_i \sqcap X_0 \sqcap \cdots \sqcap X_{i-1} )
  & \sqsubseteq & X_i, & \text{ for $i < n$}, \\
    {\bigsqcap}_{\sigma \in \{r,s\}}\exists \sigma . ( X_i \sqcap X_0 \sqcap \cdots \sqcap X_{i-1} )
  & \sqsubseteq & \ol{X}_i, & \text{ for $i < n$}, \\
    {\bigsqcap}_{\sigma \in \{r,s\}} \exists \sigma . ( \ol{X_i} \sqcap \ol{X}_j ) & \sqsubseteq & \ol{X_i},
  & \text{ for $j < i < n$}, \\
    {\bigsqcap}_{\sigma \in \{r,s\}} \exists \sigma . ( X_i \sqcap \ol{X}_j ) & \sqsubseteq & X_i,
  & \text{ for $j < i < n$}, \\
  X_0 \sqcap \cdots \sqcap X_{n-1} & \sqsubseteq & B.
\end{array}
$$
The concept names $X_0,\dots,X_{n-1}$ and
$\overline{X}_0,\dots,\overline{X}_{n-1}$ are used to represent a
binary counter: if $X_i$ is true,
then the $i$-th bit is positive and if $\ol{X}_i$ is true, then it is
negative. These concept names will not be used in $\Tmc_1$ and thus
cannot occur in $\mn{cDiff}_{\Sigma}(\Tmc_1,\Tmc_2)$ for the signature $\Sigma$ of $\Tmc_{1}$.
Observe that Lines~2-5 implement incrementation of the counter. We are interested in
consequences of $\Tmc_{2}$ that are of the form $C_{2^n} \sqsubseteq
B$, where
$$
  C_0 = A, \qquad\qquad
  C_i = \exists r . C_{i-1} \sqcap \exists s . C_{i-1},
$$
which we would like to be the smallest elements of
$\mn{cDiff}_{\Sigma}(\Tmc_1,\Tmc_2)$.  Clearly, $C_{2^n}$ is of size at least
$2^{2^n}$. 
Ideally, we would like to employ a trivial TBox $\Tmc_1$ that uses
only signature $\Sigma=\{A,B,r,s\}$ and has no interesting
consequences (only tautologies). If we do exactly this, though, there
are some undesired (single exponentially) `small' CIs in
$\mn{cDiff}_{\Sigma}(\Tmc_1,\Tmc_2)$, in particular $C'_n \sqsubseteq
B$, where
$$
  C'_0 = A, \qquad\qquad
  C'_i = A \sqcap \exists r . C_{i-1} \sqcap \exists s . C_{i-1}.
$$
Intuitively, the multiple use of $A$ messes up our counter, making
bits both true and false at the same time and resulting in all concept
names $X_i$ to become true already after travelling $n$ steps along
$r$. We thus have to achieve that these CIs are already consequences
of $\Tmc_1$. To this end, we define $\Tmc_1$ as
$$
\begin{array}{rcl}
  \exists \sigma . A \sqsubseteq A', \qquad A' \sqcap A \sqsubseteq B',
  \qquad \exists \sigma . B' \sqsubseteq B', \qquad B' \sqsubseteq B
\end{array}
$$
where $\sigma$ ranges over $\{r,s\}$, and include these concept
assertions also in $\Tmc_2$ to achieve $\Tmc_1 \subseteq \Tmc_2$ as
required for conservative extensions.
\end{example}

\subsection{Concept inseparability for acyclic $\mathcal{EL}$ TBoxes}

\newcommand{\cDiff}{{\sf cDiff}_{\Sigma}} We show that concept
inseparability for \emph{acyclic} \EL TBoxes can be decided in
polynomial time and discuss interesting applications to versioning and the logical
diff of TBoxes. We remark that TBoxes used in practice are
often acyclic, and that, in fact, many biomedical ontologies such as
{\sc Snomed CT} are acyclic \EL TBoxes or mild extensions thereof.

Concept inseparability of acyclic \EL TBoxes is still far from being a
trivial problem.  For example, it can be shown that smallest
counterexamples from $\cDiff(\Tmc_{1},\Tmc_{2})$ can be exponential in
size \cite{DBLP:journals/jair/KonevL0W12}. However, acyclic \EL TBoxes
enjoy the pleasant property that if $\cDiff(\Tmc_{1},\Tmc_{2})$ is
non-empty, then it must contain a concept inclusion of the form $C
\sqsubseteq A$ or $A \sqsubseteq C$, with $A$ a concept name. This is
a consequence of the following result, established in
\cite{DBLP:journals/jair/KonevL0W12}.
%
%
%
%
\newcommand{\cWtn}{{\sf cWtn}_{\Sigma}}
\newcommand{\cWtnRole}{{\sf cWtn}^{\sf R}_{\Sigma}}
\newcommand{\cWtnl}{{\sf cWtn}^{\sf lhs}_{\Sigma}}
\newcommand{\cWtnLhs}{{\sf cWtn}^{\sf lhs}_{\Sigma}}
\newcommand{\cWtnLhsA}{\ensuremath{{\sf cWtn}^{\sf lhs,A}_{\Sigma}}}
\newcommand{\cWtnLhsDom}{\ensuremath{{\sf cWtn}^{\sf lhs,dom}_{\Sigma}}}
\newcommand{\cWtnLhsRan}{\ensuremath{{\sf cWtn}^{\sf lhs,ran}_{\Sigma}}}
\newcommand{\cWtnr}{{\sf cWtn}^{\sf rhs}_{\Sigma}}
\newcommand{\cWtnRhs}{{\sf cWtn}^{\sf rhs}_{\Sigma}}
\newcommand{\cWtnDom}{{\sf cWtn}^{\sf dom}_{\Sigma}}
\newcommand{\cWtnRan}{{\sf cWtn}^{\sf ran}_{\Sigma}}
\begin{theorem}\label{splitEL}
Suppose $\Tmc_1$ and $\Tmc_2$ are acyclic $\mathcal{EL}$ TBoxes and $\Sigma$ a
signature.  If $C \sqsubseteq D \in \cDiff(\Tmc_{1},\Tmc_{2})$, 
then there exist subconcepts $C'$ of $C$ and $D'$ of $D$ such that $C'
\sqsubseteq D'\in\cDiff(\Tmc_{1},\Tmc_{2})$, and $C'$ or $D'$ is a
concept name.
\end{theorem}

Theorem~\ref{splitEL} implies that every logical difference between
$\Tmc_1$ and $\Tmc_2$ is associated with a concept name from $\Sigma$
(that must occur in $\Tmc_2)$. This opens up an interesting
perspective for representing the logical difference between TBoxes since, in contrast to
$\cDiff(\Tmc_{1},\Tmc_{2})$, the set of all concept names $A$ that are
associated with a logical difference $C \sqsubseteq A$ or $A
\sqsubseteq C$ is finite. One can thus summarize for the user the
logical difference between two TBoxes $\Tmc_1$ and $\Tmc_2$ by
presenting her with the list of all such concept names
$A$. 

Let $\Tmc_1$ and $\Tmc_2$ be acyclic \EL{} TBoxes and $\Sigma$ a
signature.  We define the set of \emph{left-hand} $\Sigma$-concept
difference witnesses $\cWtnl(\Tmc_1,\Tmc_2)$ (or \emph{right-hand}
$\Sigma$-concept difference witnesses $\cWtnr(\Tmc_1,\Tmc_2)$) as the
set of all $A\in\Sigma\cap\NC$ such that there exists a concept $C$
with $A\sqsubseteq C\in\cDiff(\Tmc_1,\Tmc_2)$ (or $C\sqsubseteq
A\in\cDiff(\Tmc_1,\Tmc_2)$, respectively). Note that, by
Theorem~\ref{splitEL}, $\Tmc_1$ $\Sigma$-concept entails $\Tmc_2$ iff
$\cWtnl(\Tmc_1,\Tmc_2) = \cWtnr(\Tmc_1,\Tmc_2)=\emptyset$.  In the
following, we explain how both sets can be computed in polynomial
time. The constructions are from~\cite{DBLP:journals/jair/KonevL0W12}.

The tractability of computing $\cWtnl(\Tmc_1,\Tmc_2)$ follows from
Theorem~\ref{theorem:equisimchar} and the fact that \EL has canonical
models. More specifically, for every \EL TBox $\Tmc$ and \EL concept
$C$ one can construct in polynomial time a canonical pointed interpretation $(\Imc_{\Tmc,C},d)$
such that, for any \EL concept $D$, we have $d\in D^{\Imc_{\Tmc,C}}$ iff
$\Tmc\models C\sqsubseteq D$. 
Then Theorem~\ref{theorem:equisimchar} yields for any $A\in \Sigma$
that
$$
A\in \cWtnl(\Tmc_1,\Tmc_2)\quad \Longleftrightarrow \quad
(\Imc_{2},d_2) \not\leq_{\Sigma}^{\sf sim} (\Imc_{1},d_1)
$$
where $(\Imc_{{i}},d_i)$ are canonical pointed interpretations for
$\Tmc_{i}$ and $A$, $i=1,2$.  Since the existence of a simulation
between polynomial size pointed interpretations can be decided in
polynomial time~\cite{CS01}, we have proved the following result.
\begin{theorem}\label{th:ellhswtn}
For $\mathcal{EL}$ TBoxes $\Tmc_{1}$ and $\Tmc_{2}$ and a signature $\Sigma$,
$\cWtnl(\Tmc_1,\Tmc_2)$ can be computed in polynomial time.
\end{theorem}
%
%

We now consider $\cWtnr(\Tmc_1,\Tmc_2)$, that is, $\Sigma$-CIs of the
form $C \sqsubseteq A$.  To check, for a concept name $A\in \Sigma$, 
whether $A \in \cWtnr(\Tmc_1,\Tmc_2)$, ideally we would like to
compute all concepts $C$ such that $\Tmc_1 \not\models C \sqsubseteq A$ and
then check whether $\Tmc_2\models C \sqsubseteq A$. Unfortunately,
there are infinitely many such concepts $C$.  Note that if $\Tmc_2
\models C \sqsubseteq A$ and $C'$ is more specific than $C$ in the
sense that $\models C' \sqsubseteq C$, then $\Tmc_2 \models C'
\sqsubseteq A$. If there is a \emph{most specific} concept $C_A$ among
all $C$ with $\Tmc_1 \not\models C \sqsubseteq A$, it thus suffices to
compute this $C_A$ and check whether $\Tmc_2 \models C_A \sqsubseteq
A$. Intuitively, though, such a $C_A$ is only guaranteed to exist when
we admit infinitary concepts. The solution is to represent $C_A$ not as
a concept, but as a TBox. We only demonstrate this approach by an
example and refer the interested reader
to~\cite{DBLP:journals/jair/KonevL0W12} for further details.

\begin{example}
  (a)~Suppose that $\Tmc_1 = \{A\equiv \exists r.A_1\}$, $\Tmc_2 =
  \{A\equiv \exists r.A_2\}$ and $\Sigma = \{A, A_1, A_2, r\}$.  A
  concept $C_A$ such that $\Tmc_1\not\models C_A\sqsubseteq A$ should
  have neither $A$ nor $\exists r.A_{1}$ as top level conjuncts.
  This can be captured by the CIs
\begin{eqnarray}
\label{eq:Anew}
X_A &\sqsubseteq& A_1\sqcap A_2\sqcap \exists r.(A\sqcap A_2\sqcap \exists r.X_\Sigma), 
\\
\label{eq:all}
X_{\Sigma} & \sqsubseteq &
A \sqcap A_1 \sqcap A_2 \sqcap \exists r.X_{\Sigma},
\end{eqnarray}
where $X_A$ and $X_\Sigma$ are fresh concept names and $X_A$
represents the most specific concept $C_A$ with $\Tmc_1\not\models
C_A\sqsubseteq A$.  We have
$\Tmc_2\cup\{(\ref{eq:Anew}),(\ref{eq:all})\}\models
X_{A}\sqsubseteq A$ and thus $A\in\cWtn(\Tmc_1,\Tmc_2)$.

(b)~Consider next  
$\Tmc_1 = \{A\equiv \exists r.A_1 \sqcap \exists r.A_2\}$, 
$\Tmc_2 = \{A\equiv \exists r.A_2\}$ and $\Sigma = \{A, A_1, A_2, r\}$.
A concept $C_A$ with $\Tmc_1\not\models C_A\sqsubseteq A$ should not have
both $\exists r.A_1$ and
$\exists r.A_2$ as top level conjuncts. Thus the most specific $C_A$
should
contain exactly one of these top level conjuncts, which gives rise to a choice.
We use the CIs
\begin{eqnarray}
\label{eq:Anew1}
X^1_A &\sqsubseteq& A_1\sqcap A_2\sqcap \exists r.(A\sqcap A_2\sqcap \exists r.X_\Sigma),\\
\label{eq:Anew2}
X^2_A &\sqsubseteq& A_1\sqcap A_2\sqcap \exists r.(A\sqcap A_1\sqcap \exists r.X_\Sigma),
\end{eqnarray}
where, intuitively, the disjunction of $X^1_A$ and $X^2_A$ represents
the most specific~$C_A$. We have 
$\Tmc_2\cup\{(\ref{eq:Anew1}),(\ref{eq:all})\}\models
X^1_{A}\sqsubseteq A$ and thus $A\in\cWtn(\Tmc_1,\Tmc_2)$.
\end{example}

The following result is proved by generalizing the examples given above.
\begin{theorem}\label{th:ellhswtn}
For $\mathcal{EL}$ TBoxes $\Tmc_{1}$ and $\Tmc_{2}$ and signatures $\Sigma$,
$\cWtnr(\Tmc_1,\Tmc_2)$ can be computed in polynomial time.
\end{theorem}

The results stated above can be generalized to extensions of acyclic $\mathcal{EL}$ with
role inclusions and domain and range restrictions and have been implemented in the CEX tool 
for computing logical difference \cite{DBLP:journals/jair/KonevL0W12}.

An alternative approach to computing right-hand $\Sigma$-concept
difference witnesses based on checking for the existence of a simulations between polynomial size hypergraphs has been
introduced in~\cite{DBLP:conf/ecai/Ludwig014}. It has recently been
extended~\cite{GCAI2015:Foundations_for_the_Logical_Difference_of_EL-TBoxes}
to the case of unrestricted \EL{} TBoxes; the hypergraphs then become exponential in the size of the input.

\section{Model Inseparability}
\label{sect:model-inseparability}

We consider inseparability relations according to which two TBoxes
are indistinguishable w.r.t.\ a signature $\Sigma$ in case their models
coincide when restricted to $\Sigma$. A central observation is that
two TBoxes are $\Sigma$-model inseparable iff they cannot be distinguished by
entailment of a second-order (SO) sentence in $\Sigma$. As a consequence, model
inseparability implies concept inseparability for any DL $\Lmc$ and is
thus language independent and very robust. It is particularly useful
when a user is not committed to a certain DL or is interested in more
than just terminological reasoning.

We start this section with introducing model inseparability and the
related notions of model entailment and model conservative
extensions. We then look at the relationship between these notions and
also compare model inseparability to concept inseparability.  We next
discuss complexity. It turns out that model inseparability is
undecidable for almost all DLs, including \EL, with the exception of
some {\sl DL-Lite} dialects. Interestingly, by restricting the signature
$\Sigma$ to be a set concept names, one can often restore
decidability. We then move to model inseparability in the case in
which one TBox is empty, which is of particular interest for
applications in ontology reuse and module extraction. While this
restricted case is still undecidable in \EL, it is decidable for
acyclic $\mathcal{EL}$ TBoxes. We close the section by discussing
approximations of model inseparability that play an important role in
module extraction.

Two interpretation $\Imc$ and $\Jmc$ \emph{coincide for} a signature
$\Sigma$, written $\Imc =_{\Sigma} \Jmc$, if
%
$\Delta^{\Imc} = \Delta^{\Jmc}$
and
$X^{\Imc} = X^{\Jmc}$ for all $X\in \Sigma$.
%
Our central definitions are now as follows.

\begin{definition}[model inseparability, entailment and conservative extensions]\em 
Let $\Tmc_{1}$ and $\Tmc_{2}$ be TBoxes 
and let $\Sigma$ be a signature. Then
\begin{itemize}
\item the \emph{$\Sigma$-model difference} between $\Tmc_{1}$ and $\Tmc_{2}$
is the set ${\sf mDiff}_{\Sigma}(\Tmc_{1},\Tmc_{2})$ of all models $\Imc$ of $\Tmc_{1}$
such that there does not exist a model $\Jmc$ of $\Tmc_{2}$ with $\Jmc =_{\Sigma}\Imc$;

\item  $\Tmc_1$ \emph{$\Sigma$-model entails} $\Tmc_2$ if
      ${\sf mDiff}_{\Sigma}(\Tmc_{1},\Tmc_{2})= \emptyset$;

\item $\Tmc_1$ and $\Tmc_2$ are \emph{$\Sigma$-model inseparable}
      if $\Tmc_{1}$ $\Sigma$-model entails $\Tmc_{2}$ and vice versa;

\item $\Tmc_{2}$ is a \emph{model conservative extension of} $\Tmc_{1}$
      if $\Tmc_{2} \supseteq \Tmc_{1}$ and $\Tmc_{1}$ and $\Tmc_{2}$ are $\sig(\Tmc_{1})$-model inseparable.
\end{itemize}
\end{definition}
Similarly to concept entailment (Example~\ref{ex:ex1}), model
entailment coincides with logical entailment when $\Sigma \supseteq
\mn{sig}(\Tmc_1 \cup\Tmc_2)$. We again recommend to the reader
to verify this to become acquainted with the definitions. 
Also, one can show as in the proof from Example~\ref{ex:ex2} that
definitorial extensions are always model conservative
extensions. 

Regarding the relationship between concept inseparability and model
inseparability, we note that the latter implies the former.  The proof
of the following result goes through for any DL \Lmc that enjoys a
coincidence lemma (that is, for any DL, and even when \Lmc is the set
of all second-order sentences).
\begin{theorem}
\label{thm:afposihd}
  Let $\Tmc_1$ and $\Tmc_2$ be TBoxes formulated in some DL \Lmc and
  $\Sigma$ a signature such that $\Tmc_{1}$ $\Sigma$-model entails
  $\Tmc_2$. Then $\Tmc_{1}$ $\Sigma$-concept entails $\Tmc_{2}$.
\end{theorem}
\begin{proof}
  Suppose $\Tmc_{1}$ $\Sigma$-model entails $\Tmc_{2}$, and let
  $\alpha$ be a $\Sigma$-inclusion in $\Lmc$ such that
  $\Tmc_{2}\models \alpha$. We have to show that $\Tmc_{1}\models
  \alpha$.  Let $\Imc$ be a model of $\Tmc_{1}$. There is a model
  $\Jmc$ of $\Tmc_{2}$ such that $\Jmc =_{\Sigma} \Imc$. Then
  $\Jmc\models \alpha$, and so $\Imc\models \alpha$ since ${\sf
    sig}(\alpha)\subseteq \Sigma$.
\qed
\end{proof}

As noted, Theorem~\ref{thm:afposihd} even holds when \Lmc is the set
of all SO-sentences. Thus, if $\Tmc_{1}$ $\Sigma$-model entails $\Tmc_{2}$
then, for every SO-sentence $\varphi$ in the signature $\Sigma$, $\Tmc_2 \models \varphi$
implies $\Tmc_1 \models \varphi$. It is proved in
\cite{DBLP:series/lncs/KonevLWW09} that, in fact, the latter exactly
characterizes $\Sigma$-model entailment.

The following example shows that concept inseparability in
$\mathcal{ALCQ}$ does not imply model inseparability
(similar examples can be given for any DL and even for full
first-order logic \cite{DBLP:series/lncs/KonevLWW09}).
\begin{example}
Consider the $\mathcal{ALCQ}$ TBoxes and signature from Example~\ref{ALCQ}:
$$
\Tmc_{1}=\{A \sqsubseteq \mathop{\ge 2} r.\top\} \qquad
\Tmc_{2}=\{A \sqsubseteq \exists r.B\sqcap \exists r.\neg B\} \qquad
\Sigma=\{A,r\}.
$$
We have noted in Example~\ref{ALCQ-ALC} that $\Tmc_{1}$ and $\Tmc_{2}$
are $\Sigma$-concept inseparable. However, it is easy to see that the
following interpretation is in ${\sf
  mDiff}_{\Sigma}(\Tmc_{1},\Tmc_{2})$.
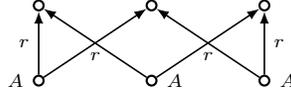
\begin{center}
  \begin{tikzpicture}
    \node (a1) at (0,0) [point, label=left:{\edgefont$A$}]{};%
    \node (a2) at (1.5,0) [point, label=right:{\edgefont$A$}]{};%
    \node (a3) at (3,0) [point, label=right:{\edgefont$A$}]{};%
    \node (b1) at (0,1) [point]{};%
    \node (b2) at (1.5,1) [point]{};%
    \node (b3) at (3,1) [point]{};%
    \draw[role] (a1) to node [left]{\edgefont$r$} (b1);%
    \draw[role] (a1) to node [below]{\edgefont$r$} (b2);%
    \draw[role] (a2) to (b1);%
    \draw[role] (a2) to (b3);%
    \draw[role] (a3) to node [below]{\edgefont$r$} (b2);%
    \draw[role] (a3) to node [right]{\edgefont$r$} (b3);%
  \end{tikzpicture}
\end{center}
\end{example}

We note that the relationship between model-based notions of
conservative extension and language-dependent notions of conservative
extensions was also extensively discussed in the literature on
software specification \cite{ByPitt,Veloso1,Veloso2,Goguen,Maibaum1}.

We now consider the relationship between model entailment and model
inseparability.  As in the concept case, model inseparability is
defined in terms of model entailment and can be decided by two model
entailment checks. Conversely, model entailment can be polynomially
reduced to model inseparability (in constast to concept
inseparability, where this depends on the DL under consideration).

\begin{lemma}\label{lem:entail}
  In any DL \Lmc, model entailment can be polynomially reduced to
  model inseparability.
\end{lemma}
\begin{proof}
  Assume that we want to decide whether $\Tmc_{1}$ $\Sigma$-model entails $\Tmc_{2}$ holds. By replacing
  every non-$\Sigma$-symbol $X$ shared by $\Tmc_{1}$ and $\Tmc_{2}$
  with a fresh symbol $X_1$ in $\Tmc_1$ and a distinct fresh symbol
  $X_2$ in $\Tmc_2$, we can achieve that $\Sigma \supseteq
  \sig(\Tmc_{1}) \cap \sig(\Tmc_{2})$ without changing the original
  (non-)$\Sigma$-model entailment of $\Tmc_2$ by $\Tmc_1$. We then have
  that $\Tmc_{1}$ $\Sigma$-model entails $\Tmc_{2}$ iff
  $\Tmc_{1}$ and $\Tmc_{1}\cup \Tmc_{2}$ are $\Sigma$-model inseparable.
\qed
\end{proof}

The proof of Lemma~\ref{lem:entail} shows that any DL \Lmc is
\emph{robust under joins for model inseparability}, defined
analogously to robustness under joins for concept inseparability; see
Definition~\ref{robustjoin}.


\subsection{Undecidability of model inseparability}

Model-inseparability is computationally much harder than concept
inseparability. In fact, it is undecidable already for $\mathcal{EL}$
TBoxes~\cite{DBLP:journals/ai/KonevL0W13}.  Here, we give a short and
transparent proof showing that model conservative extensions are
undecidable in $\mathcal{ALC}$.  The proof is by reduction of the
following undecidable $\mathbb N \times \mathbb N$ \emph{tiling
  problem}~\cite{Berger66,Robinson71,Borgeretal97}: given a finite set
$\mathfrak T$ of {\em tile types} $T$, each with four colors
$\textit{left}(T)$, $\textit{right}(T)$, $\textit{up}(T)$ and
$\textit{down}(T)$, decide whether $\mathfrak T$ {\em tiles} the grid
$\mathbb N \times\mathbb N$ in the sense that there exists a function
(called a {\em tiling\/}) $\tau$ from $\mathbb N \times\mathbb N$ to
$\mathfrak T$  such that 
\begin{itemize}
\item $\textit{up}(\tau(i,j) ) = \textit{down}(\tau(i,j+1))$ and
\item $\textit{right}(\tau(i,j)) = \textit{left}(\tau(i+1,j))$.
\end{itemize}
If we think of a tile as a physical $1\times 1$-square with a color on
each of its four edges, then a tiling $\tau$ of $\mathbb N
\times\mathbb N$ is just a way of placing tiles, each of a type from
$\mathfrak T$, to cover the $\mathbb N \times\mathbb N$ grid, with no
rotation of the tiles allowed and such that the colors on adjacent edges are
identical.

\begin{theorem}
In \ALC, model conservative extensions are undecidable.
\end{theorem}
\begin{proof}
  Given a set $\mathfrak T$ of tile types, we regard each $T\in
  \mathfrak{T}$ as a concept name and let $x$ and $y$ be role
  names. Let $\Tmc_1$ be the TBox with the following CIs:
\begin{align*}
& \top \sqsubseteq \bigsqcup_{T\in \mathfrak{T}}T,\\
& T \sqcap T' \sqsubseteq \bot, \quad \text{for $T\ne T'$},\\
& T \sqcap \exists x.T' \sqsubseteq \bot, \quad \text{for ${\sf right}(T)\ne {\sf left}(T')$},\\
& T \sqcap \exists y.T'\sqsubseteq \bot, \quad \text{for ${\sf up}(T)\ne {\sf down}(T')$},\\
& 
\top \sqsubseteq \exists x.\top \sqcap \exists y.\top.
\end{align*}
Let $\Tmc_{2}= \Tmc_{1} \cup \Tmc$, where $\Tmc$ consists of a single CI:
$$
\top \sqsubseteq \exists u.(\exists x.B \sqcap \exists x.\neg B) \sqcup \exists u.(\exists y.B \sqcap \exists y.\neg B)
\sqcup \exists u.(\exists x. \exists y.B \sqcap \exists y.\exists x.\neg B),
$$
where $u$ is a fresh role name and $B$ is a fresh concept name. Let
$\Sigma={\sf sig}(\Tmc_{1})$. One can show that $\Tmc$ can be satisfied in a
model $\Jmc=_{\Sigma} \Imc$ iff in $\Imc$ either $x$ is not functional or $y$ is
not functional or $x\circ y \not= y \circ x$. It is not hard to see then that
$\mathfrak{T}$ tiles $\mathbb N \times \mathbb N$ iff $\Tmc_{1}$ and $\Tmc_{2}$
are not $\Sigma$-model inseparable.
\qed
\end{proof}

The only standard DLs for which model inseparability is known to be
decidable are certain {\sl DL-Lite} dialects.  In fact, it is shown in
\cite{KWZ10} that $\Sigma$-model entailment between TBoxes in the
extensions of $\DLc$ with Boolean operators and unqualified number
restrictions is decidable. The computational complexity remains open
and for the extension $\DLcH$ of $\DLc$ with role hierarchies, even
decidability is open.  The decidability proof given in \cite{KWZ10} is
by reduction to the two-sorted first-order theory of Boolean algebras
(BA) combined with Presburger arithmetic (PA) for representing
cardinalities of sets. The decidability of this theory, called BAPA,
has been first proved in \cite{Feferman&Vaught59}. Here we do not go
into the decidability proof, but confine ourselves to giving an
instructive example which shows that \emph{uncountable} models have to
be considered when deciding model entailment in $\DLc$ extended with
unqualified number restrictions~\cite{KWZ10}.
\begin{example}\label{exxx}
The TBox $\mathcal{T}_{1}$ states, using auxiliary role names $r$ and $s$,
that the extension of the concept name $B$ is infinite: 
%
\begin{eqnarray*}
\mathcal{T}_1  & = &\{ \top \sqsubseteq \exists r.\top, \
\exists r^-.\top \sqsubseteq \exists s.\top, \
\exists s^{-}.\top \sqsubseteq B, \\
               &  & \ \ B \sqsubseteq \exists s.\top, \ (\mathop{\geq
                 2} s^{-}.\top) \sqsubseteq \bot, \ \exists r^{-}.\top \sqcap \exists s^{-}.\top \sqsubseteq \bot\}.
\end{eqnarray*}
The TBox $\mathcal{T}_{2}$ states that $p$ is an injective function from $A$ to $B$:
\begin{equation*}
\mathcal{T}_{2} = \{ A \equiv \exists p.\top, \ \  \exists p^{-}.\top \sqsubseteq B,\ \
(\mathop{\ge 2} p.\top) \sqsubseteq \bot,\ \ (\mathop{\ge 2} p^-.\top) \sqsubseteq \bot
\}.
\end{equation*}
Let $\Sigma = \{A,B\}$. There exists an uncountable model $\mathcal{I}$ of $\mathcal{T}_{1}$ with uncountable 
$A^{\mathcal{I}}$ and at most countable $B^{\mathcal{I}}$. Thus, there is no injection from $A^{\mathcal{I}}$ to  $B^{\mathcal{I}}$, and so $\mathcal{I} \in {\sf mDiff}_{\Sigma}(\mathcal{T}_{1},\mathcal{T}_{2})$
and $\mathcal{T}_{1}$ does not $\Sigma$-model entail $\mathcal{T}_{2}$. Observe, however, that if $\mathcal{I}$ is a 
countably infinite model of $\mathcal{T}_{1}$, then there is always an injection from  $A^{\mathcal{I}}$ to 
$B^{\mathcal{I}}$. Thus, in this case there exists a model $\mathcal{I}'$ of
$\mathcal{T}_{2}$ with $\Imc'=_{\Sigma} \Imc$. It follows that uncountable models of $\mathcal{T}_{1}$
are needed to prove that $\T_{1}$ does not $\Sigma$-model entail $\mathcal{T}_{2}$.
\end{example}

An interesting way to make $\Sigma$-model inseparability decidable is
to require that $\Sigma$ contains only concept names. We show that, in
this case, one can use the standard filtration technique from modal
logic to show that there always exists a counterexample to
$\Sigma$-model inseparability of at most exponential size (in sharp
contrast to Example~\ref{exxx}).
\begin{lemma}\label{expmodelproperty}
Suppose $\Tmc_{1}$ and $\Tmc_{2}$ are $\mathcal{ALC}$ TBoxes and $\Sigma$ contains
concept names only. If ${\sf mDiff}_{\Sigma}(\Tmc_{1},\Tmc_{2})\not=\emptyset$, then
there is an interpretation $\Imc$  in
${\sf mDiff}_{\Sigma}(\Tmc_{1},\Tmc_{2})$ such that $|\Delta^{\Imc}|\leq 2^{|\Tmc_{1}|+|\Tmc_{2}|}$.
\end{lemma}
\begin{proof}
  Assume $\Imc\in {\sf mDiff}_{\Sigma}(\Tmc_{1},\Tmc_{2})$. Define an
  equivalence relation ${\sim} \subseteq \Delta^{\Imc}\times \Delta^{\Imc}$ by
  setting $d\sim d'$ iff, for all $C\in {\sf sub}(\Tmc_{1}\cup \Tmc_{2})$, we
  have $d\in C^{\Imc_{1}}$ iff $d'\in C^{\Imc_{2}}$. Let
  $[d]=\{ d'\in \Delta^{\Imc}\mid d'\sim d\}$.  Define an
  interpretation $\Imc'$ by taking
\begin{eqnarray*}
\Delta^{\Imc'} & = & \{ [d] \mid d\in \Delta^{\Imc}\},\\
A^{\Imc'} & = & \{ [d] \mid d\in A^{\Imc}\} \mbox{ for all $A\in {\sf sub}(\Tmc_{1})$, }\\
r^{\Imc'} & = & \{ ([d],[d'])\mid \exists e\in [d]\exists e'\in [d'] (e,e')\in r^{\Imc}\} \mbox{ for all role names $r$}.
\end{eqnarray*}
It is not difficult to show that $d\in C^{\Imc}$ iff $[d]\in C^{\Imc'}$ for all $d\in \Delta^{\Imc}$
and $C\in {\sf sub}(\Tmc_{1})$. Thus $\Imc'$ is a model of $\Tmc_{1}$. We now show that there does
not exist a model $\Jmc'$ of $\Tmc_{2}$ with $\Imc'=_{\Sigma} \Jmc'$. For a proof by contradiction,
assume that such a $\Jmc'$ exists. We define a model $\Jmc$ of $\Tmc_{2}$ with $\Jmc =_{\Sigma} \Imc$,
and thus derive a contradiction to the assumption that $\Imc\in  {\sf mDiff}_{\Sigma}(\Tmc_{1},\Tmc_{2})$.
To this end, let $A^{\Jmc}=A^{\Imc}$ for all $A\in \Sigma$ and set
\begin{eqnarray*}
A^{\Jmc} & = & \{ d \mid [d]\in A^{\Jmc}\} \mbox{ for all $A\not\in\Sigma$, }\\
r^{\Jmc} & = & \{ (d,d') \mid ([d],[d'])\in r^{\Jmc}\} \mbox{ for all role names $r$. }
\end{eqnarray*}
Note that the role names (which are all not in $\Sigma$), are
interpreted in a `maximal' way. It can be proved that $d\in C^{\Jmc}$
iff $[d]\in C^{\Jmc'}$ for all $d\in \Delta^{\Imc}$ and $C\in {\sf
  sub}(\Tmc_{2})$. Thus $\Jmc$ is a model of $\Tmc_2$ and we have
derived a contradiction.
\qed
\end{proof}

Using the bounded model property established in
Lemma~\ref{expmodelproperty}, one can prove a $\coNExp^\NP$ upper
bound for model inseparability. A matching lower bound and several
extensions of this result are proved
in~\cite{DBLP:journals/ai/KonevL0W13}.
\begin{theorem}
\label{thm:conceptmod}
  In \ALC, $\Sigma$-model inseparability is $\coNExp^\NP$-complete
  when $\Sigma$ is restricted to sets of concept names.
\end{theorem}
\begin{proof}
  We sketch the proof of the upper bound. It is sufficient to show
  that one can check in $\NExp^\NP$ whether ${\sf
    mDiff}_{\Sigma}(\Tmc_{1},\Tmc_{2})\not=\emptyset$. By
  Lemma~\ref{expmodelproperty}, one can do this by guessing a model
  $\Imc$ of $\Tmc_1$ of size at most $2^{|\Tmc_{1}|+|\Tmc_{2}|}$ and
  then calling an oracle to verify that there is no model $\Jmc$ of
  $\Tmc_{2}$ with $\Jmc=_{\Sigma} \Imc$. The oracle runs in \NP
  since we can give it the guessed \Imc as an input, thus we are 
  asking for a model of $\Tmc_2$ of size polynomial in the size
  of the oracle input.

  The lower bound is proved in \cite{DBLP:journals/ai/KonevL0W13} by a
  reduction of satisfiability in circumscribed $\mathcal{ALC}$ KBs, which
  is known to be $\coNExp^\NP$-hard.
  \qed
\end{proof}

In~\cite{DBLP:journals/ai/KonevL0W13}, Theorem~\ref{thm:conceptmod} is
generalized to \ALCI. We conjecture that it can be further
extended to most standard DLs that admit the finite model property.
For DLs without the finite model property such as $\mathcal{ALCQI}$, we expect that
BAPA-based techniques, as used for circumscription in
\cite{DBLP:conf/birthday/BonattiFLSW14}, can be employed to obtain an analog of
Theorem~\ref{thm:conceptmod}.

\subsection{Model inseparability from  the empty TBox}

We now consider model inseparability in the case where one TBox is
empty.  To motivate this important case, consider the application of
ontology reuse, where one wants to import a TBox $\Tmc_{{\sf im}}$
into a TBox $\Tmc$ that is currently being developed.
%
%
Recall that the
result of importing $\Tmc_{\sf im}$ in $\Tmc$ is the union $\Tmc\cup
\Tmc_{{\sf im}}$ and that, when importing $\Tmc_{{\sf im}}$ into
$\Tmc$, the TBox $\Tmc$ is not supposed to interfere with the
modeling of the symbols from $\Tmc_{\sf im}$. 
We can formalize this requirement by demanding that
\begin{itemize}
\item $\Tmc\cup \Tmc_{{\sf im}}$ and $\Tmc_{\sf im}$ are $\Sigma$-model inseparable for $\Sigma= \sig(\Tmc_{\sf im})$.
\end{itemize}
In this scenario, one has to be prepared for the imported TBox
$\Tmc_{\sf im}$ to be revised. Thus, one would like to design
the importing TBox \Tmc such that \emph{any} TBox $\Tmc_{\sf im}$ can
be imported into \Tmc without undesired interaction as long as the
signature of $\Tmc_{\sf im}$ is not changed. Intuitively, $\Tmc$
provides a \emph{safe interface for importing ontologies} that only
share symbols from some fixed signature $\Sigma$ with $\Tmc$.  This
idea led to the definition of safety for a signature
in~\cite{JairGrau}:
\begin{definition}\label{def:safety}\em 
  Let \Tmc be an \Lmc TBox. We say that $\Tmc$ is \emph{safe for a
    signature $\Sigma$ under model inseparability} if $\Tmc\cup
  \Tmc_{{\sf im}}$ is $\mn{sig}(\Tmc_{\sf im})$-model inseparable from
  $\Tmc_{{\sf im}}$ for all \Lmc TBoxes $\Tmc_{\sf im}$ with
  $\sig(\Tmc) \cap \sig(\Tmc_{\sf im})\subseteq \Sigma$.
\end{definition}

As one quantifies over all TBoxes $\Tmc_{{\sf im}}$ in
Definition~\ref{def:safety}, safety for a signature seems hard to deal
with algorithmically. Fortunately, it turns out that the
quantification can be avoided. This is related to the following
robustness property.\footnote{Similar robustness properties and  
notions of equivalence have been discussed in logic programming, we refer the reader to~\cite{Maher,DBLP:journals/tocl/LifschitzPV01,DBLP:conf/iclp/EiterF03} and references therein. We will discuss this robustness property further in Section~\ref{sect:final}.}
\begin{definition}\em\label{def:rob_under_replacement}
  A DL \Lmc is said to be \emph{robust under replacement for model
    inseparability} if, for all \Lmc TBoxes $\Tmc_1$ and $\Tmc_2$ and
  signatures $\Sigma$, the following condition is satisfied: if
  $\Tmc_{1}$ and $\Tmc_{2}$ are $\Sigma$-model inseparable, then $\Tmc_{1}\cup
  \Tmc$ and $\Tmc_{2}\cup \Tmc$ are $\Sigma$-model inseparable for all \Lmc TBoxes $\Tmc$
  with ${\sf \sig}(\Tmc) \cap {\sf sig}(\Tmc_{1}\cup \Tmc_{2})
  \subseteq \Sigma$.
\end{definition}

The following has been observed in \cite{JairGrau}. It again applies
to any standard DL, and in fact even to second-order logic. 
\begin{theorem}
  In any DL \Lmc, model inseparability is robust under replacement.
\end{theorem}

Using robustness under replacement, it can be proved that safety for a
signature is nothing but inseparability from the empty TBox, in this
way eliminating the quantification over TBoxes used in the original
definition. This has first been observed in \cite{JairGrau}. The
connection to robustness under replacement is from
\cite{DBLP:journals/ai/KonevL0W13}.
\begin{theorem}\label{th:safety}
A TBox $\Tmc$ is safe for a signature $\Sigma$ under model-inseparability iff 
$\Tmc$ is $\Sigma$-model inseparable from the empty TBox.
\end{theorem}
\begin{proof}
  Assume first that $\Tmc$ is not $\Sigma$-model inseparable from
  $\emptyset$.  Then $\Tmc\cup \Tmc_{im}$ is not $\Sigma$-model
  inseparable from $\Tmc_{im}$, where $\Tmc_{im}$ is the trivial
  $\Sigma$-TBox $\Tmc_{im}=\{A \sqsubseteq A \mid A\in \Sigma\cap
  \NC\}\cup \{ \exists r.\top \sqsubseteq \top\mid r\in \Sigma\cap
  \NR\}$.  Hence $\Tmc$ is not safe for $\Sigma$. Now assume $\Tmc$ is
  $\Sigma$-model inseparable from $\emptyset$ and let $\Tmc_{im}$ be a
  TBox such that ${\sf sig}(\Tmc)\cap {\sf sig}(\Tmc_{im})\subseteq
  \Sigma$.  Then it follows from robustness under replacement that
  $\Tmc \cup \Tmc_{im}$ is ${\sf sig}(\Tmc_{im})$-model inseparable from
  $\Tmc_{im}$.
\qed
\end{proof}

By Theorem~\ref{th:safety}, deciding safety of a TBox \Tmc for a
signature $\Sigma$ under model inseparability amounts to checking
$\Sigma$-model inseparability from the empty TBox.  We thus consider
the latter problem as an important special case of model
inseparability.  Unfortunately, even in $\mathcal{EL}$, model
inseparability from the empty TBox is undecidable
\cite{DBLP:journals/ai/KonevL0W13}.
\begin{theorem}
  In \EL, model inseparability from the empty TBox is undecidable.
\end{theorem}

We now consider acyclic $\mathcal{EL}$ TBoxes as an important special
case. As we have mentioned before, many large-scale TBoxes are in
fact acyclic $\mathcal{EL}$ TBoxes or mild extensions thereof.
Interestingly, model inseparability of acyclic TBoxes from the empty
TBox can be decided in polynomial time
\cite{DBLP:journals/ai/KonevL0W13}. The approach is based on a
characterization of model inseparability from the empty TBox in terms
of certain syntactic and semantic \emph{dependencies}. The following
example shows two cases of how an acyclic $\mathcal{EL}$ TBox can fail
to be model inseparable from the empty TBox. These two cases will then
give rise to two types of syntactic dependencies.
\begin{example}
\label{ex:dependencies}
(a)~Let 
$\Tmc = \{A \sqsubseteq \exists r.B, B \sqsubseteq \exists s.E\}$ and
$\Sigma= \{A,s\}$. 
Then
$\Tmc$ is not $\Sigma$-model inseparable from the empty TBox: for the interpretation \Imc with
$\Delta^{\Imc}=\{d\}$, $A^{\Imc}=\{d\}$, and $s^{\Imc}=\emptyset$,
there is no model \Jmc of \Tmc with $\Jmc =_{\Sigma} \Imc$.

\smallskip
(b) Let $\Tmc = \{ A_1 \sqsubseteq \exists r . B_1, A_2
\sqsubseteq \exists r . B_2, A \equiv B_1 \sqcap B_2 \}$ and $\Sigma =
\{A_1,A_2,A\}$. Then
$\Tmc$ is not $\Sigma$-model inseparable from the empty TBox: for the
interpretation \Imc with $\Delta^{\Imc}=\{d\}$,
$A_1^{\Imc}=A_2^\Imc=\{d\}$, and $A^{\Imc}=\emptyset$, there is no
model \Jmc of \Tmc with $\Jmc =_{\Sigma} \Imc$.
\end{example}

Intuitively, in part~(a) of Example~\ref{ex:dependencies}, the reason
for separability from the empty TBox is that we can start with a
$\Sigma$-concept name that occurs on some left-hand side (which is
$A$) and then deduce from it that another $\Sigma$-symbol (which is
$s$) must be non-empty. Part~(b) is of a slightly different nature. We
start with a set of $\Sigma$-concept names (which is $\{A_1,A_2\}$)
and from that deduce a set of concepts that implies another
$\Sigma$-concept (which is $A$) via a concept definition,
right-to-left. It turns out that it is convenient to distinguish
between these two cases also in general. We first introduce some
notation. For an acyclic TBox \Tmc, let
\begin{itemize}

\item ${\sf lhs}(\Tmc)$ denote the set of concept names $A$ such that
  there is some CI $A \equiv C$ or $A \sqsubseteq C$ in $\Tmc$;

\item ${\sf def}(\Tmc)$ denote the set of concept names $A$ such that
  there is a definition $A \equiv C$ in $\Tmc$;

\item $\mn{depend}^\equiv_\Tmc(A)$ be defined exactly as
  $\mn{depend}_\Tmc(A)$ in Section~\ref{sec:introdesc}, except that
  only concept definitions $A \equiv C$ are considered while concept
  inclusions $A \sqsubseteq C$ are disregarded.

\end{itemize}

\begin{definition}
\label{def:dependencies}\em
  Let $\Tmc$ be an acyclic $\mathcal{EL}$ TBox,  $\Sigma$ a
  signature, and $A \in \Sigma$. We say that

\begin{itemize}
\item $A$ \emph{has a direct $\Sigma$-dependency in} \Tmc if
${\sf depend}_{\Tmc}(A)\cap \Sigma\not=\emptyset$;

\item $A$ \emph{has an indirect $\Sigma$-dependency in} \Tmc if $A \in
  \mn{def}(\Tmc) \cap \Sigma$ and there are
  $A_1,\dots,A_n \in \mn{lhs}(\Tmc) \cap \Sigma$ such that $A \notin \{ A_1,\dots,A_n\}$ and
$$
\mn{depend}^\equiv_\Tmc(A) \setminus \mn{def}(\Tmc) \subseteq \bigcup_{1 \leq i \leq n} \mn{depend}_{\Tmc}(A_i).
$$
\end{itemize}

We say that \Tmc \emph{contains an} (\emph{in})\emph{direct $\Sigma$-dependency} if
there is an $A \in \Sigma$ that has an (in)direct  $\Sigma$-dependency in \Tmc.
%
\end{definition}

It is proved in \cite{DBLP:journals/ai/KonevL0W13} that, for every
acyclic $\mathcal{EL}$ TBox $\Tmc$ and signature $\Sigma$, $\Tmc$ is
$\Sigma$-model inseparable from the empty TBox iff $\Tmc$ has neither
direct nor indirect $\Sigma$-dependencies. It can be decided in \PTime
in a straightforward way whether a given \EL TBox contains a direct
$\Sigma$-dependency.  For indirect $\Sigma$-dependencies, this is less
obvious since we start with a set of concept names
from $\mn{lhs}(\Tmc) \cap \Sigma$. Fortunately, it can be shown that if a
concept name $A \in \Sigma$ has an indirect $\Sigma$-dependency in
\Tmc induced by concept names $A_1,\dots,A_n \in \mn{lhs}(\Tmc) \cap
\Sigma$, then $A$ has an indirect $\Sigma$-dependency in \Tmc induced
by the set of concept names $(\mn{lhs}(\Tmc) \cap \Sigma)
\setminus\{A\}$. We thus only need to consider the latter set.

\begin{theorem}\label{thm:acyclelp}
  In \EL, model inseparability of acyclic TBoxes from the empty TBox
  is in \PTime.
\end{theorem}

Also in \cite{DBLP:journals/ai/KonevL0W13}, Theorem~\ref{thm:acyclelp}
is extended from \EL to $\mathcal{ELI}$, and it is shown that, in \ALC
and \ALCI, model inseparability from the empty TBox is
$\Pi^p_2$-complete for acyclic TBoxes.

\subsection{Locality-based approximations}\label{sec:locality}

We have seen in the previous section that model inseparability from
the empty TBox is of great practical value in the context of ontology
reuse, that it is undecidable even in \EL, and that decidability can
(sometimes) be regained by restricting TBoxes to be acyclic. In the
non-acyclic case, one option is to resort to approximations from
above. This leads to the (semantic) notion of $\emptyset$-locality and
its syntactic companion $\bot$-locality. We discuss the former in this
section and the latter in Section~\ref{sect:final}.

A TBox $\Tmc$ is called \emph{$\emptyset$-local w.r.t.\ a signature}
$\Sigma$ if, for every interpretation $\Imc$, there exists a model
$\Jmc$ of \Tmc such that $\Imc =_\Sigma \Jmc$ and
$A^{\Jmc} = r^\Jmc = \emptyset$, for all $A\in\NC\setminus\Sigma$ and
$r\in\NR\setminus\Sigma$; in other words, every interpretation of
$\Sigma$-symbols can be trivially extended to a model of $\Tmc$ by
interpreting non-$\Sigma$ symbols as the empty set. Note that, if
$\Tmc$ is $\emptyset$-local w.r.t.\ $\Sigma$, then it is
$\Sigma$-model inseparable from the empty TBox and thus, by
Theorem~\ref{th:safety}, safe for $\Sigma$ under model inseparability.
The following example shows that the converse does not hold.
\begin{example}
Let $\Tmc = \{A \sqsubseteq B\}$ and $\Sigma = \{A\}$. Then $\Tmc$ is $\Sigma$-model inseparable from $\emptyset$, but $\Tmc$ is not $\emptyset$-local w.r.t.\ $\Sigma$.
\end{example}

In contrast to model inseparability, $\emptyset$-locality is decidable
also in \ALC and beyond, and is computationally not harder than
standard reasoning tasks such as satisfiability. The next procedure for checking 
$\emptyset$-locality was given in~\cite{DBLP:series/lncs/GrauHKS09}.
\begin{theorem}\label{th:sem_locality}
Let $\Tmc$ be an \ALCQI TBox and $\Sigma$ a signature. Suppose  
$\Tmc \!\!\downharpoonright_{\Sigma = \emptyset}$ is obtained from $\Tmc$ by
replacing all concepts of the form $A$, $\exists r.C$, $\exists r^-.C$,
$(\geq n\, r.C)$ and
$(\geq n\, r^-.C)$
with
$\bot$ whenever $A\notin\Sigma$ and $r\notin\Sigma$. Then $\Tmc$ is $\emptyset$-local w.r.t.\ $\Sigma$ iff
$\Tmc\!\!\downharpoonright_{\Sigma = \emptyset}$ is logically equivalent to the empty TBox.
\end{theorem}

While Theorem~\ref{th:sem_locality} is stated here for \ALCQI---the
most expressive DL considered in this paper---the original result
in~\cite{DBLP:conf/www/GrauHKS07} is more general and applies to
$\mathcal{SHOIQ}$ knowledge bases. There is also a dual notion of
$\Delta$-locality \cite{JairGrau}, in which non-$\Sigma$ symbols are
interpreted as the entire domain and which can also be reduced to
logical equivalence.  

We also remark that, unlike model inseparability from the empty TBox,
model inseparability cannot easily be reduced to logical equivalence
in the style of Theorem~\ref{th:sem_locality}.
\begin{example}
Let  $\Tmc = \{A\sqsubseteq B\sqcup C\}$, $\Tmc' =
\{A\sqsubseteq B\}$ and $\Sigma = \{A,B\}$. Then the TBoxes
$\Tmc\!\!\downharpoonright_{\Sigma = \emptyset}$ and
$\Tmc'\!\!\downharpoonright_{\Sigma=\emptyset}$ are logically equivalent, yet $\Tmc$
is not $\Sigma$-model inseparable from $\Tmc'$.
\end{example}

$\emptyset$-locality and its syntactic companion $\bot$-locality are
prominently used in ontology
modularization~\cite{DBLP:conf/ijcai/GrauHKS07,JairGrau,DBLP:conf/www/GrauHKS07,DBLP:conf/dlog/SattlerSZ09}.
A subset $\Mmc$ of $\Tmc$ is called a \emph{$\emptyset$-local
  $\Sigma$-module} of $\Tmc$ if $\Tmc\setminus\Mmc$ is
$\emptyset$-local w.r.t.\ $\Sigma$.  It can be shown that every
$\emptyset$-local $\Sigma$-module \Mmc of $\Tmc$ is self-contained
(that is, $\Mmc$ is $\Sigma$-model inseparable from $\Tmc$) and
depleting (that is, $\Tmc\setminus\Mmc$ is
$\Sigma\cup\sig(\Mmc)$-model inseparable from the empty TBox).  In
addition, $\emptyset$-local modules are also
\emph{subsumer-preserving}, that is, for every $A\in\Sigma\cap\NC$ and
$B\in\NC$, if $\Tmc\models A\sqsubseteq B$ then $\Mmc\models
A\sqsubseteq B$.  This property is particular useful in modular
reasoning~\cite{DBLP:conf/ore/RomeroGHJ13,DBLP:conf/dlog/RomeroGH12,DBLP:journals/jair/RomeroKGH16}.

A $\emptyset$-local module of a given ontology $\Tmc$ for a given
signature $\Sigma$ can be computed in a straightforward way as
follows.  Starting with $\Mmc = \emptyset$, iteratively add to $\Mmc$
every $\alpha\in\Tmc$ such that
$\alpha\!\!\downharpoonright_{\Sigma\cup\sig(\Mmc)=\emptyset}$ is not
a tautology until $\Tmc\setminus\Mmc$ is $\emptyset$-local w.r.t.\
$\Sigma\cup\sig(\Mmc)$. The resulting module might be larger than
necessary because this procedure actually generates a
$\emptyset$-local $\Sigma \cup \mn{sig}(\Mmc)$-module rather than only
a $\Sigma$-module and because $\emptyset$-locality overapproximates
model inseparability, but in most practical cases results in
reasonably small modules~\cite{DBLP:conf/ijcai/GrauHKS07}.

\section{Query Inseparability for KBs}\label{sect:query-separability}

In this section, we consider inseparability of KBs rather than of TBoxes.  One main
application of KBs is to provide access to the data stored in their
ABox by means of database-style queries, also taking into account the
knowledge from the TBox to compute more complete answers.  This
approach to querying data is known as \emph{ontology-mediated
  querying} \cite{DBLP:journals/tods/BienvenuCLW14}, and it is a core
part of the \emph{ontology-based data access} (OBDA) paradigm
\cite{PLCD*08}. In many applications of KBs, a reasonable notion of
inseparability is thus the one where both KBs are required to give the
same answers to all relevant queries that a user might pose.  Of
course, such an inseparability relation depends on the class of
relevant queries and on the signature that we are allowed to use in
the query. 
We will consider the two most important query languages, which are conjunctive
queries (CQs) and unions thereof (UCQs), and their rooted fragments, rCQs and
rUCQs.

%
%

We start the section by introducing query inseparability of KBs and
related notions of query entailment and query conservative
extensions. We then discuss the connection to the logical equivalence
of KBs, how the choice of a query language impacts query
inseparability, and the relation between query entailment and query
inseparability. Next, we give model-theoretic characterizations of
query inseparability which are based on model classes that are
complete for query answering and on (partial or full)
homomorphisms. We then move to decidability and complexity, starting
with $\mathcal{ALC}$ and then proceeding to \DLite, \EL, and
Horn-\ALC.  In the case of \ALC, inseparability in terms of CQs turns
out to be undecidable while inseparability in terms of UCQs is
decidable in 2\ExpTime (and the same is true for the rooted versions
of these query languages). In the mentioned Horn DLs, CQ
inseparability coincides with UCQ inseparability and is decidable,
with the complexity ranging from {\sc PTime} for \EL via {\sc ExpTime}
for $\DLcH$, $\DLhH$, and Horn-$\ALC$ to {\sc 2ExpTime} for
Horn-$\mathcal{ALCI}$.
\begin{definition}[query inseparability, entailment and conservative extensions]\em\label{def1}
  Let $\Kmc_1$ and $\Kmc_2$ be KBs, $\Sigma$ a signature, and
  $\mathcal{Q}$ a class of queries. Then
\begin{itemize}
\item the \emph{$\Sigma\text{-}\mathcal{Q}$ difference} between $\Kmc_{1}$ and $\Kmc_{2}$ is the set ${\sf qDiff}_{\Sigma}^{\mathcal{Q}}(\Kmc_{1},\Kmc_{2})$ of all $\q(\vec{a})$ such that $\q(\vec{x}) \in \mathcal{Q}_{\Sigma}$, $\vec{a} \subseteq \ind(\Amc_{2})$, $\Kmc_{2}\models \q(\vec{a})$ and $\Kmc_{1}\not\models \q(\vec{a})$;

\item \emph{$\Kmc_1$ $\Sigma\text{-}\mathcal{Q}$ entails} $\Kmc_2$ if
      ${\sf qDiff}_{\Sigma}^{\mathcal{Q}}(\Kmc_{1},\Kmc_{2})= \emptyset$;
\item $\Kmc_1$ and $\Kmc_2$ are \emph{$\Sigma\text{-}\mathcal{Q}$ inseparable}
      if $\Kmc_{1}$ $\Sigma\text{-}\mathcal{Q}$ entails $\Kmc_2$ and vice versa;

\item $\Kmc_{2}$ is a \emph{$\mathcal{Q}$-conservative extension of} $\Kmc_{1}$
      if $\Kmc_{2} \supseteq \Kmc_{1}$, and $\Kmc_{1}$ and $\Kmc_{2}$
      are $\sig(\Kmc_{1})\text{-}\mathcal{Q}$ inseparable.
\end{itemize}
If $\q(\vec{a}) \in {\sf qDiff}_{\Sigma}^{\mathcal{Q}}(\Kmc_{1},\Kmc_{2})$, then we say that 
$\q(\vec{a})$ \emph{$\Sigma$-$\mathcal{Q}$ separates} $\Kmc_{1}$ and $\Kmc_{2}$.
\end{definition}

Note that slight variations of the definition of query inseparability
are possible; for example, one can allow signatures to also contain
individual names and then consider only query answers that consist of
these names \cite{BotoevaKRWZ16}.

Query inseparability is a coarser relationship between KBs than
logical equivalence even when $\Sigma \supseteq \mn{sig}(\Kmc_1) \cup
\mn{sig}(\Kmc_2)$. Recall that this is in sharp contrast to concept
and model inseparability, for which we observed that they coincide
with logic equivalence under analogous assumptions on $\Sigma$.
\begin{example}[query inseparability and logical equivalence]
Let $\Kmc_{i}= (\Tmc_{i},\Amc_{i})$, $i=1,2$, where $\Amc_{1}=\{A(c)\}$, $\Tmc_{1}= \{ A \sqsubseteq B\}$,
$\Amc_{2}= \{A(c),B(c)\}$, and $\Tmc_{2}= \emptyset$. Then $\Kmc_{1}$ and $\Kmc_{2}$ are $\Sigma$-UCQ inseparable
for any signature $\Sigma$ but clearly $\Kmc_{1}$ and $\Kmc_{2}$ are not logically equivalent.
\end{example}
This example shows that there are drastic logical differences between
KBs that cannot be detected by UCQs. This means that, when we aim to
replace a KB with a query inseparable one, we have significant freedom
to modify the KB. In the example above, we went from a KB with a
non-empty TBox to a KB with an empty TBox, which should be easier to
deal with when queries have to be answered efficiently.

We now compare the notions of \Qmc inseparability induced by different
choices of the query language \Qmc. A first observation is that, for
Horn DLs such as \EL, there is no difference between UCQ
inseparability and CQ inseparability. The same applies to rCQs and
rUCQs. This follows from the fact that KBs formulated in a Horn DL
have a universal model, that is, a single model that gives the same
answers to queries as the KB itself---see
Section~\ref{subsect:blaaaaa} for more details.\footnote{In fact, when
  we say `Horn DL', we mean a DL in which every KB has a universal
  model.}
\begin{theorem}\label{thm:e}
Let $\Kmc_1$ and $\Kmc_2$ be KBs formulated in a Horn DL, and let
$\Sigma$ be a signature. Then
\begin{description}
\item[\rm (\emph{i})]$\Kmc_{1}$ $\Sigma$-UCQ entails $\Kmc_{2}$ iff $\Kmc_{1}$ $\Sigma$-CQ entails $\Kmc_{2}$;
\item[\rm (\emph{ii})]$\Kmc_{1}$ $\Sigma$-\RUCQ{} entails $\Kmc_{2}$ iff $\Kmc_{1}$ $\Sigma$-rCQ entails $\Kmc_{2}$.
\end{description}
\end{theorem}

The equivalences above do not hold for DLs that are not Horn, as shown
by the following example:
\begin{example}\label{UCQ-CQ}
  Let $\Kmc_{i}= (\Tmc_{i},\Amc)$, for $i=1,2$, be the \ALC KBs where
  $\Tmc_1=\emptyset$, $\Tmc_2=\{ A \sqsubseteq B_{1} \sqcup B_{2}\}$,
  and $\Amc=\{A(c)\}$. Let $\Sigma = \{A,B_{1},B_{2}\}$. Then $\Kmc_{1}$ $\Sigma$-CQ entails $\Kmc_2$, but the UCQ
  (actually rUCQ) $\q(x) = B_{1}(x) \vee B_{2}(x)$ shows that
  $\Kmc_{1}$ does not $\Sigma$-UCQ entail $\Kmc_2$.
\end{example}

As in the case of concept and model inseparability (of TBoxes), it is
instructive to consider the connection between query entailment and
query inseparability. As before, query inseparability is defined in
terms of query entailment. The converse direction is harder to
analyze. Recall that, for concept and model inseparability, we
employed robustness under joins to reduce entailment to
inseparability. Robustness under joins is defined as follows for
$\Sigma$-\Qmc-inseparability: if $\Sigma\supseteq {\sf
  sig}(\Kmc_{1})\cap {\sf sig}(\Kmc_{2})$, then $\Kmc_{1}$
$\Sigma$-$\mathcal{Q}$ entails $\Kmc_{2}$ iff $\Kmc_{1}$ and
$\Kmc_{1}\cup \Kmc_{2}$ are $\Sigma$-$\mathcal{Q}$
inseparable. Unfortunately, this property does not hold.

\begin{example}
\label{ex:blablo}
Let $\Kmc_{i}= (\Tmc_{i},\Amc)$, $i=1,2$, be Horn-\ALC KBs with
$$
\Tmc_{1}=\{ A \sqsubseteq \exists r.B \sqcap \exists r.\neg B\}, \quad 
\Tmc_{2} = \{ A \sqsubseteq \exists r.B \sqcap \forall r.B\}, \quad \Amc=\{A(c)\}. 
$$
Let $\Sigma=\{A,B,r\}$. Then, for any class of
queries $\mathcal{Q}$ introduced above, $\Kmc_{1}$
$\Sigma$-$\mathcal{Q}$ entails $\Kmc_{2}$ but $\Kmc_{1}$ and
$\Kmc_{1}\cup \Kmc_{2}$ are not $\Sigma$-$\mathcal{Q}$ inseparable
since $\Kmc_{1}\cup \Kmc_{2}$ is not satisfiable.
\end{example}

Robustness under joins has not yet been studied systematically for
query inseparability. While Example~\ref{ex:blablo} shows that query
inseparability does not enjoy robustness under joins in Horn-\ALC, it
is open whether the same is true in \EL and the {\sl DL-Lite} family.
Interestingly, there is a (non-trivial) polynomial time reduction of query
entailment to query inseparability that works for many Horn DLs and
does not rely on robustness under joins \cite{BotoevaKRWZ16}.  
%
\begin{theorem}
$\Sigma$-CQ entailment of KBs is polynomially reducible to $\Sigma$-CQ inseparability of KBs for any Horn DL containing $\EL$ or $\DLcH$, and contained in $\hALCHI$.
\end{theorem}
%
%
%
%
%
%


\subsection{Model-theoretic criteria for query inseparability}
\label{subsect:blaaaaa}

We now provide model-theoretic characterizations of query inseparability.
Recall that query inseparability is defined in terms of certain
answers and that, given a KB $\Kmc$ and a query $\q(\vec{x})$, a tuple
$\vec{a} \subseteq \ind(\Kmc)$ is a certain answer to $\q(\vec{x})$
over $\Kmc$ iff, for every model $\Imc$ of $\Kmc$, we have
$\Imc \models \q(\vec{a})$. It is well-known that, in many cases, it
is actually not necessary to consider \emph{all} models \Imc of \Kmc
to compute certain answers.  We say that a class $\Mod$ of models of $\Kmc$
is \emph{complete for $\Kmc$ and a class $\mathcal{Q}$ of queries}
if, for every $\q(\avec{x})\in \mathcal{Q}$, we have
$\Kmc \models \q(\avec{a})$ iff $\Imc \models \q(\avec{a})$ for all
$\Imc\in \Mod$.

In the following, we give some important examples of model classes for
which KBs are complete.
\begin{example}\label{ex:complete}
  Given an $\mathcal{ALC}$ KB $\K=(\T,\A)$, we denote by $\Mod_{\it
    tree}^{b}(\K)$ the class of all models $\Imc$ of $\K$ that can be
  constructed by choosing, for each $a\in \ind(\A)$, a tree
  interpretation $\Imc_a$ (see Section~\ref{sect:concinsepalc}) of
  outdegree bounded by $|\T|$ and with root $a$, taking their disjoint
  union, and then adding the pair $(a,b)$ to $r^\Imc$ whenever $r(a,b)
  \in \Amc$.  It is known that $\Mod_{\it tree}^{b}(\K)$ is complete
  for $\K$ and UCQs, and thus for any class of queries considered in
  this paper \cite{GliHoLuSa-JAIR08}. If $\K$ is formulated in
  Horn-$\mathcal{ALC}$ or in $\mathcal{EL}$, then there is even a
  single model $\C_{\Kmc}$ in $\Mod_{\it tree}^{b}(\K)$ such that
  $\{\C_{\Kmc}\}$ is complete for $\K$ and UCQs, the \emph{universal}
  (or \emph{canonical}) \emph{model} of $\K$ \cite{lutz2012non}.

  \smallskip

  If the KB is formulated in an extension of \ALC, the class of models
  needs to be adapted appropriately. The only such extension we are
  going to consider is $\mathcal{ALCHI}$ and its fragment $\DLcH$. In
  this case, one needs a more liberal definition of tree
  interpretation where role edges can point both downwards and upwards
  and multi edges are allowed. We refer to the resulting class of models as
  $\Mod_{\it utree}^b$.
\end{example}  

It is well-known from model theory~\cite{ChangKeisler90} that, for any CQ
$\q(\avec{x})$ and any tuple $\avec{a} \subseteq \ind(\Kmc)$, we have
$\Imc \models \avec{q}(\avec{a})$ for all $\Imc \in \Mod$ iff
$\prod \Mod \models \q(\avec{a})$, where $\prod \Mod$ is the
\emph{direct product} of interpretations in $\Mod$. More precisely, if
$\Mod = \{\Imc_i \mid i\in I\}$, for some set $I$, then
$\prod\Mod = (\Delta^{\prod\Mod}, \cdot^{\prod\Mod})$, where
\begin{itemize}
\item $\Delta^{\prod \Mod} = \prod_{i\in I} \Delta^{\Imc_i}$ is the Cartesian product of the $\Delta^{\Imc_i}$;

\item $a^{\prod \Mod} = (a^{\Imc_i})_{i\in I}$, for any individual name $a$;

\item $A^{\prod \Mod} = \{(d_i)_{i\in I} \mid d_i \in A^{\Imc_i} \text{ for all } i \in I\}$, for any concept name $A$;

\item $r^{\prod \Mod} = \{(d_i,e_i)_{i\in I} \mid (d_i,e_i) \in r^{\Imc_i} \text{ for all } i \in I\}$, for any role name $r$.
\end{itemize}
It is to be noted that in general $\prod\Mod$ is \emph{not} a model of
\Kmc, even if every interpretation in $\Mod$ is.

\begin{example}
Two interpretations $\I_1$ and $\I_2$ are shown below together with their direct product $\I_1 \times \I_2$ (all the arrows are assumed to be labelled with $r$):

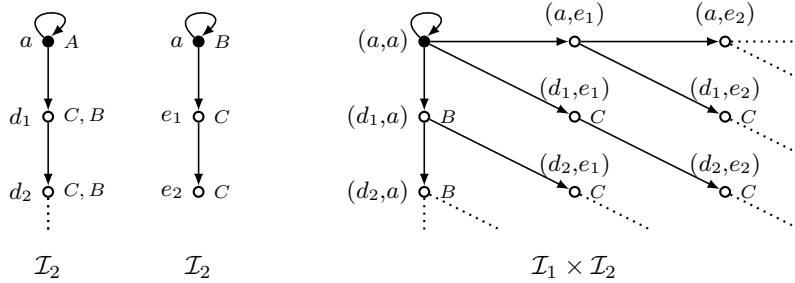
  \begin{tikzpicture}[xscale=2, %
  point/.style={thick,circle,draw=black,fill=white, minimum
    size=1.3mm,inner sep=0pt}%
  ]

  \foreach \al/\x/\y/\lab/\wh/\extra in {%
    a/0/0/A/right/constant, %
    d_1/0/-1/{C,B}/right/, %
    d_2/0/-2/{C,B}/right/%
  }{ \node[point, \extra, label=left:{\vertexfont$\al$}, %
    label=\wh:{\edgefont$\lab$}] (\al) at (\x,\y) {}; }

  \foreach \from/\to in {a/d_1, d_1/d_2}{\draw[role] (\from) -- (\to);}
  \draw[role] (a) to[out=120,in=60, looseness=30] (a);

  \node at (0,-3) {$\mathcal{I}_2$};

  \begin{scope}[xshift=1cm]
    \foreach \al/\x/\y/\lab/\wh/\extra in {%
      a/0/0/B/right/constant, %
      e_1/0/-1/{C}/right/, %
      e_2/0/-2/{C}/right/%
    }{ \node[point, \extra, label=left:{\vertexfont$\al$}, %
      label=\wh:{\edgefont$\lab$}] (\al) at (\x,\y) {}; }

    \foreach \from/\to in {a/e_1, e_1/e_2} {\draw[role] (\from) -- (\to);}
    \draw[role] (a) to[out=120,in=60, looseness=30] (a);

    \node at (0,-3) {$\mathcal{I}_2$};
  \end{scope}

  \begin{scope}[xshift=2.5cm]
    \foreach \al/\ali/\x/\y/\lab/\wh/\extra in {%
      a/a/0/0//left/constant, %
      a/e_1/1/0//above/, %
      a/e_2/2/0//above/,%
      d_1/a/0/-1/B/left/, %
      d_1/e_1/1/-1/{C}/above/, %
      d_1/e_2/2/-1/{C}/above/,%
      d_2/a/0/-2/B/left/, %
      d_2/e_1/1/-2/{C}/above/, %
      d_2/e_2/2/-2/{C}/above/%
    }{ \node[point, \extra, label={[scale=1]\wh:{\vertexfont$(\al$,$\ali)$}}, %
      label=2:{\edgefont$\lab$}] (\al-\ali) at (\x,\y) {}; }

    \foreach \one/\two/\three in {%
      a-a/d_1-a/d_2-a, a-a/a-e_1/a-e_2, a-a/d_1-e_1/d_2-e_2%
    }{ \draw[role] (\one) -- (\two); \draw[role] (\two) -- (\three); }

    \foreach \from/\to in {d_1-a/d_2-e_1, a-e_1/d_1-e_2}{%
      \draw[role] (\from) -- (\to);}

    \draw[role] (a-a) to[out=120,in=60, looseness=30] (a-a);

    \node at (1,-3) {$\mathcal{I}_1 \times \mathcal{I}_2$};
  \end{scope}

  \foreach \from in {d_2,d_2,d_2-a}{%
    \draw[thick, dotted] (\from) -- +(0,-0.5); }

  \draw[thick, dotted] (a-e_2) -- +(0.5,0);

  \foreach \from in {d_2-a,d_2-e_1,d_2-e_2,d_1-e_2,a-e_2}{%
    \draw[thick, dotted] (\from) -- +(0.5,-0.5); }
\end{tikzpicture}

\noindent
Now, consider the CQ
$\q_1(x) = \exists y, z \, (r(x,y) \land r(y,z) \land B(y) \land
C(z))$.
We clearly have $\I_1 \models \q_1(a)$, $\I_2 \models \q_1(a)$, and
$\I_1 \times \I_2 \models \q_1(a)$. On the other hand, for the Boolean
CQ
$\q_2 = \exists x, y, z \ (r(x,y) \land r(y,z) \land C(y) \land
B(z))$,
we have $\I_1 \models \q_2$ but $\I_2 \not\models \q_2$, and so
$\I_1 \times \I_2 \not\models \q_2$.
\end{example}

Another well-known model-theoretic notion that we need for our characterizations
is that of homomorphism. Let $\Imc_1$ and $\Imc_2$ be interpretations, and
$\Sigma$ a signature. A function $h \colon \Delta^{\Imc_2} \to \Delta^{\Imc_1}$
is a \emph{$\Sigma$-homomorphism from $\Imc_2$ to $\Imc_1$} if
\begin{itemize}
\item $h(a^{\Imc_2}) = a^{\Imc_1}$ for all $a \in \NI$ interpreted by $\Imc_2$,

\item $d \in A^{\Imc_2}$ implies $h(d) \in
A^{\Imc_1}$ for all $d \in \Delta^{\smash{\Imc_2}}$ and $\Sigma$-concept names $A$,

\item $(d,e) \in r^{\Imc_2}$ implies $(h(d),h(e)) \in
r^{\Imc_1}$ for all $d,e \in \Delta^{\smash{\Imc_2}}$ and $\Sigma$-role names $r$.
\end{itemize}
It is readily seen that if $\Imc_2 \models \q(\vec{a})$, for a
$\Sigma$-CQ $\q(\vec{x})$, and there is a $\Sigma$-homomorphism from
$\Imc_2$ to $\Imc_1$, then $\Imc_1 \models \q(\vec{a})$. Furthermore,
if we regard $\q(\vec{a})$ as an interpretation whose domain consists
of the elements in $\vec{a}$ (substituted for the answer variables)
and of the quantified variables in $\q(\vec{x})$, and whose
interpretation function is given by its atoms, then
$\Imc_2 \models \q(\vec{a})$ iff there exists a $\Sigma$-homomorphism
from $\q(\vec{a})$ to $\Imc_2$.

To give model-theoretic criteria for CQ entailment and UCQ entailment,
we actually start with partial $\Sigma$-homomorphisms, which we
replace by full homomorphisms in a second step. Let $n$ be a natural
number. We say that $\Imc_2$ is \emph{$n\Sigma$-homo\-mo\-rphically
  embeddable into $\Imc_1$} if, for any subinterpretation $\Imc_2'$ of
$\Imc_2$ with $|\Delta^{\smash{\Imc'_2}}| \le n$, there is a
$\Sigma$-homomorphism from $\Imc_2'$ to $\Imc_1$.\footnote{$\Imc_2'$
  is a \emph{subinterpretation} of $\Imc_2$ if
  $\Delta^{\smash{\Imc'_2}} \subseteq \Delta^{\smash{\Imc_2}}$,
  $A^{\smash{\Imc'_2}} = A^{\smash{\Imc_2}} \cap
  \Delta^{\smash{\Imc'_2}}$
  and
  $r^{\smash{\Imc'_2}} = r^{\smash{\Imc_2}} \cap
  (\Delta^{\smash{\Imc'_2}} \times \Delta^{\smash{\Imc'_2}})$,
  for all concept names $A$ and role names $r$.} If $\Imc_2$ is
$n\Sigma$-homo\-mo\-rphically embeddable into $\Imc_1$ for any
$n > 0$, then we say that $\Imc_2$ is \emph{finitely
  $\Sigma$-homo\-mo\-rphically embeddable into $\Imc_1$}.

\begin{theorem}\label{crit:KB}
  Let $\K_1$ and $\K_2$ be KBs, $\Sigma$ a signature, and
  $\Mod^\Qmc_{\!i}$ a class of interpretations that is complete for
  $\K_i$ and the class of queries \Qmc, for $i=1,2$ and $\Qmc \in
  \{\text{CQ},\text{UCQs}\}$. Then
\begin{description}\itemsep=0pt
\item[\rm (\emph{i})]$\K_{1}$ $\Sigma$-UCQ entails $\K_2$ iff, for any  $n>0$ and $\I_1\in \Mod^{\text{UCQ}}_{\!1}$, there exists $\I_2 \in \Mod_{\!2}^{\text{UCQ}}$ that is $n\Sigma$-homomorphically embeddable into $\I_1$.
	
	
\item[\rm (\emph{ii})]$\K_{1}$ $\Sigma$-CQ entails $\K_2$ iff $\prod \Mod^{\text{CQ}}_{\!2}$ is finitely $\Sigma$-homomorphically embeddable into $\prod \Mod^{\text{CQ}}_{\!1}$. 

%
\end{description}
\end{theorem}

As finite $\Sigma$-homomorphic embeddability is harder to deal with
algorithmically than full $\Sigma$-homomorphic embeddability, it would
be convenient to replace finite $\Sigma$-homomorphic embeddability
with $\Sigma$-homomorphic embeddability in Theorem~\ref{crit:KB}. We
first observe that this is not possible in general:
\begin{example}\label{ex:finf}
  Let $\Kmc_{i}= (\Tmc_{i},\Amc)$, $i=1,2$, be $\DLc$ KBs
  where $\Amc=\{A(c)\}$, and
\begin{align*}
& \Tmc_{1}=\{A\sqsubseteq \exists s.\top, \ \exists s^{-}.\top \sqsubseteq \exists r.\top, \ \exists r^{-}. \top \sqsubseteq \exists r.\top\},\\
& \Tmc_{2}= \{A\sqsubseteq \exists s.\top, \ \exists s^{-}.\top \sqsubseteq \exists r^{-}.\top, \ \exists r. \top \sqsubseteq \exists r^{-}.\top\}.
\end{align*}
Let $\Sigma=\{A,r\}$. Recall that
the class of models $\{\C_{\K_i}\}$ is complete for $\Kmc_i$ and UCQs,
where $\Cmc_{\K_i}$ is the canonical model of $\Kmc_i$:
\\
\centerline{
\begin{tikzpicture}
\begin{scope}
  \foreach \al/\x/\y/\lab/\wh/\extra in {%
    a/0/0/a/left/constant,%
    x1/1/0/{}/right/,%
    x2/2/0/{}/right/,%
    x3/3/0/{}/right/%
  }{ \node[point, \extra, label=\wh:{\vertexfont $\lab$}] (\al) at (\x,\y) {}; }%
  \node[anchor=north] at (a.south) {\edgefont $A$};
  \foreach \from/\to/\lab/\col in {%
    a/x1/s/gray, x1/x2/r/, x2/x3/r/%
  }{ \draw[role,\col] (\from) -- node[above] {\edgefont $\lab$} (\to); }%
  \draw[dashed, edge] (x3) -- ++(0.7,0); %
  \node at (-1,0) {$\mathcal{C}_{\Kmc_1}$};
\end{scope}
\begin{scope}[xshift=6cm]
  \foreach \al/\x/\y/\lab/\wh/\extra in {%
    a/0/0/a/left/constant,%
    x1/1/0/{}/right/,%
    x2/2/0/{}/right/,%
    x3/3/0/{}/right/%
  }{ \node[point, \extra, label=\wh:{\vertexfont $\lab$}] (\al) at (\x,\y) {}; }%
  \node[anchor=north] at (a.south) {\edgefont $A$};
  \foreach \from/\to/\lab/\col in {%
    a/x1/s/gray, x2/x1/r/, x3/x2/r/%
  }{ \draw[role,\col] (\from) -- node[above] {\edgefont $\lab$} (\to); }%
  \draw[dashed, edge] (x3) -- ++(0.7,0); %
  \node at (-1,0) {$\mathcal{C}_{\Kmc_2}$};
\end{scope}
\end{tikzpicture}
}
The KBs $\K_{1}$ and $\K_{2}$ are $\Sigma$-UCQ inseparable, but 
$\C_{\K_2}$ is not $\Sigma$-homo\-mor\-phically embeddable into 
$\C_{\K_1}$.
\end{example}  

The example above uses inverse roles and 
it turns out that these are indeed needed to construct counterexamples
against the version of Theorem~\ref{crit:KB} where finite homomorphic
embeddability is replaced with full embeddability. The following
result showcases this. It concentrates on $\hALC$
and on \ALC, which do not admit inverse roles, and establishes
characterizations of query entailment based on full homomorphic
embeddings.

\begin{theorem}\label{crit:in}~\\[-5mm]
  \begin{description}
  \item[\rm (\emph{i})] Let $\K_{1}$ and $\K_{2}$ be $\hALC$ KBs. Then
    $\K_{1}$ $\Sigma$-CQ entails $\K_2$ iff $\C_{\K_{2}}$ is
    $\Sigma$-homomorphically embeddable into $\C_{\K_{1}}$.

  \item[\rm (\emph{ii})] Let $\K_{1}$ and $\K_{2}$ be $\ALC$ KBs. Then
    $\K_{1}$ $\Sigma$-UCQ entails $\K_{2}$ iff, for every $\I_{1}\in
    \Mod_{\it tree}^{b}(\K_{1})$, there exists $\I_{2}\in \Mod_{\it
      tree}^{b}(\K_{2})$ such that $\Imc_{2}$ is
    $\Sigma$-homomorphically embeddable into $\Imc_{1}$.
\end{description}
\end{theorem}

Claim $(i)$ of Theorem~\ref{crit:in} is proved in \cite{BotoevaKRWZ16} using a
game-theoretic characterization (which we discuss below).  The proof of $(ii)$
is given in \cite{BotoevaLRWZ16-arxiv}. One first proves using an
automata-theoretic argument that one can work without loss of generality with
models in $\Mod_{\it tree}^{b}(\K_{1})$ in which the tree interpretations
$\Imc_{a}$ attached to the ABox individuals $a$ are regular. Second, since nodes
in $\Imc_{a}$ are related to their children using role names only (as opposed to
inverse roles), $\Sigma$-homomorphisms on tree interpretations correspond to
$\Sigma$-simulations (see Sections~\ref{sect:el}
and~\ref{Sec:horngames}). Finally, using this observation one can construct the
required $\Sigma$-homomorphism as the union of finite $\Sigma$-homomorphisms on
finite initial parts of the tree interpretations $\Imc_{a}$.

Note that Theorem~\ref{crit:in} omits the case of \ALC KBs and CQ
entailment, for which we are not aware of a characterization in terms of full homomorphic embeddability.

Another interesting aspect of Example~\ref{ex:finf} is that the
canonical model of $\Kmc_2$ contains elements that are not reachable
along a path of $\Sigma$-roles. In fact, just like inverse roles, this
is a crucial feature for the example to work. We illustrate this by
considering \emph{rooted} UCQs (rUCQs). Recall that in an rUCQ, every
variables has to be connected to an answer variable. For answering a
$\Sigma$-rUCQ, $\Sigma$-disconnected parts of models such as in
Example~\ref{ex:finf} can essentially be ignored since the query
cannot `see' them. As a consequence, we can sometimes replace finite
homomorphic embeddability with full homomorphic embeddability.  We
give an example characterization to illustrate this. Call an
interpretation $\Imc$ \emph{$\Sigma$-connected} if, for every $u \in
\Delta^{\I}$, there is a path $r_1^{\I}(a,u_1),\dots, r_n^{\I}(u_n,u)$
with an individual $a$ and $r_i \in \Sigma$.  An interpretation $\I_2$
is \emph{con-$\Sigma$-homo\-mo\-rphically embeddable into $\I_1$} if
the maximal $\Sigma$-connected subinterpretation $\I_2'$ of $\I_2$ is
$\Sigma$-homomorphically embeddable
into~$\I_1$. 
%
%
%
\begin{theorem}\label{crit:KB2}
  Let $\K_1$ and $\K_2$ be $\mathcal{ALCHI}$ KBs and $\Sigma$ a signature. Then
%
  $\K_{1}$ $\Sigma$-rUCQ entails $\K_2$ iff for any
  $\I_1\in \Mod_{\it utree}^{b}(\K_{1})$, there exists $\I_2 \in
  \Mod_{\it utree}^{b}(\K_{2})$ that is con-$\Sigma$-homomorphically embeddable into
  $\I_1$.
%
%
\end{theorem}

Theorem~\ref{crit:KB2} is proved for $\mathcal{ALC}$ in \cite{BotoevaLRWZ16}. The extension to $\mathcal{ALCHI}$ is straightforward.
The model-theoretic criteria given above are a good starting point for designing decision procedures for query inseparability. But can they
be checked effectively? We first consider this question for \ALC and
then move to Horn DLs.

\subsection{Query inseparability of $\ALC$ KBs} 
We begin with CQ entailment and inseparability in \ALC and show that
both problems are undecidable even for very restricted classes of
KBs. The same is true for rCQs. We then show that, in contrast to the
CQ case, UCQ inseparability in \ALC is decidable in 2\ExpTime.

The following example illustrates the notion of CQ-inseparability of $\ALC$
KBs. 

\begin{example}\label{ex:query-kb}
Suppose $\T_1=\emptyset$, $\T_2= \{ E \sqsubseteq A \sqcup B\}$, $\A$ looks like on the left-hand side of the picture below, and $\Sigma=\{r,A,B\}$. Then we can separate $\K_2= (\T_2,\A)$ from $\K_1 = (\T_1,\A)$ by the $\Sigma$-CQ $\q(x)$ shown on the right-hand side of the picture since  clearly $(\T_1,\A) \not\models \q(a)$, whereas $(\T_2,\A) \models \q(a)$. To see the latter, we first observe that, in any model $\I$ of $\K_2$, we have (\emph{i}) $c \in A^\I$ or (\emph{ii}) $c \in B^\I$. In case (\emph{i}), $\I \models \q(a)$ because of the path $r(a,c),r(c,d)$; and if (\emph{ii}) holds, then $\I \models \q(a)$ because of the path $r(a,b),r(b,c)$ (cf.~\cite[Example~4.2.5]{Scha94b}).
\begin{center}
  \begin{tikzpicture}
    \foreach \name/\x/\y/\conc/\wh in {%
      a/-1.6/0.5//right,%
      b/-0.8/-0.2/A/below, %
      c/0/0.5/E/below, %
      d/1.2/0.5/B/below%
    }{ \node[point, constant, label=\wh:{\edgefont $\conc$},
      label=above:{\vertexfont$\name$}] (\name) at (\x,\y) {}; }

    \foreach \from/\to/\wh in {%
      a/b/below, a/c/above, b/c/below, c/d/above%
    }{ \draw[role] (\from) -- node[\wh,sloped] {\edgefont $r$} (\to);
    }

    \node[anchor=east] at (-2,0.2) {\normalsize $\A$:};

  \begin{scope}[xshift=5cm, yshift=0.2cm]
    \foreach \name/\x/\conc/\wh in {%
      x/-1.2//right,%
      y_1/0/A/below, %
      y_2/1.2/B/below%
    }{ \node[inner sep=1, outer sep=0, label=\wh:{\edgefont $\conc$}]
      (\name) at (\x,0) {\vertexfont $\name$ }; }

    \foreach \from/\to/\wh in {%
      x/y_1/right, y_1/y_2/right%
    }{ \draw[role] (\from) -- node[below] {\edgefont$r$} (\to); }

    \node[anchor=east] at (-1.6,0) {\normalsize $\q(x)$:};
  \end{scope}
\end{tikzpicture}
\end{center}
\end{example}

\begin{theorem}\label{thm:undecidability}
Let $\mathcal{Q} \in \{\text{CQ},\text{\RCQ}\}$.
\begin{description}
\item[\rm (\emph{i})] $\Sigma$-$\mathcal{Q}$ entailment of an \ALC KB by an \EL
KB is undecidable.

\item[\rm (\emph{ii})] $\Sigma$-$\mathcal{Q}$ inseparability of an \ALC and
an \EL KBs is undecidable.
\end{description}
\end{theorem}

The proof of this theorem given in~\cite{BotoevaLRWZ16} uses a reduction of an  undecidable tiling problem.
As usual in encodings of tilings, it is not hard to synchronize tile colours along one dimension. The following example gives a hint of how this can be achieved in the second dimension.

\begin{example}
  Suppose a KB $\K$ has the two models
  $\Imc_i$, $i=1,2$, that are formed by the points on the path between
  $a$ and $e_i$ on the right-hand side of the picture below (this can
  be easily achieved using an inclusion of the form $D \sqsubseteq D_1
  \sqcup D_2$), with $a$ being an ABox point with a loop and the $e_i$
  being the only instances of a concept $C$. Let $\q$ be the CQ on the
  left-hand side of the picture. Then we can have $\K \models \q(a)$
  only if $d_1,d_3 \in A^{\Imc_2}$ and $d_2,d_4 \in B^{\Imc_2}$, with
  the fat black and grey arrows indicating homomorphisms from $\q$ to
  the $\Imc_i$ (the grey one sends $x_0$--$x_2$ to $a$ using the ABox
  loop). This trick can be used to pass the tile colours from one row
  to another.

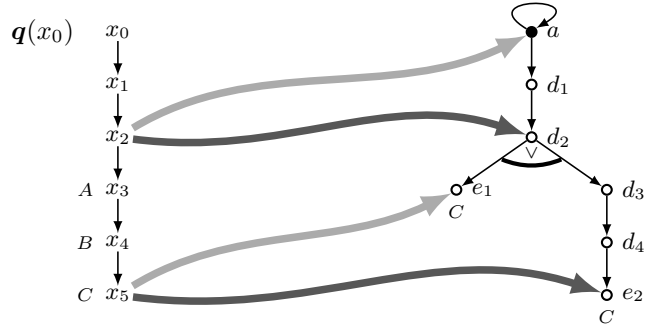
\begin{center}
\begin{tikzpicture}[yscale=0.7]
  \begin{scope}
    \foreach \al/\x/\y/\lab/\extra in {%
      a/0/0/a/constant,%
      x1/0/-1/d_1/,%
      x2/0/-2/d_2/,%
      x3/-1/-3/e_1/,%
      y1/1/-3/d_3/,%
      y2/1/-4/d_4/,%
      y3/1/-5/e_2/%
    }{\node[point, \extra, label=right:{\vertexfont $\lab$}] (\al) at (\x,\y) {};}

    \node[anchor=north] at (x3.south) {\edgefont$C$};
    \node[anchor=north] at (y3.south) {\edgefont$C$};

    \foreach \from/\to/\lab/\wh in {%
      a/x1/r/left, x1/x2/r/left, x2/x3, x2/y1, y1/y2, y2/y3%
    }{ \draw[role] (\from) -- (\to); }

    \draw[role] (a) to[out=140, in=40, looseness=20] (a);

    \draw[draw=none] (x2) -- node[outer sep=0,inner sep=0,pos=0.4] (l1){} (x3); %
    \draw[draw=none] (x2) -- node[outer sep=0,inner sep=0,pos=0.4] (r1){} (y1); %
    
    \draw[ultra thick] (l1) to[bend right] node[above, yshift=-0.05cm] {\scriptsize $\lor$} (r1);%
    
  \end{scope}

  \begin{scope}[xshift=-5.5cm]
    \node at (-1,0) {$\q(x_0)$};

    \foreach \al/\x/\y/\lab/\extra in {%
      x_0/0/0//,%
      x_1/0/-1/,%
      x_2/0/-2//,%
      x_3/0/-3/A/,%
      x_4/0/-4/B/,%
      x_5/0/-5/C/%
    }{\node[inner sep=1, outer sep=0, label=left:{\edgefont$\lab$}] (\al) at
      (\x,\y) {\vertexfont$\al$};}

    \foreach \from/\to in {%
      x_0/x_1, x_1/x_2, x_2/x_3, x_3/x_4, x_4/x_5%
    }{ \draw[role] (\from) -- (\to); }

  \end{scope}

  \foreach \from/\to in {%
    x_5/x3, x_2/a%
  }{ \draw[line width=0.1cm, -latex, black!33] (\from) to[in=215, out=40] (\to); }%
  \foreach \from/\to in {%
    x_5/y3, x_2/x2%
  }{ \draw[line width=0.1cm, -latex, black!66] (\from) to[in=155, out=-10] (\to); }

\end{tikzpicture}
\end{center}
\end{example}

As we saw in Example~\ref{UCQ-CQ}, UCQs distinguish between more KBs
than CQs, that is, UCQ inseparability is a different and in fact more
fine-grained notion than CQ inseparability. This has the remarkable
effect that decidability is regained \cite{BotoevaLRWZ16-arxiv}.
\begin{theorem}\label{thm:decidability}
  In \ALC, $\Sigma$-\Qmc entailment and $\Sigma$-\Qmc inseparability of KBs are
  2{\sc Exp\-Time}-complete, for $\mathcal{Q} \in
  \{\text{UCQ},\text{\RUCQ}\}$.
\end{theorem}

The proof of the upper bound in Theorem~\ref{thm:decidability} uses
tree automata and relies on the characterization 
of Theorem~\ref{crit:in}~(ii) \cite{BotoevaLRWZ16-arxiv}. In principle, the
automata construction is similar to the one given in the proof sketch
of Theorem~\ref{thm:ceupper}. The main differences between the two
constructions is that we have to replace bisimulations with
homomorphisms. Since homomorphisms preserve only positive and
existential information, we can actually drop some parts of the
automaton construction. On the other hand, homomorphisms require us to
consider also parts of the model $\Imc_2$ (see Theorem~\ref{crit:in})
that are not reachable along $\Sigma$-roles from an ABox individual,
which requires a bit of extra effort. A more technical issue is that
the presence of an ABox seems to not go together so well with
amorphous automata and thus one resorts to more traditional tree
automata along with a suitable encoding of the ABox and of the model
$\Imc_1$ as a labeled tree. The lower bound is proved by an ATM
reduction.


\subsection{Query inseparability of KBs in Horn
  DLs}\label{Sec:horngames}

We first consider DLs without inverse roles and then DLs that admit
inverse roles. 
In both cases, we sketch decision procedures that are based on games
played on canonical models $C_\Kmc$ as mentioned in
Example~\ref{ex:complete}. 
It is well known from logic programming and
databases~\cite{Abitebouletal95} that such models can be constructed
by the \emph{chase procedure}. We illustrate the chase (in a somewhat
different but equivalent form) by the following example.

\begin{example}\label{ex:gens}
Consider the $\DLcH$ KB $\K_2 = (\Tmc_2, \Amc_2)$ with $\Amc_2 = \{ A(a) \}$ and
\begin{multline*}
\Tmc_2 = \{ A \sqsubseteq B, \ A \sqsubseteq \exists p.\top, \ \exists p^-.\top \sqsubseteq \exists r^-.\top, \ \exists r.\top \sqsubseteq \exists q^-.\top, \  \exists q.\top \sqsubseteq \exists q^-.\top,\\
\exists r.\top \sqsubseteq \exists s^-.\top, \ \exists s.\top \sqsubseteq \exists t^-.\top, \ \exists t.\top \sqsubseteq \exists s^-.\top, \ t^- \sqcap s \sqsubseteq \bot \}.
\end{multline*}
We first construct a `closure' of the ABox $\Amc_2$ under the inclusions in $\Tmc_2$. For instance, to satisfy $A \sqsubseteq \exists p.\top$, we introduce a \emph{witness $w_p$ for $p$} and draw an arrow $\leadsto$ from $a$ to $w_p$ indicating that $p(a,w_p)$ holds. The inclusion $\exists p^-.\top \sqsubseteq \exists r^-.\top$ requires a witness $w_{r^-}$ for $r^-$ and the arrow $w_p \leadsto w_{r^-}$.  Having reached the witness $w_{q^-}$ for $q^-$ and applying $\exists q.\top \sqsubseteq \exists q^-.\top$ to it, we `reuse' $w_{q^-}$ and simply draw a loop $w_{q^-} \leadsto w_{q^-}$. The resulting finite interpretation  $\mathcal{G}_2$ shown below is called the \emph{generating structure for $\K_2$}:\\
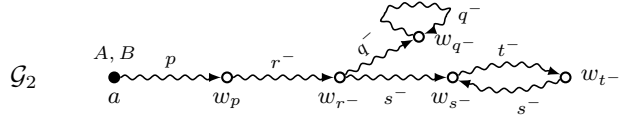
\centerline{\begin{tikzpicture}[>=latex,xscale=1.5,yscale=0.9,
point/.style={thick,circle,draw=black,minimum size=1.3mm,inner sep=0pt},
wiggly/.style={semithick,decorate,decoration={snake,amplitude=0.3mm,segment length=2mm,post length=1mm}}
]
\node[point,fill=black,label=below:{\vertexfont $a$}, label=above:{\edgefont $A, B$}] (a) at (0,0) {};
\node[point, label=below:{\vertexfont $w_{p}$}] (w1) at (1,0) {};
\draw[->,wiggly] (a) to node[above,midway] {\edgefont $p$} (w1);
\node[point, label=below:{\vertexfont $w_{r^-}$}] (w2) at (2,0) {};
\draw[->,wiggly] (w1) to node[above,midway] {\edgefont $r^-$} (w2);
\node[point, label=below:{\vertexfont $w_{s^-}$}] (w3) at (3,0) {};
\draw[->,wiggly] (w2) to node[below,midway] {\edgefont $s^-$} (w3);
\node[point, label=right:{\vertexfont $w_{t^-}$}] (w4) at (4,0) {};
\draw[->,wiggly,out=30,in=150] (w3) to node[above,midway] {\edgefont $t^-$} (w4);
\draw[->,wiggly,out=-150,in=-30] (w4) to node[below,pos=0.3] {\edgefont $s^-$} (w3);
\node[point, label={[yshift=-0.1cm]right:{\vertexfont $w_{q^-}$}}] (w3p) at (2.7,0.6) {};
\draw[->,wiggly] (w2) to node[above,midway,sloped] {\edgefont $q^-$} (w3p);
\draw[->,wiggly,out=150,in=30,looseness=30] (w3p) to node[right,pos=0.75,xshift=-0.05cm] {\edgefont $q^-$} (w3p);
\node at (-0.8,0) {\normalsize$\mathcal{G}_2$};
\end{tikzpicture}}\\
Note that $\mathcal{G}_2$ is \emph{not} a model of $\K_2$ because $(w_{s^-},w_{t^-}) \in (t^-)^{\mathcal{G}_2} \cap s^{\mathcal{G}_2}$. We can obtain a model of $\K_2$ by \emph{unravelling} the witness part of the generating structure $\mathcal{G}_2$ into an infinite tree (in general, forest). The resulting interpretation $\I_2$ shown below is a canonical model of $\K_2$.

\begin{tikzpicture}[>=latex,xscale=1.4,yscale=0.9,
point/.style={thick,circle,draw=black,minimum size=1.3mm,inner sep=0pt},
wiggly/.style={semithick,decorate,decoration={snake,amplitude=0.3mm,segment length=2mm,post length=1mm}}
]
%
%
\node[point,fill=black,label=above:{\edgefont $A,B$},label=below:{\vertexfont $a$}] (a) at (0,0) {};
\node[point,fill=white] (w1) at (1,0) {};
\draw[->, semithick] (a) to node[below,midway] {\edgefont $p$} (w1);
\node[point,fill=white] (w2) at (2,0) {};
\draw[<-, semithick] (w1) to node[below,midway] {\edgefont $r$} (w2);
\node[point,fill=white] (w3) at (3,0) {};
\draw[<-, semithick] (w2) to node[below,midway] {\edgefont $s$} (w3);
\node[point,fill=white] (w4) at (4,0) {};
\draw[<-, semithick] (w3) to node[below,pos=0.6] {\edgefont $t$} (w4);
\node[point,fill=white] (w5) at (5,0) {};
\draw[<-, semithick] (w4) to node[below,pos=0.6] {\edgefont $s$} (w5);
\node[point,fill=white] (w6) at (6,0) {};
\draw[<-, semithick] (w5) to node[below,pos=0.6] {\edgefont $t$} (w6);
\draw[dashed, semithick] (w6) -- +(0.5,0);
\node[point,fill=white] (w3p) at (2.7,0.5) {};
\draw[<-, semithick] (w2) to node[above,midway,sloped] {\edgefont $q$} (w3p);
\node[point,fill=white] (w4p) at (3.7,0.5) {};
\draw[<-, semithick] (w3p) to node[above,midway,sloped] {\edgefont $q$} (w4p);
\node[point,fill=white] (w5p) at (4.7,0.5) {};
\draw[<-, semithick] (w4p) to node[above,midway,sloped] {\edgefont $q$} (w5p);
\node[point,fill=white] (w6p) at (5.7,0.5) {};
\draw[<-, semithick] (w5p) to node[above,midway,sloped] {\edgefont $q$} (w6p);
\draw[dashed, semithick] (w6p) -- +(0.5,0);
%
%
\node at (-0.6,0) {\normalsize$\I_2$};
\end{tikzpicture}
\end{example}

The generating structure underlying the canonical model $\mathcal{C}_\Kmc$ of a Horn KB $\Kmc$ defined above will be denoted by $\mathcal{G}_\Kmc$. By Theorem~\ref{crit:in}, if $\K_1$ and $\K_2$ are KBs formulated in a Horn DL, then $\K_{1}$ $\Sigma$-CQ entails $\K_2$ iff $\mathcal{C}_{\K_2}$ is $n\Sigma$-homomorphically embeddable into $\mathcal{C}_{\K_2}$ for any $n>0$.

In what follows, we require the following upper bounds on the size of
generating structures for Horn KBs~\cite{BotoevaKRWZ16}:

\begin{theorem}\label{thm:mater}~\\[-5mm]
  \begin{description}

\item[\rm (\emph{i})]  
The generating structure for any consistent $\hALCHI$ KB $ (\mathcal{T},\mathcal{A})$ can be constructed in time $|\mathcal{A}|\cdot 2^{p(|\mathcal{T}|)}$, where $p$ is some fixed polynomial\textup{;}

\item[\rm (\emph{ii})] The generating structure for any consistent KB $ (\mathcal{T},\mathcal{A})$ formulated in a DL from the $\EL$ or $\textsl{DL-Lite}$ family  can be constructed in time $|\mathcal{A}|\cdot p(|\mathcal{T}|)$, where $p$ is some fixed polynomial.
  \end{description}
\end{theorem}

We now show that checking whether a canonical model is
$n\Sigma$-homomor\-phically embeddable into another canonical model
can be established by playing games on their underlying generating
structures. For more details, the reader is referred
to~\cite{BotoevaKRWZ16}.


Suppose $\K_1$ and $\K_2$ are (consistent) Horn KBs, $\mathcal{C}_1$ and $\mathcal{C}_2$ are their canonical models, and $\Sigma$ a signature. First, we reformulate the definition of $n\Sigma$-homomorphic embedding in game-theoretic terms. The states of our game are of the form $(\pi \mapsto \sigma)$, where $\pi\in \Delta^{\mathcal{C}_2}$ and $\sigma\in\Delta^{\mathcal{C}_1}$. Intuitively, $(\pi\mapsto \sigma)$ means that `$\pi$ is to be $\Sigma$-homomorphically mapped to $\sigma$'. The game is played by player 1 and player 2 starting from some initial state $(\pi_0\mapsto \sigma_0)$.
The aim of player~1 is to demonstrate that there exists a $\Sigma$-homomorphism from (a finite subinterpretation of) $\mathcal{C}_2$ into $\mathcal{C}_1$ with $\pi_0$ mapped to $\sigma_0$, while player~2 wants to show that there is no such homomorphism. In each round $i >0$ of the game, player~2 challenges player~1 with some $\pi_i\in\Delta^{\mathcal{C}_2}$ that is related to $\pi_{i-1}$ by some $\Sigma$-role.  Player~1, in turn, has to respond with some $\sigma_{i}\in\Delta^{\mathcal{C}_1}$ such that the already constructed partial $\Sigma$-homomorphism can be extended with $(\pi_i \mapsto \sigma_i)$, in particular:
\begin{itemize}
\item[--] $\pi_i \in A^{\mathcal{C}_2}$ implies $\sigma_i \in A^{\mathcal{C}_1}$, for any $\Sigma$-concept name $A$, and
\item[--] $(\pi_{i-1},\pi_i) \in r^{\mathcal{C}_2}$ implies $ (\sigma_{i-1},\sigma_i)\in r^{\mathcal{C}_1}$, for any $\Sigma$-role $r$.
\end{itemize}
A play of length $n$ starting from a state $\mathfrak{s}_{0}$ is any sequence $\mathfrak{s}_{0},\ldots,\mathfrak{s}_{n}$ of states obtained as described above. For any ordinal $\lambda\leq \omega$, we say that player~1 has a \emph{$\lambda$-winning strategy} in the game starting from $\mathfrak{s}_{0}$ if, for any play $\mathfrak{s}_{0},\ldots,\mathfrak{s}_n$ with $n<\lambda$ that is played according to this strategy, player~1 has a response to any challenge of player~2 in the final state $\mathfrak{s}_n$.

It is easy to see that if, for any $\pi_0\in\Delta^{\mathcal{C}_2}$, there is $\sigma_0\in\Delta^{\mathcal{C}_1}$ such that player~1 has an $\omega$-winning strategy in this game starting from $(\pi_0\to\sigma_0)$, then there is a $\Sigma$-homomorphism from $\mathcal{C}_2$ into $\mathcal{C}_1$, and the other way round.
That $\mathcal{C}_2$ is \emph{finitely} $\Sigma$-homomor\-phically embeddable into $\mathcal{C}_1$ is equivalent to the following condition:
\begin{itemize}
\item[--] for any $\pi_0\in\Delta^{\mathcal{C}_2}$ and any $n < \omega$, there exists $\sigma_0\in\Delta^{\mathcal{C}_1}$ such that  player~1 has an $n$-winning strategy in this game starting from $(\pi_0\to\sigma_0)$.
\end{itemize}
\begin{example}\label{game1}
Suppose $\mathcal{C}_1$ and $\mathcal{C}_2$ look like in the picture below. An $\omega$-winning strategy for player 1 starting from $(a \mapsto a)$ is shown by the dotted lines with the rounds of the game indicated by the numbers on the lines.\\
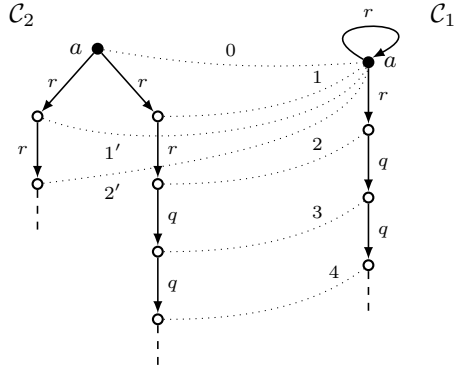
\centerline{\begin{tikzpicture}[>=latex,yscale=0.9]
\begin{scope}[xscale=2]
    \foreach \al/\x/\y/\lab/\wh/\extra in {%
      a/0/0/a/left/constant,%
      w1/0.4/-1//left/,%
      w2/0.4/-2//left/,%
      w3/0.4/-3//left/,%
      w4/0.4/-4//left/,%
      w1p/-0.4/-1//left/,%
      w2p/-0.4/-2//left/%
    }{ \node[point, \extra, label=\wh:{\small $\lab$}] (\al) at (\x,\y) {}; }
    \foreach \from/\to/\lab/\wh in {%
      a/w1/r/right, %
      w1/w2/{r}/right, %
      w2/w3/{q}/right, %
      w3/w4/{q}/right,%
      a/w1p/{r}/left,%
      w1p/w2p/{r}/left%
    }{ \draw[role] (\from) -- node[\wh] {\edgefont $\lab$} (\to); }%
    \draw[dashed, edge] (w2p) -- +(0,-0.7) (w4) -- +(0,-0.7);
    \node at (-0.5,0.5) {$\mathcal{C}_2$};
    %
    \begin{scope}[shift={(1.8,-0.2)}]
      \foreach \al/\x/\y/\lab/\wh/\extra in {%
        a2/0/0/a/right/constant,%
        x1/0/-1/{}/right/,%
        x2/0/-2/{}/right/,%
        x3/0/-3/{}/right/%
      }{ \node[point, \extra, label=\wh:{\vertexfont $\lab$}] (\al) at (\x,\y) {}; }
      \node (x4) at ($(x3)+(0,-0.8)$) {};
      \foreach \from/\to/\lab/\wh/\arr in {%
        a2/a2/{r}/above/{to[out=130,in=50,looseness=40]}, %
        a2/x1/{r}/right/--, %
        x1/x2/{q}/right/--, %
        x2/x3/{q}/right/--%
      }{ \draw[role] (\from) \arr node[\wh] {\edgefont $\lab$} (\to); }%
      \draw[dashed, edge] (x3) -- ++(0,-0.7); %
      \node at (0.5,0.7) {$\mathcal{C}_1$};
    \end{scope}
    %
    \foreach \from/\to/\lab/\arr in {%
      a/a2/0/{to[bend right=15] node[above]},%
      w1p/a2/{1'}/{to[bend right=60] node[below, pos=0.2]},%
      w2p/a2/{2'}/{to[out=15, in=265] node[below, pos=0.15]},%
      w1/a2/1/{to[bend right] node[above, pos=0.7]},%
      w2/x1/2/{to[bend right] node[above, pos=0.7]},%
      w3/x2/3/{to[bend right] node[above, pos=0.7]},%
      w4/x3/4/{to[bend right] node[above, pos=0.8]}
    }{ \draw[strategy] (\from) \arr {\homofont $\lab$} (\to); }%
  \end{scope}
\end{tikzpicture}
}
\end{example}

Note, however, that the game-theoretic criterion formulated above does not immediately yield any algorithm to decide finite homomorphic embeddability because both $\mathcal{C}_2$ and $\mathcal{C}_1$ can be infinite. It is readily seen that the canonical model $\mathcal{C}_2$ in the game can be replaced by the underlying generating structure $\mathcal{G}_2$, in which player~2 can only make challenges indicated by  the generating relation $\leadsto$. The picture below illustrates the game played on $\mathcal{G}_2$ and $\mathcal{C}_1$ from Example~\ref{game1}.\\
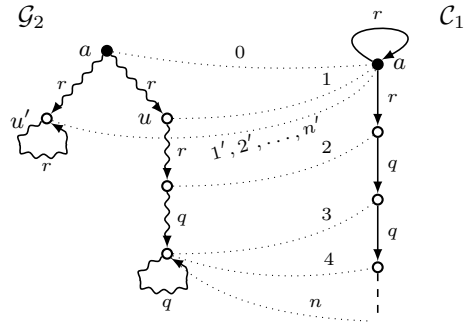
\centerline{\begin{tikzpicture}[>=latex,yscale=0.9]
\begin{scope}[xscale=2]
    \foreach \al/\x/\y/\lab/\wh/\extra in {%
      a/0/0/a/left/constant,%
      w1/0.4/-1/u/left/,%
      w2/0.4/-2//left/,%
      w3/0.4/-3//left/,%
      w1p/-0.4/-1/u'/left/%
    }{ \node[point, \extra, label=\wh:{\small $\lab$}] (\al) at (\x,\y) {}; }
    \foreach \from/\to/\lab/\wh/\arr in {%
      a/w1/r/right/--, %
      w1/w2/{r}/right/--, %
      w2/w3/{q}/right/--, %
      w3/w3/{q}/below/{to[out=-130,in=-50,looseness=40]},%
      a/w1p/{r}/left/{--},%
      w1p/w1p/{r}/below/{to[out=-130,in=-50,looseness=40]}%
    }{ \draw[wiggly] (\from) \arr node[\wh] {\edgefont $\lab$} (\to); }
    \node at (-0.5,0.5) {$\Gmc_2$};
    %
    \begin{scope}[shift={(1.8,-0.2)}]
      \foreach \al/\x/\y/\lab/\wh/\extra in {%
        a2/0/0/a/right/constant,%
        x1/0/-1/{}/right/,%
        x2/0/-2/{}/right/,%
        x3/0/-3/{}/right/%
      }{ \node[point, \extra, label=\wh:{\vertexfont $\lab$}] (\al) at (\x,\y) {}; }
      \node (x4) at ($(x3)+(0,-0.8)$) {};
      \foreach \from/\to/\lab/\wh/\arr in {%
        a2/a2/{r}/above/{to[out=130,in=50,looseness=40]}, %
        a2/x1/{r}/right/--, %
        x1/x2/{q}/right/--, %
        x2/x3/{q}/right/--%
      }{ \draw[role] (\from) \arr node[\wh] {\edgefont $\lab$} (\to); }%
      \draw[dashed, edge] (x3) -- ++(0,-0.7); %
      \node at (0.5,0.7) {$\mathcal{C}_1$};
    \end{scope}
    %
    \foreach \from/\to/\lab/\arr in {%
      a/a2/0/{to[bend right=15] node[above]},%
      w1p/a2/{1',2',\dots,n'}/{to[bend right=60] node[below,pos=0.55,yshift=0.1cm,rotate=15]},%
      w1/a2/1/{to[bend right] node[above, pos=0.7]},%
      w2/x1/2/{to[bend right] node[above, pos=0.7]},%
      w3/x2/3/{to[bend right] node[above, pos=0.7]},%
      w3/x3/4/{to[bend right] node[above, pos=0.8]},%
      w3/x4/n/{to[bend right] node[above, pos=0.8]}%
    }{ \draw[strategy] (\from) \arr {\homofont $\lab$} (\to); }%
  \end{scope}
\end{tikzpicture}
}\\
If the KBs are formulated in a Horn DL that does not allow \emph{inverse roles}, then $\mathcal{C}_1$ can also be replaced with its generating structure $\mathcal{G}_1$ as illustrated by the picture below:\\
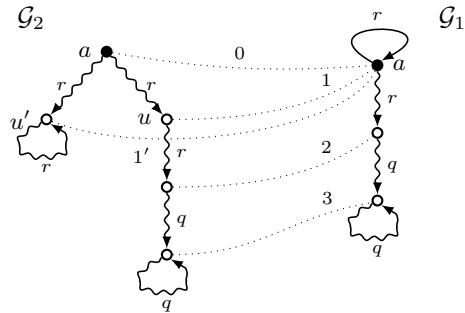
\centerline{\begin{tikzpicture}[>=latex,yscale=0.9]
\begin{scope}[xscale=2]
    \foreach \al/\x/\y/\lab/\wh/\extra in {%
      a/0/0/a/left/constant,%
      w1/0.4/-1/u/left/,%
      w2/0.4/-2//left/,%
      w3/0.4/-3//left/,%
      w1p/-0.4/-1/u'/left/%
    }{ \node[point, \extra, label=\wh:{\small $\lab$}] (\al) at (\x,\y) {}; }
    \foreach \from/\to/\lab/\wh/\arr in {%
      a/w1/r/right/--, %
      w1/w2/{r}/right/--, %
      w2/w3/{q}/right/--, %
      w3/w3/{q}/below/{to[out=-130,in=-50,looseness=40]},%
      a/w1p/{r}/left/{--},%
      w1p/w1p/{r}/below/{to[out=-130,in=-50,looseness=40]}%
    }{ \draw[wiggly] (\from) \arr node[\wh] {\edgefont $\lab$} (\to); }
    \node at (-0.5,0.5) {$\Gmc_2$};
    %
    \begin{scope}[shift={(1.8,-0.2)}]
      \foreach \al/\x/\y/\lab/\wh/\extra in {%
        a2/0/0/a/right/constant,%
        x1/0/-1/{}/right/,%
        x2/0/-2/{}/right/%
      }{ \node[point, \extra, label=\wh:{\vertexfont $\lab$}] (\al) at (\x,\y) {}; }
      \node (x4) at ($(x3)+(0,-0.8)$) {};
      \foreach \from/\to/\lab/\wh/\arr/\type in {%
        a2/a2/{r}/above/{to[out=130,in=50,looseness=40]}/role, %
        a2/x1/{r}/right/--/wiggly, %
        x1/x2/{q}/right/--/wiggly, %
        x2/x2/{q}/below/{to[out=-130,in=-50,looseness=40]}/wiggly%
      }{ \draw[\type] (\from) \arr node[\wh] {\edgefont $\lab$} (\to); }%
      \node at (0.5,0.7) {$\Gmc_1$};
    \end{scope}
    %
    \foreach \from/\to/\lab/\arr in {%
      a/a2/0/{to[bend right=15] node[above]},%
      w1p/a2/{1'}/{to[bend right=60] node[below,pos=0.25]},%
      w1/a2/1/{to[bend right] node[above, pos=0.7]},%
      w2/x1/2/{to[bend right] node[above, pos=0.7]},%
      w3/x2/3/{to[out=0, in=200] node[above, pos=0.8]}%
    }{ \draw[strategy] (\from) \arr {\homofont $\lab$} (\to); }%
  \end{scope}
\end{tikzpicture}
}\\
Reachability or simulation games on finite graphs such as the one
discussed above have been extensively investigated in game
theory~\cite{Maza01,ChHe12}. In particular, it follows that checking
the existence of $n$-winning strategies, for all $n < \omega$, can be
done in polynomial time in the number of states and the number of the
available challenges. Together with Theorem~\ref{thm:mater}, this
gives the upper bounds in the following theorem. Claim~(\emph{i}) was first
observed in \cite{DBLP:journals/jsc/LutzW10}, while~(\emph{ii}) and the
results on data complexity are from  \cite{BotoevaKRWZ16}.

\begin{theorem}\label{thm:EL-ALCH}
$\Sigma$-CQ entailment and $\Sigma$-CQ inseparability of KBs are
\begin{description}
\item[\rm (\emph{i})] in {\sc PTime} for \EL;
\item[\rm (\emph{ii})]\ExpTime-complete for $\hALC$.
\end{description}
Both problems are in {\sc PTime} for data complexity for both \EL and $\hALC$. 
\end{theorem}

Here, by `data complexity' we mean that only the ABoxes of the two
involved KBs are regarded as input, while the TBoxes are fixed.
Analogously to data complexity in query answering, the rationale
behind this setup is that, in data-centric applications, ABoxes tend
to be huge compared to the TBoxes and thus it can result in more
realistic complexities to assume that the latter are actually of
constant size. The lower bound in Theorem~\ref{thm:EL-ALCH} is proved
by reduction of the word problem of polynomially space-bounded ATMs. We
remind the reader at this point that, in all DLs studied in this
section, CQ entailment coincides with UCQ entailment, and likewise for
inseparability.

If inverse roles are available, then replacing canonical models with
their generating structures in games often becomes less
straightforward. We explain the issues using an example in $\DLcH$,
where inverse roles interact in a problematic way with role
inclusions. Similar effects can be observed in Horn-$\mathcal{ALCI}$,
though, despite the fact that no role inclusions are available there.

\begin{example}
  Consider the $\DLcH$ KBs $\K_1 = (\T_1, \{ Q(a,a) \})$ with
\begin{align*}
\T_1 = \{ & \ A \sqsubseteq \exists s.\top, \ \exists s^-.\top \sqsubseteq \exists t.\top, \ \exists t^-.\top \sqsubseteq \exists s.\top, \
 s \sqsubseteq q, \  t \sqsubseteq q, \ \exists q^-.\top \sqsubseteq \exists r.\top\ \}
\end{align*}
and $\K_2$ from Example~\ref{ex:gens}. Let $\Sigma = \{q,r,s,t\}$. The generating structure $\mathcal{G}_2$ for $\K_2$ and the canonical model $\mathcal{C}_1$ for $\K_1$, as well as a 4-winning strategy for player~1 in the game over $\mathcal{G}_2$ and $\mathcal{C}_1$ starting from the state $(u_0,\sigma_4)$ are shown in the picture below:\\
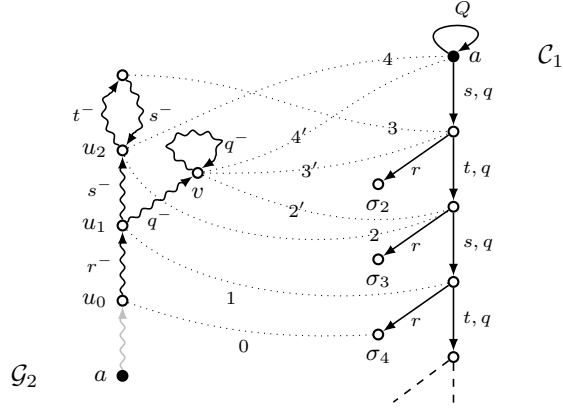
\centerline{
\begin{tikzpicture}[>=latex,yscale=1, xscale=2]
  \begin{scope}
    \begin{scope}
      \foreach \al/\x/\y/\lab/\wh/\extra in {%
        a/0/0/a/left/constant,%
        w1/0/1/u_0/left/,%
        w2/0/2/u_1/left/,%
        w3/0/3/u_2/left/,%
        w4/0/4//above/,
        w3p/0.5/2.7/v/below/%
      }{ \node[point, \extra, label=\wh:{\vertexfont $\lab$}] (\al) at (\x,\y) {}; }
      \foreach \from/\to/\lab/\arr/\extra in {%
        a/w1//{-- node[above]}/gray!50, %
        w1/w2/{r^-}/{-- node[left]}/, %
        w2/w3/{s^-}/{-- node[left]}/,%
        w3/w4/{t^-}/{to[bend left=20] node[left]}/,%
        w4/w3/{s^-}/{to[bend left=20] node[right]}/, %
        w2/w3p/{q^-}/{-- node[below,pos=0.5]}/, %
        w3p/w3p/{q^-}/{to[out=130,in=50,looseness=40] node[right,pos=0.75,xshift=-0.1cm]}/%
      }{ \draw[wiggly,\extra] (\from) \arr {\edgefont $\lab$} (\to); }
      \node at (-0.65,0) {$\mathcal{G}_2$};
      %
      \begin{scope}[shift={(2.2,-0.75)}]
        \foreach \al/\x/\y/\lab/\wh/\extra in {%
          x0/0/5/a/right/constant,%
          x1/0/4/{}/right/,%
          x2/0/3/{}/right/,%
          x3/0/2/{}/right/,%
          x4/0/1//below/,%
          y1/-0.5/3.3/{\sigma_2}/below/,%
          y2/-0.5/2.3/{\sigma_3}/below/,%
          y3/-0.5/1.3/{\sigma_4}/below/%
        }{ \node[point, \extra, label=\wh:{\vertexfont $\lab$}] (\al) at (\x,\y) {}; }
        \draw[role] (x0) to[out=130,in=50,looseness=30]
        node[above,pos=0.6] {\edgefont $Q$} (x0);
        \foreach \from/\to/\lab/\wh in {%
          x0/x1/{s,q}/right, x1/x2/{t,q}/right, %
          x2/x3/{s,q}/right, x3/x4/{t,q}/right, %
          x1/y1/r/below, x2/y2/r/below, %
          x3/y3/r/below%
        }{ \draw[role] (\from) to node[\wh] {\edgefont $\lab$} (\to); }
        \draw[dashed, edge] (x4) -- ++(0,-0.7); %
        \draw[dashed, edge] (x4) -- ++(-0.4,-0.6);
        \node at (0.65,5) {$\C_1$};
      \end{scope}
    \end{scope}
    %
    \foreach \from/\to/\lab/\arr in {%
      w1/y3/0/{to[bend left=-20] node[below]}, %
      w2/x3/1/{to[bend left=-45] node[below,pos=0.4]}, %
      w3/x2/2/{to[out=-80, in=225, looseness=1.2] node[pos=0.8]}, %
      w4/x1/3/{to[out=0, in=180] node[pos=0.85]}, %
      w3/x0/4/{to[out=40, in=180] node[pos=0.85]}, %
      w3p/x2/2'/{to[out=-40, in=210] node[pos=0.4]},%
      w3p/x1/3'/{to[bend left=-20] node[pos=0.4]}, %
      w3p/x0/4'/{to[out=20, in=210] node[pos=0.35,yshift=0.03cm]}%
    }{ \draw[strategy] (\from) \arr {\homofont $\lab$} (\to); }
  \end{scope}
\end{tikzpicture}
}\\
\noindent
(In fact, for any $n >0$, player~1 has an $n$-winning strategy starting from any $(u_0 \mapsto \sigma_m)$ provided that $m$ is even and $m \ge n$.)

This game over $\mathcal{G}_2$ and $\C_1$ has its obvious counterparts over $\mathcal{G}_2$ and $\mathcal{G}_1$; one of them is  shown on the left-hand side of the picture below. It is to be noted, however, that---unlike Example~\ref{game1}---the responses of player~1 are in the reverse direction of the $\leadsto$-arrows (which is possible because of the inverse roles).\\
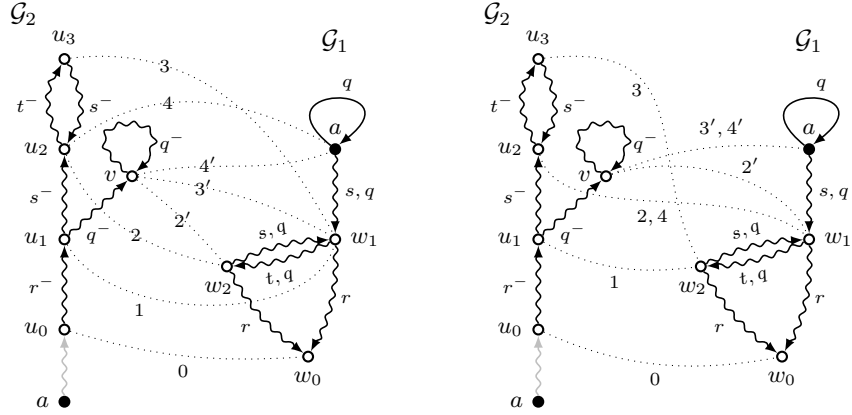
\centerline{
\begin{tikzpicture}[yscale=1.2, xscale=1.8]
  \begin{scope}[xshift=3.5cm]
    \foreach \al/\x/\y/\lab/\wh/\extra in {%
      a/0/0.2/a/left/constant,%
      u1/0/1/u_0/left/,%
      u2/0/2/u_1/left/,%
      u3/0/3/u_2/left/,%
      u4/0/4/u_3/above/,%
      u3p/0.5/2.7/v/left/%
    }{ \node[point, \extra, label=\wh:{\vertexfont $\lab$}] (\al) at (\x,\y) {}; }
    \foreach \from/\to/\lab/\arr/\extra in {%
      a/u1//{-- node[above]}/gray!50, %
      u1/u2/{r^-}/{-- node[left]}/, %
      u2/u3/{s^-}/{-- node[left]}/,%
      u3/u4/{t^-}/{to[bend left=20] node[left]}/,%
      u4/u3/{s^-}/{to[bend left=20] node[right]}/, %
      u2/u3p/{~~q^-}/{-- node[below,pos=0.4]}/, %
      u3p/u3p/{q^-}/{to[out=130,in=50,looseness=40] node[right,pos=0.8,xshift=-0.1cm]}/%
    }{ \draw[wiggly,\extra] (\from) \arr {\edgefont $\lab$} (\to); }
    \node at (-0.3,4.5) {$\mathcal{G}_2$};
    %
    \begin{scope}[shift={(2,1)}]
      \foreach \al/\x/\y/\lab/\wh/\extra in {%
        x0/0/2/a/above/constant,%
        x1/0/1/{w_1}/right/,%
        x2/-0.8/0.7/{w_2~~}/below/,%
        y/-0.2/-0.3/w_0/below/%
      }{ \node[point, \extra, label=\wh:{\vertexfont $\lab$}] (\al) at (\x,\y) {}; }
      \foreach \from/\to/\lab/\arr in {%
        x0/x1/{s,q}/{-- node[right]}, %
        x1/x2/{t,q}/{to[bend left=15] node[below,sloped,pos=0.6]}, %
        x2/x1/{s,q}/{to[bend left=15] node[above,sloped]}, %
        x1/y/r/{to[bend left=10] node[right]}, %
        x2/y/r/{to[bend left=-10] node[below left]}%
      }{ \draw[wiggly] (\from) \arr {\edgefont $\lab$} (\to); }%
      \draw[role] (x0) to[out=130,in=50,looseness=40]
      node[above,pos=0.6] {\edgefont $q$} (x0);%
      \node at (0,3.2) {$\mathcal{G}_1$};
    \end{scope}
    %
    \foreach \from/\to/\lab/\arr in {%
      u1/y/0/{to[bend left=-25] node[below]}, %
      u2/x2/1/{to[bend left=-25] node[below]}, %
      u3/x1/{2,4}/{to[out=-80,in=130] node[below]}, %
      u4/x2/3/{to[in=110,out=10, looseness=1.2] node[below,pos=0.3]}, %
      u3p/x1/2'/{to[bend left=50] node[above,pos=0.6]},%
      u3p/x0/{3',4'}/{to[bend left=20] node[above,pos=0.6]}%
    }{ \draw[strategy] (\from) \arr {\homofont $\lab$} (\to); }
  \end{scope}
  \begin{scope}
    \foreach \al/\x/\y/\lab/\wh/\extra in {%
      a/0/0.2/a/left/constant,%
      u1/0/1/u_0/left/,%
      u2/0/2/u_1/left/,%
      u3/0/3/u_2/left/,%
      u4/0/4/u_3/above/,%
      u3p/0.5/2.7/v/left/%
    }{ \node[point, \extra, label=\wh:{\vertexfont $\lab$}] (\al) at (\x,\y) {}; }
    \foreach \from/\to/\lab/\arr/\extra in {%
      a/u1//{-- node[above]}/gray!50, %
      u1/u2/{r^-}/{-- node[left]}/, %
      u2/u3/{s^-}/{-- node[left]}/,%
      u3/u4/{t^-}/{to[bend left=20] node[left]}/,%
      u4/u3/{s^-}/{to[bend left=20] node[right]}/, %
      u2/u3p/{~~q^-}/{-- node[below,pos=0.4]}/, %
      u3p/u3p/{q^-}/{to[out=130,in=50,looseness=40] node[right,pos=0.8,xshift=-0.1cm]}/%
    }{ \draw[wiggly,\extra] (\from) \arr {\edgefont $\lab$} (\to); }
    \node at (-0.3,4.5) {$\mathcal{G}_2$};
    %
    \begin{scope}[shift={(2,1)}]
      \foreach \al/\x/\y/\lab/\wh/\extra in {%
        x0/0/2/a/above/constant,%
        x1/0/1/{w_1}/right/,%
        x2/-0.8/0.7/{w_2~~}/below/,%
        y/-0.2/-0.3/w_0/below/%
      }{ \node[point, \extra, label=\wh:{\vertexfont $\lab$}] (\al) at (\x,\y) {}; }
      \foreach \from/\to/\lab/\arr in {%
        x0/x1/{s,q}/{-- node[right]}, %
        x1/x2/{t,q}/{to[bend left=15] node[below,sloped,pos=0.6]}, %
        x2/x1/{s,q}/{to[bend left=15] node[above,sloped]}, %
        x1/y/r/{to[bend left=10] node[right]}, %
        x2/y/r/{to[bend left=-10] node[below left]}%
      }{ \draw[wiggly] (\from) \arr {\edgefont $\lab$} (\to); }%
      \draw[role] (x0) to[out=130,in=50,looseness=40]
      node[above,pos=0.6] {\edgefont $q$} (x0);%
      \node at (0,3.2) {$\mathcal{G}_1$};
    \end{scope}
    %
    \foreach \from/\to/\lab/\arr in {%
      u1/y/0/{to[bend left=-15] node[below]}, %
      u2/x1/1/{to[in=-100, out=-70, looseness=1.2] node[below,pos=0.3]}, %
      u3/x2/2/{to[bend left=-30] node[pos=0.6]}, %
      u4/x1/3/{to[in=110,out=10] node[pos=0.25]}, %
      u3/x0/4/{to[out=50, in=140, looseness=1.2] node[pos=0.4]}, %
      u3p/x2/2'/{to[bend left=5] node},%
      u3p/x1/3'/{to[out=-10,in=140] node[pos=0.3]}, %
      u3p/x0/4'/{to[out=20, in=210] node[pos=0.35,yshift=0.03cm]}%
   }{ \draw[strategy] (\from) \arr {\homofont $\lab$} (\to); }
  \end{scope}
\end{tikzpicture}}\\
On the other hand, such reverse responses may create paths in $\mathcal{G}_1$ that do not have any real counterparts in $\mathcal{C}_1$, and so do not give rise to $\Sigma$-homomorphisms we need. An example is shown on the right-hand side of the picture above, where $u_3$ is mapped to $w_2$ and $v$ to $a$ in round 3, which is impossible to reproduce in the \emph{tree-shaped} $\mathcal{C}_1$.


One way to ensure that, in the `backwards game' over $\mathcal{G}_2$ and $\mathcal{G}_1$, all the challenges made by player~1 in any given state are responded by the same element of $\mathcal{G}_1$, is to use states of the form $(\Xi \mapsto w)$, where $\Xi$ is the \emph{set} of elements of $\mathcal{G}_2$ to be mapped to an element  $w$ of $\mathcal{G}_1$. In our example above, we can use the state  $(\{u_2, v\} \mapsto w_2)$, where the only challenge of player 2 is the set of $\leadsto$- successors of $u_2$ and $v$ marked by $\Sigma$-roles, that is, $\Xi' = \{u_3, v\}$, to which player 1 responds with $(\Xi' \mapsto w_1)$.
\end{example}

By allowing more complex states, we increase their number and, as a consequence, the complexity of deciding finite $\Sigma$-homomorphic embeddability. The proof of the following theorem can be found in \cite{BotoevaKRWZ16}:

\begin{theorem}\label{thm:in2exptime}
$\Sigma$-CQ entailment and inseparability of KBs are 
\begin{description}
\item[\rm (\emph{i})]\ExpTime-complete for  $\DLhH$ and $\DLcH$;
\item[\rm (\emph{ii})]2\ExpTime-complete for $\hALCI$ and $\hALCHI$.
\end{description}
For all of these DLs, both problems are in {\sc PTime} for data complexity. 
\end{theorem}

The lower bounds are once again proved using alternating Turing
machines. We remark that CQ entailment and inseparability are in
{\sc PTime} in $\DLc$ and $\DLh$. In {\sl DL-Lite}, it is thus the
combination of inverse roles and role hierarchies that causes hardness.



\section{Query Inseparability of TBoxes}
\label{sec:query-insep-tboxes}

Query inseparability of KBs, as studied in the previous section,
presupposes that the data (in the form of an ABox) is known to the
user, as is the case for example in KB exchange \cite{DBLP:aij-kb-exchange}. In
many query answering applications, though, the data is either not
known during the TBox design or it changes so frequently that query
inseparability w.r.t.\ one fixed data set is not a sufficiently robust
notion. In such cases, one wants to decide query inseparability of
\emph{TBoxes} $\Tmc_1$ and $\Tmc_2$, defined by requiring that, for
\emph{all} ABoxes \Amc, the KBs $(\Tmc_1,\Amc)$ and $(\Tmc_2,\Amc)$
are query inseparable. To increase flexibility, we can also specify a signature of the ABoxes that we are considering,
and we do not require that it coincides with the signature of the
queries. In a sense, this change in the setup brings us closer to
the material from Sections~\ref{sect:separability}
and~\ref{sect:model-inseparability}, where also inseparability of
TBoxes is studied. As in the KB case, the main classes of queries that
we consider are CQs and UCQs as well as rCQs and rUCQs. To relate
concept inseparability and query inseparability of TBoxes, we
additionally consider a class of tree-shaped CQs. 

The structure of this section is as follows. We start by discussing
the impact that the choice of query language has on  query
inseparability of TBoxes.
%
%
We then relate query inseparability of TBoxes to logical equivalence,
concept inseparability, and model inseparability. For Horn DLs, query
inseparability and concept inseparability are very closely related,
while this is not the case for DLs with disjunction. Finally, we
consider the decidability and complexity of query inseparability of
TBoxes. Undecidability of CQ inseparability of \ALC KBs transfers to
the TBox case, and the same is true of upper complexity bounds in
{\sl DL-Lite}. New techniques are needed to establish decidability and
complexity results for other Horn DLs such as $\mathcal{EL}$ and
Horn-$\mathcal{ALC}$. A main observation underlying our algorithms is that it is sufficient to consider tree-shaped ABoxes when
searching for witnesses of  query separability of TBoxes.
\begin{definition}[query inseparability, entailment and conservative extensions]
\label{def:queryinsep}\em Let $\Tmc_{1}$
  and $\Tmc_{2}$ be TBoxes, $\Theta=(\Sigma_{1},\Sigma_{2})$ a pair of
  signatures, and $\mathcal{Q}$ a class of queries. Then
\begin{itemize}
\item the \emph{$\Theta$-$\mathcal{Q}$ difference} between $\Tmc_{1}$ and $\Tmc_{2}$ is the set
${\sf qDiff}_{\Theta}^{\mathcal{Q}}(\Tmc_{1},\Tmc_{2})$
of all pairs $(\Amc,\q(\vec{a}))$ such that
$\Amc$ is a $\Sigma_{1}$-ABox 
and $\q(\vec{a})\in {\sf qDiff}_{\Sigma_{2}}^{\mathcal{Q}}(\Kmc_{1},\Kmc_{2})$,
where $\Kmc_{i}=(\Tmc_{i},\Amc)$ for $i=1,2$;
\item $\Tmc_1$ \emph{$\Theta$-$\mathcal{Q}$ entails} $\Tmc_2$ if
      ${\sf qDiff}_{\Theta}^{\mathcal{Q}}(\Tmc_{1},\Tmc_{2})= \emptyset$;

\item $\Tmc_1$ and $\Tmc_2$ are \emph{$\Theta$-$\mathcal{Q}$ inseparable}
      if $\Tmc_{1}$ $\Theta$-$\mathcal{Q}$ entails $\Tmc_2$ and vice versa;

\item $\Tmc_{2}$ is a \emph{$\mathcal{Q}$ conservative extension of} $\Tmc_{1}$
      iff $\Tmc_{2} \supseteq \Tmc_{1}$ and $\Tmc_{1}$ and $\Tmc_{2}$ are $\Theta$-$\mathcal{Q}$ inseparable for
      $\Sigma_{1}=\Sigma_{2}= \sig(\Tmc_{1})$.
\end{itemize}
If $(\Amc,\q(\vec{a})) \in {\sf qDiff}_{\Theta}^{\mathcal{Q}}(\Tmc_{1},\Tmc_{2})$,  we say that $(\Amc,\q(\vec{a}))$ $\Theta$-$\mathcal{Q}$
\emph{separates} $\Tmc_{1}$ and $\Tmc_{2}$. 
\end{definition}

Note that Definition~\ref{def:queryinsep} does not require the
separating ABoxes to be satisfiable with $\Tmc_1$ and $\Tmc_2$.  Thus,
a potential source of separability is that there is an ABox with which
one of the TBoxes is satisfiable while the other is not. One could
also define a (more brave) version of query inseparability where only
ABoxes are considered that are satisfiable with $\Tmc_1$ and $\Tmc_2$.
%
We will discuss this further in
Section~\ref{subsect:lasttechnical}.

\smallskip

We now analyze the impact that the choice of query language has on
query inseparability. To this end, we introduce a class of tree-shaped
CQs that is closely related to \EL-concepts. Every
$\mathcal{EL}$-concept $C$ corresponds to a tree-shaped rCQ
$\q_{C}(x)$ such that, for any interpretation $\Imc$ and $d\in
\Delta^{\Imc}$, we have $d\in C^{\Imc}$ iff $\Imc\models \q_{C}(d)$.
We denote $\q_{C}(x)$ by $C(x)$ and the Boolean CQ $\exists x \,
\q_{C}(x)$ by $\exists x \, C(x)$. We use $\mathcal{Q}_{\mathcal{EL}}$ to
denote the class of all queries of the former kind and
$\mathcal{Q}_{\mathcal{EL}^u}$ for the class of queries of any
of these two kinds. In the following theorem, the equivalence of~(\emph{iii}) with the other two conditions is of particular
interest. Informally it says that, in \EL and Horn-\ALC, tree-shaped
queries are always sufficient to separate TBoxes.
%

\begin{theorem}\label{equivalencequeries}
Let $\Lmc$ be a Horn DL, $\Tmc_{1}$ and $\Tmc_{2}$ TBoxes formulated
in \Lmc, and $\Theta=(\Sigma_{1},\Sigma_{2})$ a pair of signatures.
Then
the following conditions are equivalent:
\begin{description}
\item[\rm (\emph{i})]$\Tmc_{1}$ $\Theta$-UCQ entails $\Tmc_{2}$;
\item[\rm (\emph{ii})]$\Tmc_{1}$ $\Theta$-CQ entails $\Tmc_2$.
\end{description}
If $\Lmc$ is $\mathcal{EL}$ or Horn-$\mathcal{ALC}$, then these
conditions are also equivalent to
\begin{description}
\item[\rm (\emph{iii})]$\Tmc_{1}$ $\Theta$-$\mathcal{Q}_{\mathcal{EL}^u}$ entails
  $\Tmc_2$.

\end{description}
The same is true when UCQs are replaced with rUCQs, CQs with rCQs, and
$\mathcal{Q}_{\mathcal{EL}^u}$ with $\mathcal{Q}_{\mathcal{EL}}$ \textup{(}simultaneously\textup{)}.
\end{theorem}
\begin{proof}
  The first equivalence follows directly from the fact that KBs in
  Horn DLs are complete w.r.t.~a single model
  (Example~\ref{ex:complete}).  We sketch the proof that
  $\Theta$-$\mathcal{Q}_{\mathcal{EL}^u}$ entailment implies
  $\Theta$-CQ entailment in $\mathcal{EL}$ and Horn-$\mathcal{ALC}$.
  Assume that there is a $\Sigma_{1}$-ABox $\A$ such that $\K_{1}$
  does not $\Sigma_{2}$-CQ entail $\K_{2}$ for $\K_{1}=(\T_{1},\A)$
  and $\K_{2}=(\T_{2},\A)$. Then $\mathcal{C}_{\K_{2}}$ is not
  finitely $\Sigma$-homomorphically embeddable into
  $\mathcal{C}_{\K_{1}}$ (see Theorem~\ref{crit:KB}). We thus find
  a finite subinterpretation $\Imc$ of an interpretation $\Imc_{a}$ in
  $\mathcal{C}_{\K_{2}}$ (see Example~\ref{ex:complete}) that is not
  $\Sigma$-homomorphically embeddable into $\C_{\K_{1}}$. We can
  regard the $\Sigma$-reduct of $\Imc$ as a $\Sigma$-query in
  $\mathcal{Q}_{\mathcal{EL}^u}$ which takes the form $C(x)$ if $\Imc$
  contains $a$ and $\exists x C(x)$ otherwise.  This query witnesses
  that $\Kmc_1$ does not $\Theta$-$\mathcal{Q}_{\mathcal{EL}^u}$ entail
  $\K_{2}$, as required.
\qed
\end{proof}

For Horn DLs other than \EL and Horn-\ALC, the equivalence between
(ii) and~(iii) in Theorem~\ref{equivalencequeries} does often not hold
in exactly the stated form. The reason is that additional constructors
such as inverse roles and role inclusions require us to work with 
slightly different types of canonical models; see
Example~\ref{ex:complete}. However, the equivalence then holds for
appropriate extensions of $\Qmc_{\EL^u}$, for example, by replacing \EL concepts
with $\mathcal{\ELI}$ concepts when moving from Horn-\ALC to
Horn-\ALCI.

Theorem~\ref{equivalencequeries} does not hold for non-Horn DLs. We
have already observed in Example~\ref{UCQ-CQ} that UCQ entailment and
CQ entailment of KBs do not coincide in \ALC. Since the example uses
the same ABox in both KBs, it also applies to the inseparability of
TBoxes. In the following example, we prove that the equivalence
between CQ inseparability and $\Qmc_{\EL^u}$ inseparability fails in
\ALC, too.  The proof can actually be strengthened to show that, in
\ALC, CQ inseparability is a stronger notion than inseparability by
\emph{acyclic} CQs (which generalize $\Qmc_{\EL^u}$ by allowing
multiple answer variables and edges in trees that are directed
upwards).
\begin{example}
%
  Let $\Sigma_{1}=\{r\}$, $\Sigma_{2}=\{r,A\}$, and
  $\Theta=(\Sigma_{1},\Sigma_{2})$.  We construct an \ALC TBox
  $\Tmc_{1}$ as follows:
  \begin{itemize}

  \item to ensure that, for any $\Sigma_{1}$-ABox $\Amc$, the KB $(\Tmc_1,\Amc)$ is satisfiable iff 
    $\Amc$ (viewed as an undirected graph
    with edges $\{ \{a,b\} \mid r(a,b)\in \Amc\}$) is two-colorable, we take the CIs
$$
B \sqsubseteq \forall r. \neg B, \quad \neg B \sqsubseteq \forall r.B;
$$

\item to ensure that, for any $\Sigma_{1}$-ABox $\Amc$, any model in $\Mod_{\it tree}^{b}(\Tmc_1,\A)$
  has an infinite $r$-chain of nodes labeled with the
  concept name $A$ whose root is not reachable from an ABox individual
  along $\Sigma_{2}$-roles we add the CIs
$$
\exists r.\top \sqsubseteq \exists s.B', \quad B' \sqsubseteq A \sqcap
\exists r. B'.
$$

  \end{itemize}
  $\Tmc_2$ is the extension of $\Tmc_1$ with the CI
  $$\exists r.\top \sqsubseteq A \sqcup \forall r.A.
  $$
  Thus, models of $(\Tmc_{2},A)$ extend models of $(\Tmc_{1},\Amc)$
  by labeling certain individuals in $\Amc$ with $A$. The non-$\Amc$
  part is not modified as we can assume that its elements are already labeled with $A$.
  Observe that $\T_1$ and $\T_2$ can be distinguished by the ABox $\A
  = \{ r(a,b), r(b,a)\}$ and the CQ $\q = \exists x,y \, (A(x) \land
  r(x,y) \land r(y,x))$. Indeed, $a\in A^{\Imc}$ or $b\in A^{\Imc}$
  holds in every model $\Imc$ of $(\T_2, \A)$ but this is not the case
  for $(\T_1, \A)$. We now argue that, for every $\Sigma_1$ ABox $\A$
  and every $\Sigma_{2}$ concept $C$ in $\mathcal{EL}$, $(\T_2, \A)
  \models \exists x \, C(x)$ implies $(\T_1, \A) \models \exists x \,
  C(x)$ and $(\T_2, \A) \models C(a)$ implies $(\T_1, \A) \models
  C(a)$. Assume that $\Amc$ and $C$ are given. As any model in $\Mod_{\it tree}^{b}(\Tmc_1,\A)$
	has infinite $r$-chains labeled with $A$, we have 
  $(\Tmc_{1},\Amc)\models \exists x \, C(x)$ for any
  $\Sigma_{2}$-concept $C$ in $\mathcal{EL}$. Thus, we only have to
  consider the case $(\Tmc_{2},\A)\models C(a)$.  If $C$ does not
  contain $A$, then clearly $\Amc\models C(a)$, and so
  $(\Tmc_{1},\A)\models C(a)$, as required.  If $\Amc$ is not
  2-colorable, we also have $(\Tmc_{1},\A)\models C(a)$, as
  required. Otherwise $C$ contains $A$ and $\Amc$ is 2-colorable. But
  then it is easy to see that $(\Tmc_{2},\A)\not\models C(a)$ and
  we have derived a contradiction.
\end{example}

\subsection{Relation to other notions of inseparability}

We now consider the relationship between query inseparability, model
inseparability, and logical equivalence.  Clearly, $\Sigma$-model
inseparability entails $\Theta$-$\mathcal{Q}$ inseparability for
$\Theta=(\Sigma,\Sigma)$ and any class $\mathcal{Q}$ of queries.  The
same is true for logical equivalence, where we can even choose
$\Theta$ freely.  The converse direction is more interesting. 

An ABox $\A$ is said to be \emph{tree-shaped} if the directed graph $(\ind(\A),
\{ (a,b) \mid r(a,b)\in \A\})$ is a tree and $r(a,b)\in \A$ implies
$s(a,b)\not\in \A$ for any $a,b\in \ind(\A)$ and $s\not=r$. We call $\A$
\emph{undirected tree-shaped} (or \emph{utree-shaped}) if the undirected
graph $(\ind(\A), \{ \{a,b\} \mid r(a,b)\in \A\})$ is a tree and
$r(a,b)\in \A$ implies $s(a,b)\not\in \A$ for any $a,b\in \ind(\A)$
and $s\not=r$. Observe that every $\mathcal{EL}$ concept $C$
corresponds to a tree-shaped ABox $\Amc_{C}$ and, conversely, every
tree-shaped ABox $\Amc$ corresponds to an $\mathcal{EL}$-concept
$C_{\Amc}$. In particular, for any TBox $\Tmc$ and $\mathcal{EL}$
concept $D$, we have $\Tmc \models C \sqsubseteq D$ iff
$(\Tmc,\Amc_{C})\models D(\rho_{C})$, $\rho_{C}$ the root of
$\Amc_{C}$.
\begin{theorem}\label{logquery}
  Let $\mathcal{L} \in \{ \DLcH, \mathcal{EL} \}$ and let
  $\Theta=(\Sigma_{1},\Sigma_{2})$ be a pair of signatures such that
  $\Sigma_{i}\supseteq {\sf sig}(\Tmc_{1})\cup {\sf
    sig}(\Tmc_{2})$ for $i \in \{1,2\}$. Then the following conditions are equivalent:
\begin{description}
\item[\rm (\emph{i})] $\Tmc_{1}$ and $\Tmc_{2}$ are logically equivalent;
\item[\rm (\emph{ii})] $\Tmc_{1}$ and $\Tmc_{2}$ are $\Theta$-rCQ inseparable.
\end{description}
\end{theorem}
\begin{proof}
  We show $(ii) \Rightarrow (i)$ for $\mathcal{EL}$, the proof for
  $\DLcH$ is similar and omitted.  Assume $\Tmc_{1}$ and $\Tmc_{2}$
  are $\mathcal{EL}$ TBoxes that are not logically equivalent. Then
  there is $C \sqsubseteq D \in \Tmc_{2}$ such that
  $\Tmc_{1}\not\models C \sqsubseteq D$ (or vice versa). We regard $C$
  as the tree-shaped $\Sigma_{1}$-ABox $\Amc_{C}$ with root $\rho_{C}$
  and $D$ as the $\Sigma_{2}$-rCQ $D(x)$. Then
  $(\Tmc_{2},\A_{C})\models D(\rho_{C})$ but
  $(\Tmc_{1},\A_{C})\not\models D(\rho_{C})$. Thus $\Tmc_{1}$ and
  $\Tmc_{2}$ are $\Theta$-rCQ separable.
  \qed
\end{proof}

Of course, Theorem~\ref{logquery} fails when the restriction of
$\Theta$ is dropped. The following example shows that, even with this
restriction, Theorem~\ref{logquery} does not hold for
Horn-$\mathcal{ALC}$.
\begin{example}\label{ex:logical}
Consider the Horn-$\mathcal{ALC}$ TBoxes 
$$
\Tmc_{1}= \{A \sqsubseteq \exists r.\neg A\} \quad \text{and} \quad \Tmc_{2}=\{ A \sqsubseteq \exists r.\top\}.
$$
Clearly, $\Tmc_{1}$ and $\Tmc_{2}$ are not logically equivalent. However, it is easy to see that they are $\Theta$-UCQ inseparable
for any  $\Theta$.
\end{example}

We now relate query inseparability and concept inseparability.  In
\ALC, these notions are incomparable. It already follows from
Example~\ref{ex:logical} that UCQ inseparability does not imply
concept inseparability. The following example shows that the converse
implication does not hold either.
\begin{example}
  Consider the \ALC TBoxes $\Tmc_1 = \emptyset$ and $$\Tmc_2 = \{ B
  \sqcap \exists r . B \sqsubseteq A, \neg B \sqcap \exists r . \neg B
  \sqsubseteq A \}.$$  
  Using Theorem~\ref{bisimuniform}, one can show
  that $\Tmc_{1}$ and $\Tmc_{2}$ are $\Sigma$-concept inseparable, for $\Sigma = \{A,r\}$. 
	However, $\Tmc_{1}$ and $\Tmc_{2}$ are not $\Theta$-CQ inseparable 
	for any $\Theta =
  (\Sigma_{1},\Sigma_{2})$ with $r\in \Sigma_{1}$ and $A \in \Sigma_2$
  since for the ABox $\A = \{r(a,a)\}$ we have $(\Tmc_1,\A) \not
  \models A(a)$ and $(\Tmc_2,\A) \models A(a)$.
\end{example}

In Horn DLs, in contrast, concept inseparability and query
inseparability are closely related. To explain why this is the case,
consider \EL as a paradigmatic example. Since \EL concepts are
positive and existential, an \EL concept inclusion $C \sqsubseteq D$
which shows that two TBoxes $\Tmc_1$ and $\Tmc_2$ are not concept
inseparable is almost the same as a witness $(\Amc,\q(\vec{a}))$ that
query separates $\Tmc_1$ and $\Tmc_2$. In fact, both ABoxes and
queries are positive and existential as well, but they need not be
tree-shaped. Thus, a first puzzle piece is provided by
Theorem~\ref{equivalencequeries} which implies that we need to
consider only tree-shaped queries $\q$. This is complemented by the
observation that it also suffices to consider only tree-shaped ABoxes
\Amc. The latter is also an important foundation for designing
decision procedures for query inseparability in Horn DLs.  The
following result
%
was first proved in \cite{DBLP:journals/jsc/LutzW10} for \EL. We state
it here also for Horn-$\ALC$ \cite{BotoevaLRWZ16} as this will be
needed later on.
\begin{theorem}\label{treeshape}
Let $\Tmc_{1}$ and $\Tmc_{2}$ be Horn-$\mathcal{ALC}$ TBoxes and $\Theta=(\Sigma_{1},\Sigma_{2})$. Then
the following are equivalent:
\begin{description}
\item[\rm (\emph{i})] $\Tmc_{1}$ $\Theta$-CQ entails $\Tmc_{2}$;
\item[\rm (\emph{ii})] for all utree-shaped $\Sigma_{1}$-ABoxes $\Amc$
  and all \EL-concepts $C$ in signature $\Sigma_{2}$: 
\begin{itemize}
\item[$(a)$] if $(\Tmc_{2},\A)\models C(a)$, then $(\Tmc_{1},\A)\models C(a)$
  where $a$ is the root of \A;
\item[$(b)$] if $(\Tmc_{2},\A)\models \exists x \,C(x)$, then $(\Tmc_{1},\A)\models \exists x\,C(x)$.
\end{itemize}
\end{description}
If $\Tmc_{1}$ and $\Tmc_{2}$ are $\mathcal{EL}$ TBoxes, then it is
sufficient to consider tree-shaped ABoxes in~{\rm
  (\emph{ii})}. 
The same holds when CQs are replaced with rCQs and $(b)$ is dropped from~{\rm
  (\emph{ii})}.
\end{theorem}

Theorem~\ref{treeshape} can be proved by an unraveling argument. It is
closely related to the notion of unraveling tolerance from
\cite{DBLP:conf/kr/LutzW12}. As explained above,
Theorem~\ref{treeshape} allows us to prove that concept inseparability
and query inseparability are the same notion. Here, we state this
result only for \EL \cite{DBLP:journals/jsc/LutzW10}. Let
$\Sigma$-$\EL^u$-concept entailment between \EL TBoxes be defined like
$\Sigma$-concept entailment between \EL TBoxes, except that in ${\sf
  cDiff}_{\Sigma}(\Tmc_{1},\Tmc_{2})$ we now admit concept inclusions
$C \sqsubseteq D$ where $C$ is an \EL-concept and $D$ an
$\EL^u$-concept.
\begin{theorem}\label{queryconceptreduction}
Let $\Tmc_{1}$ and $\Tmc_{2}$ be $\mathcal{EL}$ TBoxes and $\Theta=
(\Sigma,\Sigma)$. Then 
\begin{description}

\item[\rm (\emph{i})] $\Tmc_{1}$ $\Sigma$-concept entails $\Tmc_{2}$
  iff $\Tmc_{1}$ $\Theta$-rCQ entails $\Tmc_{2}$;

\item[\rm (\emph{ii})] $\Tmc_{1}$ $\Sigma$-$\EL^u$-concept entails
  $\Tmc_{2}$ iff $\Tmc_{1}$ $\Theta$-CQ entails $\Tmc_{2}$.

\end{description}
\end{theorem}

\subsection{Deciding query inseparability of TBoxes}
\label{subsect:lasttechnical}

We now study the decidability and computational complexity of query
inseparability. Some results can be obtained by transferring results
from Section~\ref{sect:query-separability} on query inseparability for
KBs (in the \ALC and {\sl DL-Lite} case) or results from
Section~\ref{sect:separability} on concept inseparability of TBoxes
(in the \EL case). In other cases, though, this does not seem
possible.  To obtain results for Horn-$\mathcal{ALC}$, in particular,
we need new technical machinery; as before, we proceed by first giving
model-theoretic characterizations and then using tree automata.

In $\DLc$ and $\DLcH$, there is a straightforward reduction of query
inseparability of TBoxes to query inseparability of KBs. Informally,
such a reduction is possible since DL-Lite TBoxes are so restricted
that they can only perform deductions from a single ABox assertion,
but not from multiple ones together.
\begin{theorem}\label{DLLITEh}
  For $\mathcal{Q}\in \{ CQ, rCQ\}$, $\Theta$-$\mathcal{Q}$ entailment and $\Theta$-\Qmc
  inseparability of TBoxes are 
  \begin{description}
\item[\rm (\emph{i})] in \PTime for $\DLc$;
  \item[\rm (\emph{ii})] \ExpTime-complete for $\DLcH$.
  \end{description}
\end{theorem}
\begin{proof}
  Let $\Theta= (\Sigma_{1},\Sigma_{2})$. Using the fact that every CI
  in a $\DLcH$ TBox has only a single concept of the form $A$ or
  $\exists r.\top$ on the left-hand side, one can show that if
  $\Tmc_{1}$ does not $\Theta$-$\mathcal{Q}$ entail $\Tmc_{2}$, then
  there exists a singleton $\Sigma_{1}$-ABox (containing either a
  single assertion of the form $A(c)$ or $r(a,b)$) such that $\K_{1}$
  does not $\Sigma_{2}$-$\mathcal{Q}$ entail $\K_{2}$ for
  $\K_{i}=(\T_{i},\A)$ for $i=1,2$. Now the upper bounds follow from
  Theorem~\ref{thm:in2exptime}. The lower bound proof is a variation
  of the one establishing Theorem~\ref{thm:in2exptime}.
  \qed
\end{proof}

The undecidability proof for CQ (and rCQ) entailment and
inseparability of $\mathcal{ALC}$ KBs
(Theorem~\ref{thm:undecidability}) can also be lifted to the TBox
case; see \cite{BotoevaLRWZ16} for details.
\begin{theorem}\label{thm:undecidabilityTBox}
Let $\mathcal{Q} \in \{\text{CQ},\text{\RCQ}\}$.
\begin{description}
\item[\rm (\emph{i})] $\Theta$-\Qmc entailment of an \ALC TBox by an $\EL$ TBox is undecidable.

\item[\rm (\emph{ii})] $\Theta$-$\mathcal{Q}$ inseparability of an \ALC 
and an \EL TBoxes is undecidable. 
\end{description}
\end{theorem}

In contrast to the KB case, decidability of UCQ entailment and
inseparability of \ALC TBoxes remains open, as well as for the rUCQ versions. Note that, for the extension $\mathcal{ALCF}$ of $\mathcal{ALC}$ with functional roles, 
undecidability of $\Theta$-$\mathcal{Q}$ inseparability can be proved for any class $\mathcal{Q}$ of queries
contained in UCQ and containing an atomic query of the form $A(x)$ or $\exists x A(x)$.
The proof is by reduction to predicate and query emptiness problems that are shown to be undecidable in
\cite{DBLP:journals/jair/BaaderBL16}. Consider, for example, the class of all CQs. It is undecidable whether for an $\mathcal{ALCF}$
TBox $\Tmc$, a signature $\Sigma$, and a concept name $A\not\in \Sigma$, there exists a
$\Sigma$-ABox $\Amc$ such that $(\Tmc,\Amc)$ is satisfiable and $(\Tmc,\Amc)\models \exists x A(x)$ \cite{DBLP:journals/jair/BaaderBL16}.
One can easily modify the TBoxes $\Tmc$ constructed in \cite{DBLP:journals/jair/BaaderBL16} to prove that this problem is still undecidable
if $\Sigma$-ABoxes $\Amc$ such that the KB $(\Tmc,\Amc)$ is not satisfiable are admitted. Now observe that there exists a $\Sigma$-ABox $\Amc$
with $(\Tmc,\Amc)\models \exists x A(x)$ iff $\Tmc$ is not $\Theta$-$\mathcal{CQ}$ inseparable from the empty TBox for $\Theta=(\Sigma,\{A\})$.

\smallskip

We now consider CQ inseparability in $\mathcal{EL}$. From
Theorem~\ref{queryconceptreduction} and Theorem~\ref{thm:ceupperel},
we obtain \ExpTime-completeness of $\Theta$-rCQ inseparability when
$\Theta$ is of the form $(\Sigma,\Sigma)$. \ExpTime-completeness of
$\Theta$-CQ inseparability in this special case was established in
\cite{DBLP:journals/jsc/LutzW10}. Both results actually generalize to
unrestricted signatures $\Theta$.
\begin{theorem}\label{ELqinsepT}
  Let $\mathcal{Q} \in \{\text{CQ},\text{\RCQ}\}$. In
  \EL, $\Theta$-$\mathcal{Q}$-entailment and inseparability of TBoxes are \ExpTime-complete.
\end{theorem}

Theorem~\ref{ELqinsepT} has not been formulated in this generality in
the literature, so we briefly discuss proofs. In the rooted case, the
\ExpTime upper bound follows from the same bound for
Horn-$\mathcal{ALC}$ which we discuss below. In the non-rooted case,
the \ExpTime upper bound for the case $\Theta=(\Sigma,\Sigma)$ in
\cite{DBLP:journals/jsc/LutzW10} is based on
Theorem~\ref{queryconceptreduction} and a direct algorithm for
deciding $\Sigma$-$\EL^u$-entailment. It is not difficult to extend
this algorithm to the general case. Alternatively, one can obtain the
same bound by extending the model-theoretic characterization of
$\Sigma$-concept entailment in $\mathcal{EL}$ given in
Theorem~\ref{equisimuniform} to `$\Theta$-$\EL^u$-concept entailment',
where the concept inclusions in ${\sf
  cDiff}_{\Sigma}(\Tmc_{1},\Tmc_{2})$ are of the form $C \sqsubseteq
D$ with $C$ an \EL concept in signature $\Sigma_{1}$ and $D$ an
$\EL^u$ concept in signature $\Sigma_{2}$.  Based on such a
characterization, one can then modify the automata construction from
the proof of Theorem~\ref{thm:ceupperel} to obtain an \ExpTime upper
bound.

We note that, for acyclic $\mathcal{EL}$ TBoxes (and their extensions with role inclusions and domain and range restrictions), $\Theta$-CQ entailment 
can be decided in polynomial time. This can be proved by a straightforward generalization of the results in \cite{DBLP:journals/jair/KonevL0W12} where 
it is assumed that $\Theta=(\Sigma_{1},\Sigma_{2})$ with $\Sigma_{1}=\Sigma_{2}$. The proof extends the approach sketched in Section~\ref{sect:separability} for deciding concept inseparability for acyclic $\mathcal{EL}$ TBoxes. A prototype system deciding $\Theta$-CQ inseparability and
computing a representation of the logical difference for query inseparability is presented in \cite{CEX25}. 

\smallskip

We now consider query entailment and inseparability in
Horn-$\mathcal{ALC}$, which requires more effort than the cases
discussed so far. We will concentrate on CQs and rCQs. To start with,
it is convenient to break down our most basic problem, query
entailment, into two subproblems:
\begin{enumerate}

\item $\Theta$-\Qmc entailment \emph{over satisfiable ABoxes} is 
  defined in the same way as $\Theta$-\Qmc entailment except that only
  ABoxes satisfiable with both $\Tmc_1$ and $\Tmc_2$ can witness
  inseparability; see the remark after
  Definition~\ref{def:queryinsep}.

\item A TBox $\Tmc_1$ \emph{$\Sigma$-ABox entails} a TBox
  $\Tmc_2$, for a signature $\Sigma$, if for every $\Sigma$-ABox \Amc,
  unsatisfiability of $(\Tmc_2,\Amc)$ implies unsatisfiability of
  $(\Tmc_1,\Amc)$.

\end{enumerate}
It is easy to see that a TBox $\Tmc_2$ is $\Theta$-\Qmc-entailed by a
TBox $\Tmc_1$, $\Theta=(\Sigma_1,\Sigma_2)$, if $\Tmc_2$ is
$\Theta$-\Qmc-entailed by $\Tmc_1$ over satisfiable ABoxes and
$\Tmc_2$ is $\Sigma_1$-ABox entailed by $\Tmc_1$. For proving
decidability and upper complexity bounds, we can thus concentrate on
problems~1 and~2 above. ABox entailment, in fact, is reducible in
polynomial time and in a straightforward way to the containment
problem of ontology-mediated queries with CQs of the form $\exists x
\, A(x)$, which is \ExpTime-complete in $\hALC$~\cite{HLSW-IJCAI16}.
For deciding query entailment over satisfiable ABoxes, we can find a
transparent model-theoretic characterization. The following result
from \cite{BotoevaLRWZ16} is essentially a consequence of 
Theorem~\ref{crit:in}~(i), Theorem~\ref{crit:KB2}, and
Theorem~\ref{treeshape}, but additionally establishes a bound on the
branching degree of witness ABoxes.
\begin{theorem}
\label{thm:second}
Let $\T_{1}$ and $\T_{2}$ be $\hALC$ TBoxes and $\Theta=(\Sigma_{1},\Sigma_{2})$.  
Then 
\begin{description}
\item[\rm (\emph{i})] $\T_{1}$ $\Theta$-CQ entails $\T_{2}$ over
  satisfiable ABoxes iff, for all utree-shaped $\Sigma_{1}$-ABoxes $\A$
  of outdegree $\le |\T_{2}|$ and consistent with $\T_{1}$
  and~$\T_{2}$, $\C_{(\T_{2},\A)}$ is $\Sigma_{2}$-homomorphically
  embeddable into $\C_{(\T_{1},\A)}$;

\item[\rm (\emph{ii})] $\T_{1}$ $\Theta$-rCQ entails $\T_{2}$ over
  satisfiable ABoxes iff, for all utree-shaped $\Sigma_{1}$-ABoxes
  $\A$ of outdegree $\le |\T_{2}|$ and consistent with $\T_{1}$
  and~$\T_{2}$, $\C_{(\T_{2},\A)}$ is con-$\Sigma_{2}$-homomorphically
  embeddable into $\C_{(\T_{1},\A)}$.
\end{description}
\end{theorem}

Based on Theorem~\ref{thm:second}, we can derive upper bounds for
query inseparability in $\hALC$ using tree automata techniques.
\begin{theorem}
\label{thm:lastone}
  In Horn-$\mathcal{ALC}$,
  \begin{description} 

  \item[\rm (\emph{i})] $\Theta$-rCQ entailment and inseparability of TBoxes is
    \ExpTime-complete;

  \item[\rm (\emph{ii})] $\Theta$-CQ entailment and inseparability of TBoxes is
    2\ExpTime-complete.
  \end{description}
\end{theorem}

The automaton constructions are more sophisticated than those used for
proving Theorem~\ref{thm:decidability} because the ABox is not fixed.
The construction in \cite{BotoevaLRWZ16} uses traditional tree
automata whose inputs encode a tree-shaped ABox together with (parts
of) its tree-shaped canonical models for the TBoxes $\Tmc_1$ and
$\Tmc_2$. It is actually convenient to first replace
Theorem~\ref{thm:decidability} with a more fine-grained
characterization that uses simulations instead of homomorphisms and is
more operational. Achieving the upper bounds stated in
Theorem~\ref{thm:lastone} requires a careful automaton construction
using appropriate bookkeeping in the input and mixing alternating with
non-deterministic automata. The lower bound is based on an ATM
reduction.

Interestingly, the results presented above for query inseparability 
between {\sl DL-Lite} TBoxes have recently been applied to analyse containment and inseparability for 
TBoxes with declarative mappings that relate the signature of the data one
wants to query to the signature of the TBox that provides the interface for
formulating queries~\cite{DBLP:conf/kr/BienvenuR16}. We conjecture that the results
we presented for $\mathcal{EL}$ and Horn-$\mathcal{ALC}$ can also be lifted to
the extension by declarative mappings.  

We note that query inseparability between TBoxes is closely related to program expressiveness~\cite{DBLP:conf/pods/ArenasGP14} and to CQ-equivalence of schema mappings~\cite{DBLP:conf/pods/FaginKNP08,DBLP:journals/mst/PichlerSS13}. The latter
is concerned with declarative mappings from a source signature $\Sigma_{1}$
to a target signature $\Sigma_{2}$. Such mappings $\Mmc_{1}$ and $\Mmc_{2}$ are CQ-equivalent if, for any data instance in $\Sigma_{1}$, the certain answers to CQs in the signature $\Sigma_{2}$ under $\Mmc_{1}$ and $\Mmc_{2}$ coincide. The computational complexity of deciding CQ-equivalence of
schema mappings has been investigated in detail~\cite{DBLP:conf/pods/FaginKNP08,DBLP:journals/mst/PichlerSS13}. Regarding the former, translated into the language of DL the program expressive power of a TBox $\Tmc$ is the set of all triples $(\Amc,\q,\vec{a})$ such
that $\Amc$ is an ABox, $\q$ is a CQ, and $\vec{a}$ is a tuple in $\ind(\Amc)$ such that $\Tmc,\Amc\models q(\vec{a})$. It follows that two TBoxes $\Tmc_{1}$ and $\Tmc_{2}$ are $\Theta$-CQ inseparable for a pair $\Theta =(\Sigma_{1},\Sigma_{2})$ iff $\Tmc_{1}$ and $\Tmc_{2}$ have the same program expressive power.





\section{Discussion}
\label{sect:final}

We have discussed a few inseparability relations between description
logic TBoxes and KBs, focussing on model-theoretic characterizations
and deciding inseparability.  In this section, we briefly survey
three other important topics that were not covered in the main
text. (1)~We observe that many inseparability relations considered above (in particular, concept inseparability) fail to satisfy natural
robustness conditions such as robustness under replacement, and discuss
how to overcome this. (2)~Since inseparability tends to be of high
computational complexity or even undecidable, it is interesting to
develop approximation algorithms; we present a brief overview of the
state of the art. (3)~One is often not only interested in deciding
inseparability, but also in computing useful members of an equivalence
class of inseparable ontologies such as uniform interpolants and
the result of forgetting irrelevant symbols from an ontology. We
briefly survey results in this area as well.

\medskip
\noindent
{\bf Inseparability and robustness.} 
We have seen that robustness under replacement is a central property
in applications of model inseparability to ontology reuse and module
extraction. In principle, one can of course also use other
inseparability relations for these tasks. The corresponding notion of
robustness under replacement can be defined in a straightforward
way~\cite{DBLP:series/lncs/KonevLWW09,KWZ10}.
\begin{definition}\em 
Let $\Lmc$ be a DL and $\equiv_{\Sigma}$ an inseparability relation. Then $\Lmc$ is \emph{robust
under replacement} for $\equiv_{\Sigma}$ if $\Tmc_{1} \equiv_{\Sigma} \Tmc_{2}$ implies 
that $\Tmc_{1} \cup \Tmc \equiv_{\Sigma} \Tmc_{2} \cup \Tmc$ for all $\Lmc$ TBoxes $\Tmc_{1},\Tmc_{2}$
and $\Tmc$ such that ${\sf sig}(\Tmc) \cap {\sf sig}(\Tmc_{1}\cup \Tmc_{2})\subseteq \Sigma$.\!\footnote{Robustness under replacement can be defined
for KBs as well and is equally important in that case. In this short discussion, however, we only consider TBox inseparability.}
\end{definition}

Thus, robustness under replacement ensures that $\Sigma$-inseparable
TBoxes can be equivalently replaced by each other even if a new TBox
that shares with $\Tmc_{1}$ and $\Tmc_{2}$ only $\Sigma$-symbols is
added to both. This seems a useful requirement not only for TBox
re-use and module extraction, but also for versioning and
forgetting. Unfortunately, with the exception of model inseparability,
none of the inseparability relations considered in the main part of
this survey is robust under replacement for the DLs in question. The
following counterexample is a variant of examples given in
\cite{KWZ10,DBLP:conf/aaai/KonevKLSWZ11}.

\begin{example}
Suppose $\Tmc_{1}=\emptyset$, $\Tmc_{2}= \{A \sqsubseteq \exists r.B, E \sqcap B \sqsubseteq \bot\}$, and $\Sigma=\{A,E\}$.
Then $\Tmc_{1}$ and $\Tmc_{2}$ are $\Sigma$-concept inseparable in expressive DLs such as $\mathcal{ALC}$ and they are 
$\Theta$-CQ inseparable for $\Theta=(\Sigma,\Sigma)$.
However, for $\Tmc= \{ \top \sqsubseteq E\}$ the TBoxes $\Tmc_{1}\cup
\Tmc$ and $\Tmc_{2}\cup \Tmc$ are neither $\Sigma$-concept inseparable nor
$\Theta$-CQ inseparable.
\end{example}

The only DLs for which concept inseparability is robust under replacement are certain extensions of $\mathcal{ALC}$ with the universal role.
Indeed, recall that by $\Lmc^{u}$ we denote the extension of a DL $\Lmc$ with the universal role $u$. 
Assume that $\Tmc_{1}$ and $\Tmc_{2}$ are $\Sigma$-concept inseparable in $\mathcal{L}^{u}$ and let $\Tmc$ be an $\mathcal{L}^{u}$ TBox
with ${\sf sig}(\Tmc) \cap {\sf sig}(\Tmc_{1}\cup \Tmc_{2}) \subseteq \Sigma$. As $\Lmc$ extends $\mathcal{ALC}$, it is known that $\Tmc$ is logically
equivalent to a TBox of the form $\{ C \equiv \top\}$, where $C$ is an $\Lmc^{u}$ concept. 
Let $D_{0} \sqsubseteq D_{1}$ be a $\Sigma$-CI in $\mathcal{L}^{u}$. Then
\begin{eqnarray*}
\Tmc_{1}\cup \Tmc \models D_{0} \sqsubseteq D_{1} & \quad \mbox{iff}
\quad & \Tmc_{1} \models D_0 \sqcap \forall u.C \sqsubseteq D_{1}\\
                                                   & \mbox{iff} &
                                                   \Tmc_{2} \models
                                                   D_0 \sqcap \forall u.C \sqsubseteq D_{1}\\
																									 & \mbox{iff} & \Tmc_{2}\cup \Tmc \models D_{0} \sqsubseteq D_{1},
\end{eqnarray*}
where the second equivalence holds by $\Sigma$-concept inseparability of $\Tmc_{1}$ and $\Tmc_{2}$ \emph{if we assume that} 
${\sf sig}(\Tmc) \subseteq \Sigma$ (and so ${\sf sig}(C) \subseteq \Sigma$). Recall that, in the definition of robustness under replacement, 
we only require ${\sf sig}(\Tmc) \cap {\sf sig}(\Tmc_{1}\cup \Tmc_{2}) \subseteq \Sigma$,  
and so an additional step is needed for the argument to go through. This step is captured by the following definition.

\begin{definition}\em 
  Let $\Lmc$ be a DL and $\equiv_{\Sigma}$ an inseparability
  relation. Then $\Lmc$ is \emph{robust under vocabulary extensions}
  for $\equiv_{\Sigma}$ if $\Tmc_{1} \equiv_{\Sigma} \Tmc_{2}$ implies
  that $\Tmc_{1} \equiv_{\Sigma'} \Tmc_{2}$ for all $\Sigma'\supseteq
  \Sigma$ with ${\sf sig}(\Tmc_{1} \cup \Tmc_{2}) \cap
  \Sigma'\subseteq \Sigma$.
\end{definition}

Let us return to the argument above. Clearly, if $\mathcal{L}^{u}$ is
robust under vocabulary extensions for concept inseparability, then
the second equivalence is justified and we can conclude that
$\mathcal{L}^{u}$ is robust under replacement for concept
inseparability. In \cite{DBLP:series/lncs/KonevLWW09}, robustness
under vocabulary extensions is investigated for many standard DLs and
inseparability relations. In particular, the following is shown:
\begin{theorem}\label{univers}
The DLs $\mathcal{ALC}^{u}$, $\mathcal{ALCI}^{u}$,
$\mathcal{ALCQ}^{u}$, and $\mathcal{ALCQI}^{u}$ are robust under
vocabulary extensions for concept inseparability, 
and thus also robust under replacement. 
\end{theorem}

Because of Theorem~\ref{univers}, it would be interesting to
investigate concept inseparability for DLs with the universal role and
establish, for example, the computational complexity of concept
inseparability. We conjecture that the techniques used to prove the
2\ExpTime upper bounds without the universal role can be used to
obtain 2\ExpTime upper bounds here as well.

We now consider robustness under replacement for DLs without the universal role. To simplify the discussion, we consider \emph{weak robustness
under replacement}, which preserves inseparability only if TBoxes $\Tmc$ with ${\sf sig}(\Tmc) \subseteq  \Sigma$ are added
to $\Tmc_{1}$ and $\Tmc_{2}$, respectively. It is then a separate task
to lift weak robustness under replacement to full robustness under replacement using, for example, robustness under vocabulary extensions.
It is, of course, straightforward to extend the inseparability
relations studied in this survey in a minimal way so that weak
robustness under replacement is achieved.
For example, say that two $\Lmc$ TBoxes $\Tmc_{1}$ and $\Tmc_{2}$ are \emph{strongly $\Sigma$-concept inseparable in $\Lmc$}
if, for all $\Lmc$ TBoxes $\Tmc$ with ${\sf sig}(\Tmc) \subseteq \Sigma$, we have that $\Tmc_{1} \cup \Tmc$
and $\Tmc_{2} \cup \Tmc$ are $\Sigma$-concept inseparable. Similarly, say that two $\Lmc$ TBoxes $\Tmc_{1}$ and $\Tmc_{2}$ are \emph{strongly $\Theta$-$\mathcal{Q}$-inseparable} 
if, for all $\Lmc$ TBoxes $\Tmc$ with ${\sf sig}(\Tmc) \subseteq \Sigma_{1}\cap \Sigma_{2}$, we have that $\Tmc_{1} \cup \Tmc$
and $\Tmc_{2} \cup \Tmc$ are $\Theta$-$\mathcal{Q}$ inseparable (we assume $\Theta=(\Sigma_{1},\Sigma_{2})$). 
Unfortunately, with the exception of results for the {\sl DL-Lite} family, nothing is known about the properties of the resulting inseparability relations. 
It is proved in \cite{DBLP:conf/aaai/KonevKLSWZ11} that strong $\Theta$-CQ inseparability is still in \ExpTime for $\DLcH$ if $\Sigma_{1}=\Sigma_{2}$. We conjecture that
this result still holds for arbitrary $\Theta$. A variant of strong $\Theta$-CQ inseparability is also discussed in~\cite{KWZ10} and
analyzed for $\DLc$ extended with (some) Boolean operators and unqualified number restrictions. However, the authors of~\cite{KWZ10} do not
consider CQ-inseparability as defined in this survey but inseparability with respect to generalized CQs that use atoms $C(x)$ with $C$ a $\Sigma$-concept in $\Lmc$ 
instead of a concept name in $\Sigma$. This results in a stronger notion of inseparability that is preserved under definitorial extensions and
has, for the {\sl DL-Lite} dialects considered, many of the robustness properties introduced above. It would be of interest to extend this notion of
query inseparability to DLs such as $\mathcal{ALC}$. Regarding strong concept and query inseparability, it would be interesting to investigate
its algorithmic properties for $\mathcal{EL}$ and Horn-$\mathcal{ALC}$.  


\medskip
\noindent
{\bf Approximation.}\label{sec:approximations}
We have argued throughout this survey that inseparability relations and
conservative extensions can play an important role in a variety of 
applications including ontology versioning, ontology refinement,
ontology re-use, ontology modularization, ontology mapping, knowledge base
exchange and forgetting.  One cannot help noticing, though, another
common theme: the high computational complexity of the corresponding reasoning
tasks, which can hinder the \emph{practical use} of these notions or even make
it infeasible. We now give a brief overview of methods that approximate the
notions introduced in the previous sections while incurring lower computational
costs. We will focus on modularization and logical difference.

Locality-based approximations have already been discussed in
Section~\ref{sec:locality}, where we showed how the extraction of
depleting modules can be reduced to standard ontology reasoning.
Notice that $\emptyset$-locality, in turn, can be approximated with a
simple syntactic check. Following~\cite{DBLP:series/lncs/GrauHKS09},
let $\Sigma$ be a signature. Define two sets of $\mathcal{ALCQI}$
concepts $\Cmc^\bot_\Sigma$ and $\Cmc^\top_\Sigma$ as follows:
\begin{align*}
\Cmc^\bot_\Sigma &~::=~ A^\bot \ \mid \ \lnot C^\top \ \mid \ C\sqcap C^\bot \ \mid \ \exists r^\bot.C \mid \ 
\exists r.C^\bot \ \mid \ \geq n\, r^\bot.C \ \mid \ \geq n\,r.C^\bot,\\
\Cmc^\top_\Sigma &~::=~ \lnot C^\bot \ \mid \ C_1^\top\sqcap C_2^\top,
\end{align*}
where $A^\bot\notin \Sigma$ is an atomic concept, $r$ is a role (a role name or
an inverse role) and  $C$ is a concept, $C^\bot\in\Cmc^\bot_\Sigma$,
$C^\top_i\in\Cmc^\top_\Sigma$, $i=1,2$, and $r^\bot$ is 
 $r$ or $r^-$, for $r\in\NR\setminus\Sigma$.
A CI $\alpha$ is \emph{syntactically $\bot$-local} w.r.t.\ $\Sigma$ if it is of
the form $C^\bot\sqsubseteq C$ or $C\sqsubseteq C^\top$. A TBox $\Tmc$ is $\bot$-\emph{local}
if all CIs in $\Tmc$ are $\bot$-local.
Then every TBox $\Tmc$ that is syntactically $\bot$-local w.r.t.\ a signature $\Sigma$ is 
$\emptyset$-local w.r.t.\ $\Sigma$, as shown in~\cite{DBLP:series/lncs/GrauHKS09}.
Notice that checking whether a CI is syntactically $\bot$-local can be done in
linear time. A dual notion of syntactic $\top$-locality has been
introduced in~\cite{JairGrau}. 
Both notions can be used to define $\bot$-local and $\top$-local
modules; $\top$-  and $\bot$-locality module extraction can be iterated leading
to smaller modules~\cite{DBLP:conf/dlog/SattlerSZ09}.

A comprehensive study of different locality
flavours~\cite{DBLP:conf/semweb/VescovoKPS0T13} identified that there is no
statistically significant difference in the sizes of semantic and syntactic
locality modules. In contrast, \cite{DBLP:journals/ai/KonevL0W13}
found that the difference in size between minimal modules (only available for
acyclic \EL TBoxes) and locality-based approximations can be large.  In a
separate line of research, \cite{DBLP:conf/ecai/GatensKW14} showed that
intractable depleting module approximations for unrestricted OWL ontologies
based on  reductions to QBF can also be significantly smaller, indicating a
possibility for better tractable approximations.
Reachability-based
approximations~\cite{DBLP:conf/dlog/NortjeBM13,DBLP:conf/lpar/NortjeBM13}
refine syntactic locality modules. While they are typically smaller,
self-contained and justification preserving, reachability modules are only
$\Sigma$-concept inseparable from the original TBox but not $\Sigma$-model inseparable. A variety of tractable
approximations based on notions of inseparability ranging form
classification inseparability to model inseparability can be computed
by reduction to Datalog reasoning~\cite{DBLP:journals/jair/RomeroKGH16}.

Syntactic restrictions on elements of $\cDiff(\Tmc_1,\Tmc_2)$ lead to
approximations of concept inseparability.  In \cite{HORROCKS}, the  authors consider
counterexamples of the form $A\sqsubseteq B$, $A\sqsubseteq\lnot B$,
$A\sqsubseteq\exists r.B$, $A\sqsubseteq \forall r.B$ and $r\sqsubseteq s$
only, where $A$, $B$ are $\Sigma$-concept names and $r$, $s$ are
$\Sigma$-roles, and use standard reasoners to check for entailment.  This
approach has been extended in~\cite{DBLP:conf/semweb/GoncalvesPS12a} to allow
inclusions between $\Sigma$-concepts to be constructed in accordance with some
grammar rules. In~\cite{DBLP:conf/aaai/KonevKLSWZ11}, CQ-inseparability for
$\DLcH$ is approximated by reduction to a tractable simulation check between
the canonical models. An experimental evaluation showed that this approach is
incomplete in a very small number of cases on real-world ontologies.

\medskip
\noindent {\bf Computing representatives.} Inseparability relations
are equivalence relations on classes of TBoxes. 
One is often interested not only in deciding inseparability, but also in computing useful members
of an equivalence class of inseparable ontologies such as uniform interpolants (or, equivalently, the result of forgetting irrelevant symbols from an ontology).
%
Recall from
Section~\ref{sec:inseparability} that an ontology $\mathcal{O}_{\sf
  forget}$ is the result of forgetting a signature $\Gamma$ in
$\mathcal{O}$ for an inseparability relation $\equiv$ if $\mathcal{O}$
uses only symbols in $\Sigma={\sf sig}(\mathcal{O})\setminus \Gamma$
and $\mathcal{O}$ and $\mathcal{O}_{\sf forget}$ are
$\Sigma$-inseparable for $\equiv$. Clearly, $\mathcal{O}_{\sf forget}$
can be regarded as a representation of its equivalence class under
$\equiv_{\Sigma}$. For model-inseparability, this representation is
unique up to logical equivalence while this need not be the case for other
inseparability relations. 

Forgetting has been studied extensively for various inseparability
relations. A main problems addressed in the literature is that, for
most inseparability relations and ontology languages, the result of
forgetting is not guaranteed to be expressible in the language of the
original ontology. For example, for the TBox $\Tmc=\{A \sqsubseteq \exists r.B, B
\sqsubseteq \exists r.B\}$,  there is no $\mathcal{ALC}$
TBox using only symbols from $\Sigma=\{A,r\}$ that is $\Sigma$-concept
inseparable from $\Tmc$.  This problem gives rise to three interesting
research problems: given an ontology $\mathcal{O}$ and signature
$\Sigma$, can we decide whether the result of forgetting exists in the
language of $\mathcal{O}$ and, if so, compute it? If not, can we
approximate it in a principled way?  Or can we express it in a more
powerful ontology language?  The existence and computation of uniform
interpolants for TBoxes under concept inseparability has been studied
in \cite{DBLP:conf/ijcai/KonevWW09} for acyclic $\mathcal{EL}$ TBoxes,
in \cite{DBLP:journals/ai/NikitinaR14,DBLP:conf/kr/LutzSW12} for
arbitrary $\mathcal{EL}$ TBoxes, and for $\mathcal{ALC}$ and more
expressive DLs in
\cite{DBLP:conf/ijcai/LutzW11,DBLP:conf/cade/KoopmannS14}. The
generalization to KBs has been studied in
\cite{DBLP:journals/amai/WangWTP10,DBLP:journals/ci/WangWTPA14,DBLP:conf/aaai/KoopmannS15}. Approximations
of uniform interpolants obtained by putting a bound on the role depth
of relevant concept inclusions are studied in
\cite{DBLP:conf/semweb/WangWTPA09,DBLP:conf/aaai/ZhouZ11,DBLP:journals/ci/WangWTPA14,DBLP:conf/kr/LudwigK14}. In~\cite{DBLP:conf/aaai/KoopmannS15,DBLP:conf/cade/KoopmannS14},
(a weak form of) uniform interpolants that do not exist in the
original DL are captured using fix-point operators.  The relationship
between deciding concept inseparability and deciding the existence of
uniform interpolants is investigated in
\cite{DBLP:conf/ijcai/LutzW11}. Forgetting under model inseparability
has been studied extensively in logic \cite{GabbaySchmidtSzalas08} and
more recently for DLs \cite{DBLP:conf/semweb/ZhaoS15}. Note that
the computation of universal CQ solutions in knowledge exchange
\cite{DBLP:aij-kb-exchange}  is identical to forgetting the 
signature of the original KB under $\Sigma$-CQ-inseparability.

Uniform interpolants are not the only useful representatives of
equivalence classes of inseparable ontologies. In the KB case, for
example, it is natural to ask whether for a given KB $\K$ and
signature $\Sigma$ there exists a KB $\K'$ with empty TBox that is
$\Sigma$-query inseparable from $\K$. In this case, answering a
$\Sigma$-query in $\K$ could be reduced to evaluating the query in an
ABox. Another example is TBox rewriting, which asks whether for a
given TBox $\Tmc$ in an expressive DL there exists a TBox $\Tmc'$ that
is ${\sf sig}(\Tmc)$-inseparable from $\Tmc$ in a less expressive
DL. In this case tools that are only available for the less expressive
DL but not for the expressive DL would become applicable to the
rewritten TBox. First results regarding this question have been  obtained in
\cite{rewritecons16}.


\smallskip
\noindent
{\bf Acknowledgements.} Carsten Lutz was supported by 
 ERC grant 647289. Boris Konev, Frank Wolter and Michael Zakharyaschev were supported by the UK EPSRC grants EP/M012646, EP/M012670, EP/H043594, and EP/H05099X. 

\bibliographystyle{abbrv}
\bibliography{bibliography}


\newpage
\printindex

\end{document}